\documentclass{article}
\usepackage{paper}

\newcommand{\Bismarck}{\small\textsf{Bismarck}\normalsize}
\newcommand{\PSGD}{\small\textsf{PSGD}\normalsize}
\newcommand{\eat}[1]{}

\title{
  Bolt-on Differential Privacy for Scalable \\Stochastic Gradient Descent-based Analytics
}
\author{
  Xi Wu$^{1}$\thanks{\scriptsize Work done while at UW-Madison.\vskip 1pt} \hspace{10mm}
  Fengan Li$^{1}\footnotemark[1]$ \hspace{10mm}
  Arun Kumar$^2$ \hspace{10mm}
  Kamalika Chaudhuri$^2$ \hspace{10mm}\\
  \vspace{1mm}
  Somesh Jha$^3$ \hspace{15mm}
  Jeffrey Naughton$^{1}\footnotemark[1]$ \\
  \vspace{2mm}
  $^{1}$Google\hspace{10mm}
  $^{2}$University of California, San Diego \hspace{10mm}
  $^{3}$University of Wisconsin-Madison \\
  \vspace{2mm}
  $^{1}$\{wuxi, fenganl, naughton\}@google.com,
  $^2$\{arunkk, kamalika\}@eng.ucsd.edu, \\
  $^3$jha@cs.wisc.edu
}

\begin{document}
\maketitle

\begin{abstract}
  While significant progress has been made separately on analytics systems
  for scalable stochastic gradient descent (SGD) and private SGD,
  none of the major scalable analytics frameworks have incorporated
  differentially private SGD.
  There are two inter-related issues for this disconnect between research
  and practice: (1) low model accuracy due to added noise to guarantee privacy,
  and (2) high development and runtime overhead of the private algorithms.
  This paper takes a first step to remedy this disconnect and
  proposes a private SGD algorithm to address {\em both} issues
  in an integrated manner.
  In contrast to the white-box approach adopted by previous work,
  we revisit and use the classical technique of {\em output perturbation} to 
  devise a novel ``bolt-on'' approach to private SGD.
  While our approach trivially addresses (2), it makes (1) even more challenging.
  We address this challenge by providing a novel analysis of the $L_2$-sensitivity of SGD,
  which allows, under the same privacy guarantees, better convergence of SGD 
  when only a constant number of passes can be made over the data.
  We integrate our algorithm, as well as other state-of-the-art differentially private SGD,
  into \Bismarck{}, a popular scalable SGD-based analytics system on top of an RDBMS. 
  Extensive experiments show that  our algorithm can be easily integrated, incurs virtually 
  no overhead, scales well, and most importantly, yields substantially better (up to 4X)
  test accuracy than the state-of-the-art algorithms on many real datasets.
\end{abstract}

%%% Local Variables: 
%%% mode: latex
%%% TeX-master: t
%%% End: 

\section{Introduction}
\label{sec:intro}
The past decade has seen significant interest from both the data management industry and 
academia in integrating machine learning (ML) algorithms into scalable data processing systems 
such as RDBMSs~\cite{madlib,FKRR12}, Hadoop~\cite{mahout}, and Spark~\cite{spark}. 
In many data-driven applications such as personalized medicine, finance, web search, 
and social networks, there is also a growing concern about the privacy of individuals. 
To this end, {\em differential privacy}, a cryptographically motivated notion,
has emerged as the gold standard for protecting data privacy.
{\em Differentially private ML} has been extensively studied by researchers from 
the database, ML, and theoretical computer science 
communities~\cite{BST14, CMS11, DJW13, JT13, KST12, ZXYZW13, ZZXYW12}.

In this work, we study differential privacy for {\em stochastic gradient descent 
(SGD)}, which has become the optimization algorithm of choice in many scalable
ML systems, especially in-RDBMS analytics systems. For example, \Bismarck~\cite{FKRR12} offers a 
highly efficient in-RDBMS implementation of SGD to provide a single framework 
to implement many convex analysis-based ML techniques. Thus, creating a private 
version of SGD would automatically provide private versions of all these ML techniques.

While previous work has {\em separately} studied in-RDBMS SGD and differentially private SGD,
our conversations with developers at several database companies revealed that none 
of the major in-RDBMS ML tools have incorporated differentially private SGD.
There are two inter-related reasons for this disconnect between research and practice:
(1) low model accuracy due to the noise added to guarantee privacy,
and (2) high development and runtime overhead of the private algorithms.
One might expect that more sophisticated private algorithms might be needed to address 
issue (1) but then again, such algorithms might in turn exacerbate issue (2)!
   
%\kc{Upto now, we are very good. The next paragraph feels a little unfocused.}
%To understand the issues better, we integrated two state-of-the-art differentially private
%SGD algorithms -- Song, Chaudhuri and Sarwate (SCS13~\cite{SCS13})
%and Bassily, Smith and Thakurta (BST14~\cite{BST14})
%-- into the in-RDBMS SGD framework of \Bismarck.
%In our experiments the {\em accuracy} of models produced by
%these algorithms is significantly worse than the non-private ones.
%To this end, SCS13 adds a large amount of noise {\em at each iteration} of SGD,
%resulting in a noisy solution. While BST14 reduces the noise per iteration,
%and can prove optimal convergence of their algorithm when running $O(m)$ passes
%over the data ($m$ is the training set size),
%the convergence is still unsatisfying in practice when only $O(1)$ passes can be afforded.
%Moreover, the ``white-box'' paradigm adopted by both algorithms
%requires modifying the gradient update steps of a non-private SGD algorithm
%and inject noise at each step.
%This induces significant runtime overhead compared to non-private SGD,
%and requires deeper modifications for integration.

%\kc{A possible alternative paragraph}
To understand these issues better, we integrate two state-of-the-art differentially private 
SGD algorithms -- Song, Chaudhuri and Sarwate (SCS13~\cite{SCS13}) and Bassily, Smith and Thakurta
(BST14 \cite{BST14}) -- into the in-RDBMS SGD architecture of \Bismarck. SCS13 adds noise
{\em at each iteration} of SGD, enough to make the iterate differentially private.
BST14 reduces the amount of noise per iteration by subsampling and can guarantee
optimal convergence using $O(m)$ passes over the data (where $m$ is the
training set size); however, in many real applications, we can only afford a constant number
of passes, and hence, we derive and implement a version for $O(1)$ passes.
Empirically, we find that both algorithms suffer from both issues (1) and (2):
their accuracy is much worse than the accuracy of non-private SGD, while their 
``white box'' paradigm requires deep code changes that require modifying the 
gradient update steps of SGD in order to inject noise. In turn, these changes for 
repeated noise sampling lead to a significant runtime overhead. 

%In our experiments the {\em accuracy} of models produced by
%these algorithms is significantly worse than the non-private ones.
%To this end, SCS13 adds a large amount of noise {\em at each iteration} of SGD,
%resulting in a noisy solution. While BST14 reduces the noise per iteration,
%and can prove optimal convergence of their algorithm when running $O(m)$ passes
%over the data ($m$ is the training set size),
%the convergence is still unsatisfying in practice when only $O(1)$ passes can be afforded.
%Moreover, 

In this paper, we take a first step towards mitigating both issues in an
integrated manner.  In contrast to the white box approach of prior work, 
we consider a new approach to differentially private SGD in which we treat the 
SGD implementation as a ``black box'' and inject noise \textit{only at the end}.
In order to make this \textit{bolt-on} approach feasible, we revisit and use the 
classical technique of {\em output perturbation}~\cite{DMNS06}.  
An immediate consequence is that our approach can be trivially integrated into 
\textit{any} scalable SGD system, including in-RDBMS analytics systems such 
as \Bismarck{}, with no changes to the internal code. 
Our approach also incurs virtually no runtime overhead and preserves the 
scalability of the existing system.

While output perturbation obviously addresses the runtime and integration challenge,
it is unclear what its effect is on model accuracy. 
%In this work, we provide a novel
%output perturbation algorithm that gives higher model accuracy than the
%state-of-the-art private SGD algorithms. 
In this work, we provide a novel analysis that leads to an output perturbation procedure with {\em{higher model accuracy}} than the state-of-the-art private SGD algorithms. 
The essence of our solution is a new bound on the $L_2$-sensitivity of SGD which allows, under the same privacy
guarantees, better convergence of SGD when only a constant number of passes over
the data can be made. As a result, our algorithm produces private models that
are significantly more accurate than both SCS13 and BST14 for practical problems. 
Overall, this paper makes the following contributions:

\begin{itemize}
\item We propose a novel bolt-on differentially private algorithm for SGD 
  based on output perturbation. An immediate consequence of our approach is 
  that our algorithm directly inherits many desirable properties of SGD, while 
  allowing easy integration into existing scalable SGD-based analytics systems.

%\item  To enable constructing accurate models with output perturbation,
%  we provide a novel analysis of the $L_2$-sensitivity of SGD.
%  Importantly, our analysis allows better convergence
%  when one can only afford running a {\em constant number of passes} over the data,
%  which is the typical situation in practice.
%  Key to our analysis is the use of the well-known {\em expansion properties}
%  of gradient operators~\cite{Nesterov04, Polyak87}.

\item We provide a novel analysis of the $L_2$-sensitivity of SGD that leads to 
an output perturbation procedure with higher model accuracy than the 
state-of-the-art private SGD algorithms. Importantly, our analysis allows better 
convergence when one can only afford running a {\em constant number of passes} over the data,
which is the typical situation in practice.
Key to our analysis is the use of the well-known {\em expansion properties} of gradient operators~\cite{Nesterov04, Polyak87}.

%gives higher model accuracy than the
%state-of-the-art private SGD algorithms. The essence of our solution is a novel
%analysis of the $L_2$-sensitivity of SGD which allows, under the same privacy
%guarantees, better convergence of SGD when only constant number of passes over
%the data can be made.

\item We integrate our private SGD algorithms, SCS13, and BST14 into \Bismarck{} and 
  conduct a comprehensive empirical evaluation.
  We explain how our algorithms can be easily integrated with little development effort.
  Using several real datasets, we demonstrate that our algorithms run significantly faster, 
  scale well, and yield substantially better test accuracy
  (up to 4X) than SCS13 or BST14 for the same settings.
\end{itemize}

The rest of this paper is organized as follows:
In Section~\ref{sec:preliminaries} we present preliminaries.
In Section~\ref{sec:psgd}, we present our private SGD algorithms
and analyze their privacy and convergence guarantees. Along the way,
we extend our main algorithms in various ways to incorporate common practices of SGD.
We then perform a comprehensive empirical study in Section~\ref{sec:experiments}
to demonstrate that our algorithms satisfy key desired properties for in-RDBMS analytics:
ease of integration, low runtime overhead, good scalability, and high accuracy.
We provide more remarks on related theoretical work in Section~\ref{sec:related}
and conclude with future directions in Section~\ref{sec:conclusion}.

%%% Local Variables:
%%% mode: latex
%%% TeX-master: t
%%% End:

\section{Preliminaries}
\label{sec:preliminaries}
This section reviews important definitions and existing results.

\noindent\textbf{Machine Learning and Convex ERM}.
%While our techniques work for the generalized setting of learning,
%we will describe our problem in the supervised learning setting.
% We begin with a description of our problem in the supervised learning setting.
Focusing on supervised learning, we have a sample space $Z = X \times Y$,
where $X$ is a space of feature vectors and $Y$ is a label space.
We also have an ordered training set $((x_i,y_i))_{i=1}^m$.
Let $\calW \subseteq \bbR^d$ be a hypothesis space equipped with the standard
inner product and 2-norm $\|\cdot\|$.
\eat{\footnote{\small Using standard results
  in machine learning, our results easily extend to the case
  when the hypothesis $w$ lies in a Hilbert Space.}}
We are given a loss function $\ell: \calW \times Z \mapsto \Real$
which measures the how well a $w$ classifies an example $(x, y) \in Z$,
so that given a hypothesis $w \in \calW$ and a sample $(x,y) \in Z$,
we have a loss $\ell(w, (x,y))$.
Our goal is to minimize the {\em empirical risk}
over the training set $S$ (i.e., the empirical risk minimization, or ERM),
defined as $L_S(w) = \frac{1}{m}\sum_{i=1}^m \ell(w, (x_i, y_i))$.
Fixing $S$, $\ell_i(w) = \ell(w, (x_i, y_i))$ is a function of $w$.
In both in-RDBMS and private learning, {\em convex} ERM problems are common, 
where every $\ell_i$ is {\em convex}.
\red{We start by defining some basic properties of loss functions that will be 
needed later to present our analysis.}

\begin{definition}
  Let $f: \calW \mapsto \bbR$ be a function:
  \begin{itemize}
  \item $f$ is {\em convex} if for any $u, v \in \calW$,\\
    $f(u) \ge f(v) + \langle \nabla f(v), u-v \rangle$
  \item $f$ is {\em $L$-Lipschitz} if for any $u, v \in \calW$,\\
    $\|f(u) - f(v)\| \le L\|u-v\|$
  \item   $f$ is {\em $\gamma$-strongly convex} if\\
    $f(u) \ge f(v) + \langle \nabla f(v), u-v \rangle + \frac{\gamma}{2}\|u-v\|^2$
  \item $f$ is {\em $\beta$-smooth} if\\
    $\|\nabla f(u) - \nabla f(v)\| \le \beta\|u - v\|$
  \end{itemize}
\end{definition}

\red{
\noindent\textbf{Example: Logistic Regression}.
The above three parameters ($L$, $\gamma$, and $\beta$) are derived by analyzing the loss 
function. We give an example using the popular $L_2$-regularized logistic regression model
with the $L_2$ regularization parameter $\lambda$.
This derivation is standard in the optimization literature (e.g., see~\cite{BV04}).
We assume some {\em preprocessing} that normalizes each feature vector,
i.e., each $\|x\| \le 1$ \final{(this assumption is common
for analyzing private optimization~\cite{BST14,CMS11,SCS13}.
In fact, such preprocessing are also common for general
machine learning problems~\cite{ml-preprocessing}, not just private ones).}
Recall now that for $L_2$-regularized logistic regression
the loss function on an example $(x, y)$ with $y \in \{\pm 1\}$) is defined as follows:
\begin{align}
  \label{eq:logistic-loss}
  \ell(w, (x, y)) = \ln\big(1 + \exp(-y \langle w, x \rangle)\big) + \frac{\lambda}{2}\|w\|^2
\end{align}
Fixing $\lambda \ge 0$, we can obtain $L$, $\gamma$, and $\beta$ by looking at the expression 
for the gradient ($\nabla\ell(w)$) and the Hessian (${\bf H}(\ell(w))$).
$L$ is chosen as a tight upper bound on $\|\nabla \ell(w)\|$, $\beta$ is chosen as a tight 
upper bound on $\|{\bf H}(\ell(w))\|$, and $\gamma$ is chosen such that ${\bf H}(\ell(w)) \succeq \gamma I$, 
i.e., ${\bf H}(\ell(w)) - \gamma I$ is positive semidefinite).

Now there are two cases depending on whether $\lambda > 0$ or not.
If $\lambda = 0$ we do not have strong convexity (in this case it is only convex),
and we have  $L = \beta = 1$ and $\gamma = 0$.
If $\lambda > 0$, we need to assume a bound on the norm of the hypothesis $w$.
(which can be achieved by rescaling). In particular, suppose $\|w\| \le R$,
then together with $\|x\| \le 1$, we can deduce that
$L = 1 + \lambda R$, $\beta = 1 + \lambda$, and $\gamma = \lambda$.
We remark that these are indeed standard values in the literature
for $L_2$-regularized logistic loss~\cite{BV04}.

The above assumptions and derivation are common in the optimization literature~\cite{BV04,Bubeck15}.
In some ML models, $\ell$ is not differentiable, e.g., the hinge loss for the linear SVM~\cite{hinge-loss}. 
The standard approach in this case is to approximate it with a differentiable and smooth function. For example, for the hinge loss, there is a body of work on the so-called
\emph{Huber SVM}~\cite{hinge-loss}.
In this paper, we focus primarily on logistic regression as our example but we also discuss the Huber SVM
and present experiments for it in the appendix.
}

\noindent\textbf{Stochastic Gradient Descent}.
\red{SGD is a simple but popular optimization algorithm that performs many incremental gradient updates
instead of computing the full gradient of $L_S$.
At step $t$, given $w_t$ and a random example $(x_t, y_t)$, SGD's update rule is as follows:}
\begin{align}
  \label{gradient-update-rule}
  w_{t+1} = G_{\ell_t, \eta_t}(w_t) = w_t - \eta_t\ell_t'(w_t)
\end{align}
where $\ell_t(\cdot) = \ell(\cdot; (x_t, y_t))$ is the loss function
and $\eta_t \in \Real$ is a parameter called the {\em learning rate},
or {\em step size}. We will denote $G_{\ell_t, \eta_t}$ as $G_t$.
A form of SGD that is commonly used in practice is permutation-based SGD (PSGD):
first sample a random permutation $\tau$ of $[m]$
($m$ is the size of the training set $S$),
and then repeatedly apply (\ref{gradient-update-rule})
by cycling through $S$ according to $\tau$.
In particular, if we cycle through the dataset $k$ times,
it is called $k$-pass PSGD.

\red{We now define two important properties of gradient updates that are needed to understand 
the analysis of SGD's convergence in general, as well as our new technical results on 
differentially private SGD: {\em expansiveness} and {\em boundedness}.
Specifically, we use these definitions to introduce a simple but important recent optimization-theoretical
result on SGD's behavior by~\cite{HRS15} that we adapt and apply to our problem setting.
Intuitively, expansiveness tells us how much $G$ can expand or contract the distance between two 
hypotheses, while boundedness tells us how much $G$ modifies a given hypothesis. 
We now provide the formal definitions (due to~\cite{Nesterov04, Polyak87}).}

\begin{definition}[Expansiveness]
  \label{def:expansiveness}
  Let $G: \calW \mapsto \calW$ be an operator that maps a hypothesis
  to another hypothesis. $G$ is said to be $\rho$-expansive if
  $\sup_{w,w'} \frac{\| G(w) - G(w') \|}{\| w - w' \|} \le \rho.$
\end{definition}
\begin{definition}[Boundedness]
  \label{def:boundedness}
  Let $G: \calW \mapsto \calW$ be an operator that maps a hypothesis
  to another hypothesis. $G$ is said to be $\sigma$-bounded if
  $\sup_{w \in \calW} \| G(w) - w \| \le \sigma.$
\end{definition}

\begin{lemma}[Expansiveness (\cite{Nesterov04,Polyak87})]
  \label{lemma:expansiveness}
  Assume that $\ell$ is $\beta$-smooth. Then, the following hold.
  % The following command is very cool, it says that when referencing an item
  % in the enumerate list, we should use the format
  % "the counter of the current lemma" plus a "dot",
  % plus the counter of the item.
  \begin{enumerate}[ref={\thelemma.\arabic*}]
  \item\label{lemma:expansiveness:smooth-convex}
    If $\ell$ is convex, then for any $\eta \le 2/\beta$,
    $G_{\ell,\eta}$ is $1$-expansive.
  \item\label{lemma:expansiveness:smooth-strongly-convex}
    If $\ell$ is $\gamma$-strongly convex,
    then for $\eta \le \frac{2}{\beta + \gamma}$,
    $G_{\ell,\eta}$ is $(1 - \frac{2\eta\beta\gamma}{\beta + \gamma})$-expansive.
  \end{enumerate}
\end{lemma}
In particular we use the following simplification due to~\cite{HRS15}.
\begin{lemma}[\cite{HRS15}]
  \label{lemma:expansiveness-strongly-convex-simplification}
  Suppose that $\ell$ is $\beta$-smooth and $\gamma$-strongly convex.
  If $\eta \le \frac{1}{\beta}$,
  then $G_{\ell,\eta}$ is $(1-\eta\gamma)$-expansive.
\end{lemma}
\begin{lemma}[Boundedness]
  \label{lemma:boundedness}
  Assume that $\ell$ is $L$-Lipschitz. Then the gradient update
  $G_{\ell, \eta}$ is $(\eta L)$-bounded.
\end{lemma}
We are ready to describe a key quantity studied in this paper.
\begin{definition}[$\delta_t$]
  \label{def:divergence}
  Let $w_0, w_1, \dots, w_T$, and $w_0', w_1', \dots, w_T'$
  be two sequences in $\calW$. We define $\delta_t$ as $\|w_t - w_t'\|$.
\end{definition}

The following lemma by Hardt, Recht and Singer~\cite{HRS15} bounds $\delta_t$
using expansiveness and boundedness properties (Lemma~\ref{lemma:expansiveness} and~\ref{lemma:boundedness}).

\begin{lemma}[\small Growth Recursion~\cite{HRS15}]
  \label{lemma:growth-recursion}
  Fix any two sequences of updates $G_1, \dots, G_T$ and
  $G_1', \dots, G_T'$. Let $w_0 = w_0'$ and $w_t = G_t(w_{t-1})$
  and $w_t' = G_t'(w_{t-1}')$ for $t = 1, 2, \dots, T$. Then
  \begin{align*}
    &\delta_0 = 0, \text{ and for $0 < t \le T$}\\
    &\delta_t \le
    \begin{cases}
      \rho\delta_{t-1} & G_t = G_t' \text{ is $\rho$-expansive.} \\
      & \\
      \min(\rho, 1)\delta_{t-1} + 2\sigma_t
      & \begin{array}{@{}l@{}}
          \text{$G_t$ and $G_t'$ are $\sigma_t$-bounded,}\\
          \text{$G_t$ is $\rho$-expansive.}
        \end{array}
    \end{cases}
  \end{align*}
\end{lemma}

\red{Essentially, Lemma~\ref{lemma:growth-recursion} is used as a tool to prove
  ``average-case stability'' of standard SGD in~\cite{HRS15}.
  We adapt and apply this result to our problem setting and devise new differentially private SGD
  algorithms.\footnote{\red{Interestingly, differential privacy can be viewed as notion of
      ``worst-case stability.'' Thus we offer ``worst-case stability.''}}
  The application is non-trivial because of our unique desiderata but we achieve it by leveraging other recent 
  important optimization-theoretical results by~\cite{Shamir16} on the convergence of PSGD.
  Overall, by synthesizing and building on these recent results,
  we are able to prove the {\em convergence} of our private SGD algorithms as well.}

\noindent\textbf{Differential Privacy}.
We say that two datasets $S, S'$ are {\em neighboring},
denoted by $S \sim S'$, if they differ on a single individual's private value. 
Recall the following definition:
\begin{definition}[$(\varepsilon, \delta)$-differential privacy]
  \label{def:dp}
  A (randomized) algorithm $A$ is said to be
  \emph{$(\varepsilon, \delta)$-differentially private}
  if for any neighboring datasets $S, S'$, and any event
  $E \subseteq {\rm Range}(A)$,
  $\Pr[A(S) \in E] \le e^{\varepsilon}\Pr[A(S') \in E] + \delta.$
\end{definition}
In particular, if $\delta = 0$, we will use $\varepsilon$-differential privacy
instead of $(\varepsilon, 0)$-differential privacy.
A basic paradigm to achieve $\varepsilon$-differential privacy
is to examine a query's $L_2$-sensitivity,
\begin{definition}[$L_2$-sensitivity]
  \label{def:l2-sensitivity}
  Let $f$ be a deterministic query that maps a dataset to a vector in $\Real^d$.
  The $L_2$-sensitivity of $f$ is defined to be
  $\Delta_2(f) = \max_{S \sim S'}\| f(S) - f(S') \|.$
\end{definition}
The following theorem relates $\varepsilon$-differential privacy and $L_2$-sensitivity.
\begin{theorem}[\cite{DMNS06}]
  \label{thm:laplace-dp}
  Let $f$ be a deterministic query that maps a database to a vector in $\Real^d$.
  Then publishing $f(D) + \kappa$ where $\kappa$ is sampled from the
  distribution with density
  \begin{align}
    \label{output-perturbation:noise-distribution}
    p(\kappa) \propto \exp\left(-\frac{\varepsilon\|\kappa\|}{\Delta_2(f)}\right)
  \end{align}
  ensures $\varepsilon$-differential privacy.
\end{theorem}
\red{For the interested reader, we provide a detailed algorithm in Appendix~\ref{sec:sampling-laplace} for how to sample from the above distribution.}

Importantly, the $L_2$-norm of the noise vector, $\|\kappa\|$,
is distributed according to the Gamma distribution
$\Gamma\left(d, \Delta_2(f)/\varepsilon\right)$.
We have the following fact about Gamma distributions:
\begin{theorem}[\cite{CMS11}]
  \label{cor:norm-noise-vector}
  For the noise vector $\kappa$, we have that with probability at least
  $1-\gamma$, $\|\kappa\| \le \frac{d\ln(d/\gamma)\Delta_2(f)}{\varepsilon}.$
\end{theorem}
% \red{Sampling from (\ref{output-perturbation:noise-distribution}) is indeed
% where the high dimension kicks in and can destroy utility
% ($\|\kappa\|$ depends linearly on the dimension $d$).}
\red{Note that the noise depends linearithmically on $d$. This could destroy utility (lower accuracy 
dramatically) if $d$ is high. But there are standard techniques to mitigate this issue that are 
commonly used in private SGD literature (we discuss more in Section~\ref{sec:experiments:method-datasets}).}
By switching to Gaussian noise, we obtain $(\varepsilon,\delta)$-differential privacy.
\begin{theorem}[\cite{DR14}]
  \label{thm:gaussian-approximate-dp}
  Let $f$ be a deterministic query that maps a database to a vector in $\Real^d$.
  Let $\varepsilon \in (0, 1)$ be arbitrary. For $c^2 > 2\ln(1.25/\delta)$,
  adding Gaussian noise sampled according to
  \begin{align}
    \label{gaussian-noise}
    {\cal N}(0, \sigma^2);\ \ \ 
    \sigma \ge \frac{c\Delta_2(f)}{\varepsilon},\ \ 
    c^2 > 2\ln\left(\frac{1.25}{\delta}\right)
  \end{align}
  ensures $(\varepsilon, \delta)$-differentially privacy.
\end{theorem}
\red{
For Gaussian noise, the dependency on $d$ is $\sqrt{d}$, instead of $d\ln d$.
}

\final{
\noindent\textbf{Random Projection}.
Known convergence results of private SGD (in fact private ERM in general)
have a poor dependencies on the dimension $d$. To handle high dimensions,
a useful technique is {\em random projection}~\cite{random-projection}.
That is, we sample a random linear transformation $T$ from certain distributions
and apply $T$ to each feature point $x$ in the training set, so $x$ is transformed to $Tx$.
Note that after this transformation two neighboring datasets
(datasets differing at one data point) remain neighboring,
so random projection does not affect our privacy analysis.
Further, the theory of random projection will tell ``what low dimension'' to project to
so that ``approximate utility will be preserved.'' (in our MNIST experiments the accuracy gap
between original and projected dimension is very small).
Thus, for problems with higher dimensions, we invoke the random projection to ``lower''
the dimension to achieve small noise and thus better utility, while preserving privacy.

In our experimental study we apply random projection to one of our datasets (MNIST).
}

%%% Local Variables: 
%%% mode: latex
%%% TeX-master: t
%%% End: 

\section{Private SGD}
\label{sec:psgd}
We present our differentially private PSGD algorithms
and analyze their privacy and convergence guarantees.
Specifically, we present a new analysis of the output perturbation method for PSGD.
Our new analysis shows that very little noise is needed to achieve
differential privacy. In fact, the resulting private algorithms have good convergence
rates with even {\em one pass} over the data.
Since output perturbation also uses standard PSGD algorithm as a black-box,
this makes our algorithms attractive for in-RDBMS scenarios.

This section is structured accordingly in two parts.
In Section~\ref{sec:psgd:algorithms} we give two main differentially private algorithms
for convex and strongly convex optimization.
In Section~\ref{sec:psgd:analysis} we first prove that these two algorithms
are differentially private (Section~\ref{sec:convex} and \ref{sec:strongly-convex}),
then extend them in various ways (Section~\ref{sec:extensions}),
and finally prove their convergence (Section~\ref{sec:convergence}).

\subsection{Algorithms}
\label{sec:psgd:algorithms}
As we mentioned before, our differentially private PSGD algorithms
uses one of the most basic paradigms for achieving differential privacy
-- the {\em output perturbation} method~\cite{DMNS06} based on
{\em $L_2$-sensitivity} (Definition~\ref{def:l2-sensitivity}).
Specifically, our algorithms are ``instantiations'' of the output perturbation method
where the $L_2$-sensitivity parameter $\Delta_2$ is derived using our new analysis.
To describe the algorithms, we assume a standard permutation-based SGD procedure
(denoted as \PSGD) which can be invoked as a black-box.
To facilitate the presentation, Table~\ref{table:parameters} summarizes the parameters.
\begin{table}[ht]
  \centering
  \begin{tabular}{l||l}
    \hline
    Symbol & Meaning \\
    \hline
    $\lambda$ & $L_2$-regularization parameter.\\
    \hline
    $L$ & Lipschitz constant. \\
    \hline
    $\gamma$ & Strong convexity. \\
    \hline
    $\beta$ & Smoothness. \\
    \hline
    $\varepsilon, \delta$ & Privacy parameters.\\
    \hline
    $\eta_t$ & Learning rate or step size at iteration $t$.\\
    \hline
    $\calW$ & A convex set that forms the hypothesis space. \\
    \hline
    $R$ & Radius of the hypothesis space $\calW$. \\
    \hline
    $k$ & Number of passes through the data.\\
    \hline
    $b$ & Mini-batch size of SGD.\\
    \hline
    $m$ & Size of the training set $S$.\\
    \hline
  \end{tabular}
  \caption{\textbf{Notations}.}
  \label{table:parameters}
\end{table}
\begin{algorithm}[ht]
  \caption{Private Convex Permutation-based SGD}
  \label{alg:p-c-psgd}
  \begin{algorithmic}[1]
    \Require{$\ell(\cdot, z)$ is convex for every $z$, $\eta \le 2/\beta$.}
    \Input{Data $S$, parameters $k, \eta, \varepsilon$}
    \Function{\sf\small PrivateConvexPSGD}{$S, k, \varepsilon, \eta$}
      \Let{$w$}{\PSGD($S$)} with $k$ passes and $\eta_t=\eta$\label{p-c-psgd:op:sgd}
      \Let{$\Delta_2$}{$2kL\eta$} \label{p-c-psgd:op:noise}
      \State Sample noise vector $\kappa$ according to
      (\ref{output-perturbation:noise-distribution}).
      \State \Return{$w + \kappa$}
    \EndFunction
  \end{algorithmic}
\end{algorithm}
\begin{algorithm}[ht]
  \caption{Private Strongly Convex Permutation-based SGD}
  \label{alg:p-sc-psgd}
  \begin{algorithmic}[1]
    \Require{$\ell(\cdot, z)$ is $\gamma$-strongly convex for every $z$}
    \Input{Data $S$, parameters $k, \varepsilon$}
    \Function{\sf\small PrivateStronglyConvexPSGD}{$S, k, \varepsilon$}
      \Let{$w$}{\PSGD($S$) with $k$ passes and
        $\eta_t=\min(\frac{1}{\beta}, \frac{1}{\gamma t})$}\label{p-sc-psgd:op:sgd}
      \Let{$\Delta_2$}{$\frac{2L}{\gamma m}$}\label{p-sc-psgd:op:noise}
      \State Sample noise vector $\kappa$ according to
      (\ref{output-perturbation:noise-distribution}).
      \State \Return{$w + \kappa$}
    \EndFunction
  \end{algorithmic}
\end{algorithm}

Algorithms~\ref{alg:p-c-psgd} and~\ref{alg:p-sc-psgd} give our private SGD algorithms
for convex and strongly convex cases, respectively.
A key difference between these two algorithms is at line~\ref{p-c-psgd:op:noise}
where different $L_2$-sensitivities are used to sample the noise $\kappa$.
Note that different learning rates are used:
In the convex case, a constant rate is used,
while a decreasing rate $\frac{1}{\gamma t}$ is used
in the strongly convex case.
Finally, note that the standard PSGD is invoked as a black box
at line~\ref{p-c-psgd:op:sgd}.

%%% Local Variables:
%%% mode: latex
%%% TeX-master: t
%%% End:

\subsection{Analysis}
\label{sec:psgd:analysis}
In this section we investigate privacy and convergence guarantees of
Algorithms~\ref{alg:p-c-psgd} and~\ref{alg:p-sc-psgd}.
Along the way, we also describe extensions to accommodate common practices
in running SGD. Most proofs in this section are deferred to the appendix.

\noindent\textbf{Overview of the Analysis and Key Observations.}
For privacy, let $A(r; S)$ denote a randomized non-private algorithm
where $r$ denotes the randomness (e.g., random permutations sampled by SGD)
and $S$ denotes the input training set.
To bound $L_2$-sensitivity we want to bound $\max_{r,r'}\| A(r;S) - A(r';S') \|$
on a pair of neighboring datasets $S, S'$,
where $r, r'$ can be {\em different randomness sequences of $A$} in general.
This can be complicated since $A(r; \cdot)$ and $A(r'; \cdot)$
may access the data in vastly different patterns.

Our {\em key observation} is that for {\em non-adaptive} randomized algorithms,
it suffices to consider randomness sequences {\em one at a time},
and thus bound $\max_r\|A(r;S)-A(r;S')\|$.
This in turn allows us to obtain a {\em small} upper bound
of the $L_2$-sensitivity of SGD by combining the expansion properties
of gradient operators and the fact that one will only access {\em once}
the differing data point between $S$ and $S'$ for each pass over the data,
if $r$ is a random permutation.

Finally for convergence, while using permutation benefits our privacy proof,
the convergence behavior of permutation-based SGD is poorly understood in theory.
Fortunately, based on very recent advances by Shamir~\cite{Shamir16} on the
sampling-without-replacement SGD, we prove convergence of our private SGD algorithms
even with only one pass over the data.

\noindent\textbf{Randomness One at a Time.}
Consider the following definition,
\begin{definition}[Non-Adaptive Algorithms]
  \label{def:non-adaptiveness}
  A randomized algorithm $A$ is {\em non-adaptive} if its random choices
  do not depend on the input data values.
\end{definition}
PSGD is clearly non-adaptive as a single random permutation is sampled
at the very beginning of the algorithm. Another common SGD variant,
where one independently and uniformly samples $i_t \sim [m]$ at iteration $t$
and picks the $i_t$-th data point, is also non-adaptive.
In fact, more modern SGD variants,
such as Stochastic Variance Reduced Gradient (SVRG~\cite{JZ13})
and Stochastic Average Gradient (SAG~\cite{RSB12}), are non-adaptive as well.
Now we have the following lemma for non-adaptive algorithms
and differential privacy.

\begin{lemma}
  \label{lemma:laplace-dp-oblivious-randomized-algorithm}
  Let $A(r;S)$ be a non-adaptive randomized algorithm
  where $r$ denotes the randomness of the algorithm
  and $S$ denotes the dataset $A$ works on. Suppose that
  \[ \sup_{S \sim S'}\sup_r \| A(r; S) - A(r; S') \| \le \Delta.  \]
  Then publishing $A(r; S) + \kappa$ where $\kappa$ is sampled with density
  $p(\kappa) \propto
  \exp\left(-\frac{\varepsilon\|\kappa\|_2}{\Delta}\right)$
  ensures $\varepsilon$-differential privacy.
\end{lemma}
\begin{proof}
  Let $\widetilde{A}$ denote the private version of $A$.
  $\widetilde{A}$ has two parts of randomness: One part is $r$,
  which is used to compute $A(r;S)$; the second part is $\kappa$,
  which is used for perturbation (i.e. $A(r;S) + \kappa$).
  Let $R$ be the random variable corresponding to the randomness of $A$.
  Note that $R$ does not depend on the input training set.
  Thus for any event $E$,
  \begin{align}
    \begin{split}
      \label{eq:1}
      &\Pr[\widetilde{A}((r,\kappa);S) \in E]\\
      =& \sum_r \Pr[R=r] \cdot \Pr_{\kappa}[A((r,\kappa); S) \in E\ |\ R=r].
    \end{split}
  \end{align}
  Denote $\Pr_{\kappa}[A((r,\kappa);S) \in E\ |\ R=r]$
  by $p_{\kappa}(A_r(S) \in E)$. Then similarly for $S'$ we have that
  \begin{align}
    \begin{split}
      \label{eq:2}
      &\Pr[\widetilde{A}((r,\kappa);S') \in E]\\
      =& \sum_r \Pr[R=r] \cdot p_\kappa(A_r(S') \in E).
    \end{split}
  \end{align}
  Compare (\ref{eq:1}) and (\ref{eq:2}) term by term (for every $r$):
  the lemma then follows as we calibrate the noise $\kappa$
  so that $p_{\kappa}(A_r(S) \in E) \le e^\varepsilon p_{\kappa}(A_r(S') \in E)$.
\end{proof}

From now on we denote PSGD by $A$. With the notations in Definition~\ref{def:divergence},
our next goal is thus to bound $\sup_{S \sim S'}\sup_r \delta_T$.
In the next two sections we bound this quantity for convex
and strongly convex optimization, respectively.

\subsubsection{Convex Optimization}
\label{sec:convex}
In this section we prove privacy guarantee when $\ell(\cdot, z)$ is convex.
Recall that for general convex optimization,
we have $1$-expansiveness by Lemma~\ref{lemma:expansiveness:smooth-convex}.
We thus have the following lemma that bounds $\delta_T$.
\begin{lemma}
  \label{lemma:sensitivity:convex}
  Consider $k$-passes PSGD for $L$-Lipschitz, convex
  and $\beta$-smooth optimization
  where $\eta_t \le \frac{2}{\beta}$ for $t = 1,\dots,T$.
  Let $S,S^i$  be any neighboring datasets.
  Let $r$ be a random permutation of $[m]$.
  Suppose that $r(i) = i^*$. Let $T = km$, then
  $\delta_T \le 2L\sum_{j=0}^{k-1}\eta_{i^* + jm}\ .$
\end{lemma}
We immediately have the following corollary on $L_2$-sensitivity
with constant step size,
\begin{corollary}[Constant Step Size]
  \label{lemma:convex:constant-step}
  Consider $k$-passes PSGD for $L$-Lipschitz, convex
  and $\beta$-smooth optimization.
  Suppose further that we have constant learning rate
  $\eta_1 = \eta_2 = \cdots = \eta_T = \eta \le \frac{2}{\beta}$. Then
  $\sup_{S \sim S'}\sup_r\delta_T \le 2kL\eta.$
\end{corollary}
This directly yields the following theorem,
\begin{theorem}
  \label{thm:convex:constant-step-dp}
  Algorithm~\ref{alg:p-c-psgd} is $\varepsilon$-differentially private.
\end{theorem}

We now give $L_2$-sensitivity results for two different choices of step sizes,
which are also common for convex optimization.
\begin{corollary}[\small Decreasing Step Size]
  \label{lemma:convex:decreasing-step}
  Let $c \in [0,1)$ be some constant.
  Consider $k$-passes PSGD for $L$-Lipschitz, convex
  and $\beta$-smooth optimization.
  Suppose further that we take decreasing step size
  $\eta_t = \frac{2}{\beta (t + m^c)}$ where $m$ is the training set size.
  Then
  $\sup_{S \sim S'}\sup_r \delta_T
  = \frac{4L}{\beta}\left(\frac{1}{m^c} + \frac{\ln k}{m}\right)$.
\end{corollary}

\begin{corollary}[\small Square-Root Step Size]
  \label{lemma:convex:square-root-step}
  Let $c \in [0, 1)$ be some constant.
  Consider $k$-passes PSGD for $L$-Lipschitz, convex
  and $\beta$-smooth optimization.
  Suppose further that we take square-root step size
  $\eta_t = \frac{2}{\beta (\sqrt{t} + m^c)}$. Then
  \begin{align*}
  \sup_{S \sim S'}\sup_r \delta_T
    \le & \frac{4L}{\beta}\left(\sum_{j=0}^{k-1}\frac{1}{\sqrt{jm+1} + m^c}\right)\\
    = & O\left(
        \frac{L}{\beta}\left( \frac{1}{m^c}
        + \min\left(\frac{k}{m^c},\sqrt{\frac{k}{m}}\right)
        \right) \right).
  \end{align*}
\end{corollary}
\red{
\noindent\text{\bf Remark on Constant Step Size}.
In Lemma~\ref{lemma:convex:constant-step} the step size is named ``constant'' for the SGD.
However, one should note that Constant step size for SGD can depend
on the size of the training set, and in particular can vanish to zero
as training set size $m$ increases. For example, a typical setting of step size is $\frac{1}{\sqrt{m}}$
(In fact, in typical convergence results of SGD, see, for example in~\cite{Bubeck15, NY83},
the constant step size $\eta$ is set to $1/T^{O(1)}$ where $T$ is the total number of iterations).
This, in particular, implies a sensitivity $O(k\eta) = O(k/\sqrt{m})$,
which vanishes to $0$ as $m$ grows to infinity.
}

% \noindent\text{\bf Remark on Setting $c$ for Square-Root Step Sizes}.
% In SGD we may want to measure the maximal progress one can make,
% i.e. the maximal distance one can walk from the starting point.
% Formally, this is the sum of all step sizes.
% We note that, for this to be large in the case of
% decreasing and square-root step sizes, $c$ should not be set to be too large,
% while making it large is beneficial for reducing noise.
% Indeed, since our initial step size is $O(1/m^c)$,
% and the final step size is $O(1/km)$, the sum is roughly
% $\int_{m^c}^{km}\frac{1}{x}dx = O((1-c)\ln m).$

%%% Local Variables: 
%%% mode: latex
%%% TeX-master: t
%%% End: 

\subsubsection{Strongly Convex Optimization}
\label{sec:strongly-convex}
Now we consider the case where $\ell(\cdot, z)$ is $\gamma$-strongly convex.
In this case the sensitivity is smaller because the gradient operators are
$\rho$-expansive for $\rho < 1$ so in particular they become contractions.
We have the following lemmas.
\begin{lemma}[Constant Step Size]
  \label{lemma:strongly-convex:constant-step}
  Consider PSGD for $L$-Lipschitz,
  $\gamma$-strongly convex and $\beta$-smooth optimization
  with constant step sizes $\eta \le \frac{1}{\beta}$.
  Let $k$ be the number of passes.
  Let $S, S'$  be two neighboring datasets differing at the $i$-th data point.
  Let $r$ be a random permutation of $[m]$.
  Suppose that $r(i) = i^*$. Let $T = km$, then
  $\delta_T \le 2L\eta\sum_{j=0}^{k-1}(1-\eta\gamma)^{(k-j)m - i^*}.$
  In particular,
  \[\sup_{S \sim S'}\sup_r \delta_T
  \le \frac{2\eta L}{1 - (1 - \eta\gamma)^m}.\]
\end{lemma}

\begin{lemma}[\small Decreasing Step Size]
  \label{lemma:strongly-convex:decreasing-step}
  Consider $k$-passes PSGD for $L$-Lipschitz,
  $\gamma$-strongly convex and $\beta$-smooth
  optimization. Suppose further that we use decreasing step length:
  $\eta_t = \min(\frac{1}{\gamma t}, \frac{1}{\beta})$.
  Let $S, S'$  be two neighboring datasets differing at the $i$-th data point.
  Let $r$ be a random permutation of $[m]$.
  Suppose that $r(i) = i^*$. Let $T = km$, then
  $\sup_{S \sim S'}\sup_r\delta_T \le \frac{2L}{\gamma m}.$
\end{lemma}
In particular, Lemma~\ref{lemma:strongly-convex:decreasing-step}
yields the following theorem,
\begin{theorem}
  \label{thm:strongly-convex:decreasing-step-dp}
  Algorithm~\ref{alg:p-sc-psgd} is $\varepsilon$-differentially private.
\end{theorem}
One should contrast this theorem with Theorem~\ref{thm:convex:constant-step-dp}:
In the convex case we bound $L_2$-sensitivity by $2kL\eta$,
while in the strongly convex case we bound it by $2L/\gamma m.$

%%% Local Variables: 
%%% mode: latex
%%% TeX-master: t
%%% End: 

\subsubsection{Extensions}
\label{sec:extensions}
In this section we extend our main argument in several ways:
$(\varepsilon,\delta)$-differential privacy, mini-batching, model averaging,
fresh permutation at each pass, and finally constrained optimization.
These extensions can be easily incorporated to standard PSGD algorithm,
as well as our private algorithms~\ref{alg:p-c-psgd} and~\ref{alg:p-sc-psgd},
and are used in our empirical study later.

\noindent\textbf{$(\varepsilon, \delta)$-Differential Privacy}.
We can also obtain $(\varepsilon, \delta)$-differential privacy easily
using Gaussian noise (see Theorem~\ref{thm:gaussian-approximate-dp}).
\begin{lemma}
  \label{lemma:gaussian-dp-oblivious-randomized-algorithm}
  Let $A(r;S)$ be a non-adaptive randomized algorithm
  where $r$ denotes the randomness of the algorithm
  and $S$ denote the dataset. Suppose that
  \[\sup_{S \sim S'}\sup_r\| A(r;S) - A(r;S') \| \le \Delta.\]
  Then for any $\varepsilon \in (0,1)$,
  publishing $A(r;S) + \kappa$ where each component of $\kappa$
  is sampled using (\ref{gaussian-noise})
  ensures $(\varepsilon, \delta)$-differential privacy.
\end{lemma}
In particular, combining this with our $L_2$-sensitivity results,
we get the following two theorems,
\begin{theorem}[Convex and Constant Step]
  \label{thm:approximate-dp-convex-const-step}
  Algorithm~\ref{alg:p-c-psgd} is $(\varepsilon, \delta)$-differentially private
  if each component of $\kappa$ at line~\ref{p-c-psgd:op:noise} is sampled
  according to equation (\ref{gaussian-noise}).
\end{theorem}

\begin{theorem}[Strongly Convex and Decreasing Step]
  \label{thm:approximate-dp-convex-sqrt-step}
  Algorithm~\ref{alg:p-sc-psgd} is $(\varepsilon, \delta)$-differentially private
  if each component of $\kappa$ at line~\ref{p-c-psgd:op:noise} is sampled
  according to equation (\ref{gaussian-noise}).
\end{theorem}

\noindent\textbf{Mini-batching}.
A popular way to do SGD is that at each step,
instead of sampling a single data point $z_t$
and do gradient update w.r.t. it,
we randomly sample a batch $B \subseteq [m]$ of size $b$, and do
\begin{align*}
  w_{t} = w_{t-1} - \eta_t\frac{1}{b}\left(\sum_{i \in B}\ell_i'(w_{t-1})\right)
  = \frac{1}{b}\sum_{i \in B}G_i(w_{t-1}).
\end{align*}
For permutation SGD, a natural way to employ mini-batch is to partition
the $m$ data points into mini-batches of size $b$
(for simplicity let us assume that $b$ divides $m$),
and do gradient updates with respect to each chunk.
In this case, we notice that mini-batch indeed improves the sensitivity
by a factor of $b$.
In fact, let us consider neighboring datasets $S, S'$, and at step $t$,
we have batches $B, B'$ that differ in at most one data point.
Without loss of generality, let us consider the case where $B, B'$
differ at one data point, then on $S$ we have
$w_{t} = \frac{1}{b}\sum_{i \in B}G_i(w_{t-1}),$
and on $S'$ we have
$w'_{t} = \frac{1}{b}\sum_{i \in B}G_i'(w'_{t-1}),$
and so
\begin{align*}
  \delta_{t}
  =& \left\| \frac{1}{b} \sum_{i \in B} G_i(w_{t-1}) - G_i'(w_{t-1}') \right\| \\
  \le& \frac{1}{b}\sum_{i=1}^B\|G_i(w_{t-1}) - G_i'(w_{t-1}')\|.
\end{align*}
We note that for all $i$ except one in $B$, $G_i = G_i'$, and so by
the Growth Recursion Lemma~\ref{lemma:growth-recursion},
$\|G_i(w_{t-1}) - G_i'(w_{t-1}')\| \le \rho\delta_{t-1}$ if $G_i$ is $\rho$-expansive,
and for the differing index $i^*$,
$\|G_{i^*}(w_{t-1}) - G_{i^*}'(w_{t-1}')\| \le \min(\rho, 1)\delta_{t-1} + 2\sigma_t$.
Therefore, for a uniform bound $\rho_t$ on expansiveness and $\sigma_t$ on boundedness
(for all $i \in B$, which is the case in our analysis), we have that
$\delta_t \le \rho_t\delta_{t-1} + \frac{2\sigma_t}{b}.$
This implies a factor $b$ improvement for all our sensitivity bounds.

\noindent\textbf{Model Averaging}.
Model averaging is a popular technique for SGD.
For example, given iterates $w_1, \dots, w_T$, a common way to do model averaging
is either to output $\frac{1}{T}\sum_{t=1}^Tw_t$ or output the average of the last
$\log T$ iterates. We show that model averaging will not affect our sensitivity result,
and in fact it will give a constant-factor improvement when earlier iterates
have smaller sensitivities. We have the following lemma.
\begin{lemma}[Model Averaging]
  \label{lemma:model-averaging}
  Suppose that instead of returning $w_T$ at the end of the optimization,
  we return an averaged model $\bar{w} = \sum_{t=1}^T\alpha_t w_t,$
  where $\alpha_t$ is a sequence of coefficients that only depend on $t, T$. Then,
  \begin{align*}
    \sup_{S \sim S'}\sup_r\| \bar{w} - \bar{w}' \|
    \le \sum_{t=1}^T\alpha_t\| w_t - w_t' \|
    = \sum_{t=1}^T\alpha_t\delta_t.
  \end{align*}
  In particular, we notice that the $\delta_t$'s we derived
  before are non-decreasing, so the sensitivity is bounded by
  $(\sum_{t=1}^T\alpha_t)\delta_T$.
\end{lemma}

\noindent\textbf{Fresh Permutation at Each Pass}.
We note that our analysis extends verbatim to the case where in each pass
a new permutation is sampled, as our analysis applies to {\em{any}} fixed permutation.

\noindent\textbf{Constrained Optimization}.
Until now, our SGD algorithm is for unconstrained optimization.
That is, the hypothesis space $\calW$ is the entire $\bbR^d$.
Our results easily  extend to constrained optimization
where the hypothesis space $\calW$ is a convex set $\calC$.
That is, our goal is to compute $\min_{w \in {\cal C}} L_S(w)$. In this case,
we change the original gradient update rule~\ref{gradient-update-rule}
to the {\em projected gradient update rule}:
\begin{align}
  \label{rule:projected-sgd}
  w_t = \prod_{\cal C}\big( w_{t-1} - \eta_t\ell_t'(w_{t-1}) \big),
\end{align}
where $\prod_{\cal C}(w) = \argmin_{v}\|v - w\|$ is the projection of $w$
to $\cal C$. It is easy to see that our analysis carries over verbatim to
the projected gradient descent. In fact, our analysis works as long as the optimization
is carried over a Hilbert space (i.e., the $\| \cdot \|$ is induced
by some inner product). The essential reason is that
projection will not increase the distance
($\| \prod u - \prod v \| \le \| u - v \|$),
and thus will not affect our sensitivity argument.

%%% Local Variables: 
%%% mode: latex
%%% TeX-master: t
%%% End: 

\subsubsection{Convergence of Optimization}
\label{sec:convergence}
We now bound the optimization error of our private PSGD algorithms.
More specifically, we bound the {\em{excess empirical risk}}
$L_S(w) - L_S^*$ where $L_S(w)$ is the loss of the output $w$
of our private SGD algorithm and $L_S^*$ is the minimum obtained
by any $w$ in the feasible set $\calW$.
Note that in PSGD we sample data points {\em without replacement}.
While sampling without replacement benefits our $L_2$-sensitivity argument,
its convergence behavior is poorly understood in theory.
Our results are based on very recent advances by Shamir~\cite{Shamir16}
on the sampling-without-replacement SGD.

As in Shamir~\cite{Shamir16}, we assume that
the loss function $\ell_i$ takes the form of $\ell_i(\langle w, x_i \rangle) + r(w)$
where $r$ is some fixed function.
Further we assume that the optimization is carried over a convex set $\calC$
of radius $R$ (i.e., $\|w\| \le R$ for $w \in \calC$).
We use projected PSGD algorithm (i.e., we use the
projected gradient update rule~\ref{rule:projected-sgd}).

Finally,  $R(T)$ is a regret bound if for any $w \in {\cal W}$
and convex-Lipschitz $\ell_1, \dots, \ell_T$,
$\sum_{t=1}^T\ell_t(w_t) - \sum_{t=1}^T\ell_t(w) \le R(T)$
and $R(T)$ is sublinear in $T$.
We use the following regret bound,
\begin{theorem}[Zinkevich~\cite{Zinkevich03}]
  \label{thm:Zinkevich-thm1}
  For SGD with constant step size $\eta_1 = \eta_2 = \cdots = \eta_T = \eta$,
  $R(T)$ is bounded by $\frac{R^2}{2\eta} + \frac{L^2T\eta}{2}$.
\end{theorem}
The following lemma is useful in bounding excess empirical risk.
\begin{lemma}[Risk due to Privacy]
  \label{lemma:error-due-to-privacy}
  Consider $L$-Lipschitz and $\beta$-smooth optimization.
  Let $w$ be the output of the non-private SGD algorithm,
  $\kappa$ be the noise of the output perturbation,
  and $\widetilde{w} = w + \kappa$. Then
  $L_S(w) - L_S(\widetilde{w}) \le L\| \kappa \|.$
\end{lemma}

\red{
\noindent\textbf{$\varepsilon$-Differential Privacy}.
We now give convergence result for SGD with $\varepsilon$-differential privacy.
}

\noindent\textbf{\em Convex Optimization}.
If $\ell(\cdot, z)$ is convex, we use the following theorem from Shamir~\cite{Shamir16},
\begin{theorem}[Corollary 1 of Shamir~\cite{Shamir16}]
  \label{thm:Shamir-cor1}
  Let $T \le m$ (that is we take at most $1$-pass over the data).
  Suppose that each iterate $w_t$ is chosen from $\cal W$,
  and the SGD algorithm has regret bound $R(T)$,
  and that $\sup_{t, w \in {\cal W}}|\ell_t(w)| \le R$, and $\|w\| \le R$ for all
  $w \in {\cal W}$. Finally, suppose that each loss function $\ell_t$ takes the form
  $\bar{\ell}(\langle w, x_t \rangle) + r(w)$
  for some $L$-Lipschitz $\bar{\ell}(\cdot, x_t)$ and $\|x_t\| \le 1$,
  and a fixed $r$, then
  \begin{align*}
    \Exp\left[ \frac{1}{T}\sum_{t=1}^TL_S(w_t) - L_S(w^*)\right] \le \frac{R(T)}{T}
    + \frac{2(12+\sqrt{2}L)R}{\sqrt{m}}.
  \end{align*}
\end{theorem}
Together with Theorem~\ref{thm:Zinkevich-thm1}, we thus have the following lemma,
\begin{lemma}
  \label{lemma:convex:nonprivate-optimization-error}
  Consider the same setting as in Theorem~\ref{thm:Shamir-cor1},
  and $1$-pass PSGD optimization defined according to rule (\ref{rule:projected-sgd}).
  Suppose further that we have constant learning rate $\eta = \frac{R}{L\sqrt{m}}$.
  Finally, let $\bar{w}_m$ be the model averaging $\frac{1}{m}\sum_{t=1}^Tw_t$. Then,
  \begin{align*}
    \Exp[L_S(\bar{w}_T) - L_S^*] \le \frac{(L + 2(12+\sqrt{L}))R}{\sqrt{m}}.
  \end{align*}
\end{lemma}
Now we can bound the excess empirical risk as follows,
\begin{theorem}[\small Convex and Constant Step Size]
  \label{lemma:convex:private-optimization-error}
  Consider the same setting as in
  Lemma~\ref{lemma:convex:nonprivate-optimization-error}
  where the step size is constant $\eta = \frac{R}{L\sqrt{m}}$.
  Let $\tilde{w} = \bar{w}_T + \kappa$ be the result of Algorithm~\ref{alg:p-c-psgd}.
  Then
  \begin{align*}
    \Exp[L_S(\tilde{w}) - L_S^*] \le
    \frac{(L+(2(12+\sqrt{L}))R}{\sqrt{m}} + \frac{2dLR}{\varepsilon\sqrt{m}}.
  \end{align*}
\end{theorem}
Note that the term $\frac{2dLR}{\varepsilon\sqrt{m}}$ corresponds to
the expectation of $L\|\kappa\|$.

\noindent\textbf{\em Strongly Convex Optimization}.
If $\ell(\cdot, z)$ is $\gamma$-strongly convex, we instead use the following theorem,
\begin{theorem}[Theorem 3 of Shamir~\cite{Shamir16}]
  \label{thm:Shamir-cor3}
  Suppose $\cal W$ has diameter $R$, and $L_S(\cdot)$ is $\gamma$-strongly convex
  on $\cal W$. Assume that each loss function $\ell_t$ takes the for
  $\bar{\ell}(\langle w_t, x_t \rangle) + r(w)$
  where $\| x_i \| \le 1$, $r(\cdot)$ is possibly some regularization term,
  and each $\bar{\ell}(\cdot, x_t)$ is $L$-Lipschitz and $\beta$-smooth.
  Furthermore, suppose $\sup_{w \in {\cal W}}\|\ell'_t(w)\| \le G$.
  Then for any $1 < T \le m$, if we run SGD for $T$ iterations
  with step size $\eta_t = 1/\gamma t$,
  we have
  \begin{align*}
    \Exp\left[ \frac{1}{T}\sum_{t=1}^TL_S(w_t) - L_S(w^*)\right]
    \le c \cdot \frac{((L + \beta R)^2 + G^2)\log T}{\gamma T},
  \end{align*}
  where $c$ is some universal positive constant.
\end{theorem}
By the same argument as in the convex case, we have,
\begin{theorem}[\small Strongly Convex and Decreasing Step Size]
  Consider the same setting as in Theorem~\ref{thm:Shamir-cor3}
  where the step size is $\eta_t = \frac{1}{\gamma t}$.
  Consider $1$-pass PSGD. Let $\bar{w}_T$ be the result of model averaging
  and $\tilde{w} = \bar{w}_T + \kappa$ be the result of output perturbation. Then
  $\Exp[L_S(\tilde{w}) - L_S(w^*)]
  \le c \cdot \frac{((L + \beta R)^2 + G^2)\log m}{\gamma m}
  + \frac{2dG^2}{\varepsilon\gamma m}.$
\end{theorem}

\red{
\noindent\textbf{\em Remark}.
Our convergence results for $\varepsilon$-differential privacy is different from
previous work, such as BST14, which only give convergence for
$(\varepsilon, \delta)$-differential privacy for $\delta > 0$.
In fact, BST14 relies in an essential way on the advanced composition of
$(\varepsilon,\delta)$-differential privacy~\cite{DR14} and we are not aware
its convergence for $\varepsilon$-differential privacy.
Note that $\varepsilon$-differential privacy is
{\em qualitatively different} from $(\varepsilon, \delta)$-differential privacy
(see, for example, paragraph 3, pp. 18 in Dwork and Roth~\cite{DR14},
as well as a recent article by McSherry~\cite{eps-dp-vs-eps-delta-dp}).
We believe that our convergence results for $\varepsilon$-differential privacy
is important in its own right.

\noindent\textbf{$(\varepsilon, \delta)$-Differential Privacy}.
By replacing Laplace noise with Gaussian noise, we can derive similar convergence results of our algorithms
for $(\varepsilon, \delta)$-differential privacy for $1$-pass SGD.

It is now instructive to compare our convergence results with BST14 for constant number of passes.
In particular, by plugging in different parameters into the analysis of BST14
(in particular, Lemma 2.5 and Lemma 2.6 in BST14) one can derive variants of their results
for constant number of passes. The following table compares the convergence
in terms of the dependencies on the number of training points $m$, and the number of dimensions $d$.

\begin{table}[!htp]
  \centering
  \begin{tabular}{l||l|l}
    \hline
           & Ours & BST14  \\
    \hline
    Convex & $O\big(\frac{\sqrt{d}}{\sqrt{m}}\big)$ & $O\big(\frac{\sqrt{d}(\log^{3/2}m)}{\sqrt{m}}\big)$ \\
    \hline
    Strongly Convex & $O\big(\frac{\sqrt{d}\log m}{m}\big)$ & $O\big(\frac{d\log^2m}{m}\big)$\\
    \hline
  \end{tabular}
  \caption{
    Convergence for $(\varepsilon, \delta)$-DP and {\em constant number of passes}.
  }
  \label{table:step-size}
\end{table}
In particular, in the convex case our convergence is better with a $\log^{3/2}m$ factor,
and in the strongly convex case ours is better with a $\sqrt{d}\log m$ factor.
{\em These logarithmic factors are inherent in BST14 due to its dependence
  on some optimization results (Lemma 2.5, 2.6 in their paper), which we do not rely on}.
Therefore, this comparison gives
{\em theoretical evidence that our algorithms converge better for constant number passes}.
On the other hand, these logarithmic factors become irrelevant for BST14 with $m$ passes,
as the denominator becomes $m$ in the convex case, and becomes $m^2$ in the strongly case,
giving better dependence on $m$ there.
}

%%% Local Variables: 
%%% mode: latex
%%% TeX-master: t
%%% End: 

%%% Local Variables:
%%% mode: latex
%%% TeX-master: t
%%% End:

%%% Local Variables: 
%%% mode: latex
%%% TeX-master: t
%%% End: 

\section{Implementation and Evaluation}
\label{sec:experiments}
In this section, we present a comprehensive empirical study
comparing three alternatives for private SGD: two previously 
proposed state-of-the-art private SGD algorithms, 
SCS13 \cite{SCS13} and BST14 \cite{BST14}, and our algorithms 
which are instantiations of the output perturbation method
with our new analysis.

Our goal is to answer four main questions associated with the key 
desiderata of in-RDBMS implementations of private SGD, viz., ease of 
integration, runtime overhead, scalability, and accuracy:
\begin{enumerate}
\item {\em What is the effort to integrate each algorithm into an in-RDBMS analytics system?}
\item {\em What is the runtime overhead and scalability of the private SGD implementations?}
\item {\em How does the test accuracy of our algorithms compare to SCS13 and BST14?}
\item {\em How do various parameters affect the test accuracy?}
\end{enumerate}
As a summary, our main findings are the following:
{\bf (i)} Our SGD algorithms require almost no changes to \Bismarck,
while both SCS13 and BST14 require deeper code changes.
{\bf (ii)} Our algorithms incur virtually no runtime overhead,
while SCS13 and BST14 run much slower. Our algorithms scale linearly with 
the dataset size. While SCS13 and BST14 also enjoy linear scalability,
the runtime overhead they incur also increases linearly.
{\bf (iii)} Under the same differential privacy guarantees,
our private SGD algorithms yield substantially better accuracy
than SCS13 and BST14, for all datasets and settings of parameters we test.
{\bf (iv)} As for the effects of parameters, our empirical results align well
with the theory. For example, as one might expect, mini-batch sizes are important for
reducing privacy noise. The number of passes is more subtle. For our algorithm,
if the learning task is only convex, more passes result in larger noise
(e.g., see Lemma~\ref{lemma:convex:constant-step}),
and so give rise to potentially worse test accuracy. On the other hand,
if the learning task is strongly convex, the number of passes will not affect
the noise magnitude (e.g., see Lemma~\ref{lemma:strongly-convex:decreasing-step}).
As a result, doing more passes may lead to better convergence
and thus potentially better test accuracy.
Interestingly, we note that slightly enlarging mini-batch size can reduce noise
very effectively so it is affordable to run our private algorithms for more passes
to get better convergence in the convex case.
This corroborates the results of~\cite{SCS13} that mini-batches are helpful
in private SGD settings.

In the rest of this section we give more details of our evaluation.
Our discussion is structured as follows:
In Section~\ref{sec:experiments:previous-approaches}
we first discuss the implemented algorithms.
In particular, we discuss how we modify SCS13 and BST14 to make them better
fit into our experiments. We also give some remarks on other relevant
previous algorithms, and on parameter tuning.
Then in Section~\ref{sec:experiments:systems} we discuss the effort of integrating
different algorithms into \Bismarck.
Next Section~\ref{sec:experiments:method-datasets} discusses
the experimental design and datasets for runtime overhead, scalability and test accuracy.
Then in Section~\ref{sec:experiments:bismarck}, we report runtime overhead and scalability results.
We report test accuracy results for various datasets and parameter settings,
and discuss the effects of parameters in Section~\ref{sec:experiments:accuracy-and-parameter-effect}.
Finally, we discuss the lessons we learned from our experiments~\ref{sec:lessons}.

\subsection{Implemented Algorithms}
\label{sec:experiments:previous-approaches}
We first discuss implementations of our algorithms, SCS13 and BST14.
Importantly, we extend both SCS13 and BST14 to make them better fit into
our experiments. Among these extensions, probably most importantly,
we extend BST14 to support a smaller number of iterations
through the data and {\em reduce the amount of noise} needed for each iteration.
Our extension makes BST14 more competitive in our experiments.

\noindent\textbf{Our Algorithms}.
We implement Algorithms~\ref{alg:p-c-psgd} and~\ref{alg:p-sc-psgd} with the extensions
of mini-batching and constrained optimization (see Section~\ref{sec:extensions}).
Note that \Bismarck{} already supports standard PSGD algorithm with mini-batching and constrained optimization.
Therefore the only change we need to make for Algorithms~\ref{alg:p-c-psgd}
and~\ref{alg:p-sc-psgd} (note that the total number of updates is $T=km$)
is the setting of $L_2$-sensitivity parameter $\Delta_2$ at line~\ref{p-c-psgd:op:noise} of respective algorithms,
which we divide by $b$ if the mini-batch size is $b$.

\noindent\textbf{SCS13~\cite{SCS13}}. We modify~\cite{SCS13},
which originally only supports one pass through the data,
to support multi-passes over the data.

\noindent\textbf{BST14~\cite{BST14}}.
BST14 provides a second solution for private SGD following the same paradigm as SCS13,
but with less noise per iteration. This is achieved by first,
using a novel subsampling technique and second,
relaxing the privacy guarantee to $(\varepsilon, \delta)$-differential privacy
for $\delta > 0$. This relaxation is necessary as they need to use
advanced composition results for $(\varepsilon,\delta)$-differential privacy.

However, the original BST14 algorithm needs $O(m^2)$ iterations to finish,
which is prohibitive for even moderate sized datasets.
We extend it to support $cm$ iterations for some constant $c$.
Reducing the number of iterations means that potentially
we can {\em reduce the amount of noise for privacy} because data is ``less examined.''
This is indeed the case: One can go through the same proof in~\cite{BST14}
with a smaller number of iterations, and show that each iteration only needs
a smaller amount of noise than before (unfortunately this does not give
convergence results). Our extension makes BST14 {\em more competitive}.
In fact it yields significantly better test accuracy compared to the case
where one na\"{i}vely stops BST14 after $c$ passes,
but the noise magnitude in each iteration is the same as in the original paper~\cite{BST14}
(which is for $m$ passes).
The extended BST14 algorithms are given in Algorithm~\ref{alg:convex-bst14-plus}
and~\ref{alg:strongly-convex-bst14-plus}.
Finally, we also make straightforward extensions
so that BST14 supports mini-batching.

\noindent\textbf{Other Related Work.}
We also note the work of Jain, Kothari and Thakurta~\cite{JKT12}
which is related to our setting. In particular their Algorithm 6 is
similar to our private SGD algorithm in the setting of strong convexity and
$(\varepsilon, \delta)$-differential privacy.
However, we note that their algorithm uses
Implicit Gradient Descent (IGD), which belongs to {\em proximal algorithms}
(see for example Parikh and Boyd~\cite{PB14}) and is known to be more difficult to
implement than stochastic gradient methods. Due to this consideration, in this study
we will not compare empirically with this algorithm.
Finally, we also note that~\cite{JKT12} also has an SGD-style algorithm (Algorithm 3)
for strongly convex optimization and $(\varepsilon, \delta)$-differential privacy.
This algorithm adds noise comparable to our algorithm {\em{at each step}}
of the optimization, and thus we do not compare with it either. 

\noindent\textbf{Private Parameter Tuning}.
We observe that for all SGD algorithms considered,
it may be desirable to fine tune some parameters to achieve the best performance. 
For example, if one chooses to do $L_2$-regularization,
then it is customary to tune the parameter $\lambda$.
We note that under the theme of differential privacy, such parameter tunings
must also be done {\em privately}. To the best of our knowledge however,
no previous work have evaluated the effect of private parameter tuning for SGD.
Therefore we take the natural step to fill in this gap.
We note that there are two possible ways to do this.

\noindent{\bf\em Tuning using Public Data}.
Suppose that one has access to a public data set, which is assumed to be drawn
from the {\em same distribution} as the private data set.
In this case, one can use standard methods to tune SGD parameters,
and apply the parameters to the private data.

\noindent{\bf\em Tuning using a Private Tuning Algorithm}.
When only private data is available, we use a {\em private} tuning algorithm
for private parameter tuning. Following the principle on free parameters~\cite{HMMCZ15}
in experimenting with differential privacy, we note free parameters
$\lambda, \varepsilon, \delta, R, k, b$. For these parameters, $\varepsilon,\delta$
are specified as privacy guarantees. Following common practice
for constrained optimization (e.g.~\cite{SCS13})
we set $R = \frac{1}{\lambda}$ for numeric stability.
Thus the parameters we need to tune are $k, b, \lambda$.
We call $k, b, \lambda$ the {\em tuning parameters}.
We use a standard grid search~\cite{gridsearch}
with commonly used values to define the space of parameter values,
from which the tuning algorithm picks values for the parameters to tune.

We use the tuning algorithm described in the original paper of
Chaudhuri, Monteleoni and Sarwate~\cite{CMS11},
though the methodology and experiments in the following are readily extended to
other private tuning algorithms~\cite{CV13}.
Specifically, let $\theta = (k, b, \lambda)$ denote a tuple of the tuning parameters.
Given a space $\Theta = \{ \theta_1, \dots, \theta_l \}$,
Algorithm~\ref{alg:private-tuning} gives the details of the tuning algorithm.

\begin{algorithm}[ht]
  \caption{Private Tuning Algorithm for SGD}
  \label{alg:private-tuning}
  \begin{algorithmic}[1]
    \Input{ Data $S$,
      space of tuning parameters $\Theta=\{ \theta_1, \dots, \theta_l \}$,
      privacy parameters $\varepsilon,\delta$.
    }
    \Function{\sf\small PrivatelyTunedSGD}{$S, \Theta, \varepsilon, \delta$}
      \State Divide $S$ into $l+1$ equal portions $S_1, \dots, S_{l+1}$.
      \State For each $i \in [l]$, train a hypothesis $w_i$
      using any algorithm \ref{alg:p-c-psgd} -- \ref{alg:strongly-convex-bst14-plus}
      with training set $S_i$ and parameters $\theta_i, \varepsilon, \delta$
      and $R = 1/\lambda$ (if needed).
      \State Compute the number of classification errors $\chi_i$ made by $w_i$
      on $S_{l+1}$.
      \State Pick output hypothesis $w = w_i$ with probability
      \begin{align*}
        p_i =
        \frac{e^{-\varepsilon\chi_i/2}}{\sum_{j=1}^l e^{-\varepsilon\chi_j/2}}.
      \end{align*}
    \EndFunction
  \end{algorithmic}
\end{algorithm}

%%% Local Variables:
%%% mode: latex
%%% TeX-master: "make"
%%% End:

\subsection{Integration with Bismarck}
\label{sec:experiments:systems}
We now explain how we integrate private SGD algorithms in RDBMS.
To begin with, we note that the state-of-the-art way to do in-RDBMS data analysis
is via the User Defined Aggregates (UDA) offered by almost all RDBMSes~\cite{gray}.
Using UDAs enables scaling to larger-than-memory datasets seamlessly
while still being fast.\footnote{\scriptsize The MapReduce abstraction
  is similar to an RDBMS UDA~\cite{mahout}.
  Thus our implementation ideas apply to MapReduce-based systems as well.}
A well-known open source implementation of the UDAs required is
\Bismarck{} \cite{FKRR12}.
\Bismarck{} achieves high performance and scalability through
a unified architecture of in-RDBMS data analytics systems
using the permutation-based SGD.

Therefore, we use \Bismarck{} to experiment with private SGD inside RDBMS.
Specifically, we use \Bismarck{} on top of PostgreSQL,
which implements the UDA for SGD in C to provide high runtime efficiency.
Our results carry over naturally to any other UDA-based implementation of analytics
in an RDBMS. The rest of this section is organized as follows.
We first describe \Bismarck's system architecture.
We then compare the system extensions and the implementation effort
needed for integrating our private PSGD algorithm as well as SCS13 and BST14.
\begin{figure}[!htp]
  \centering
  \includegraphics[width=0.5\columnwidth]{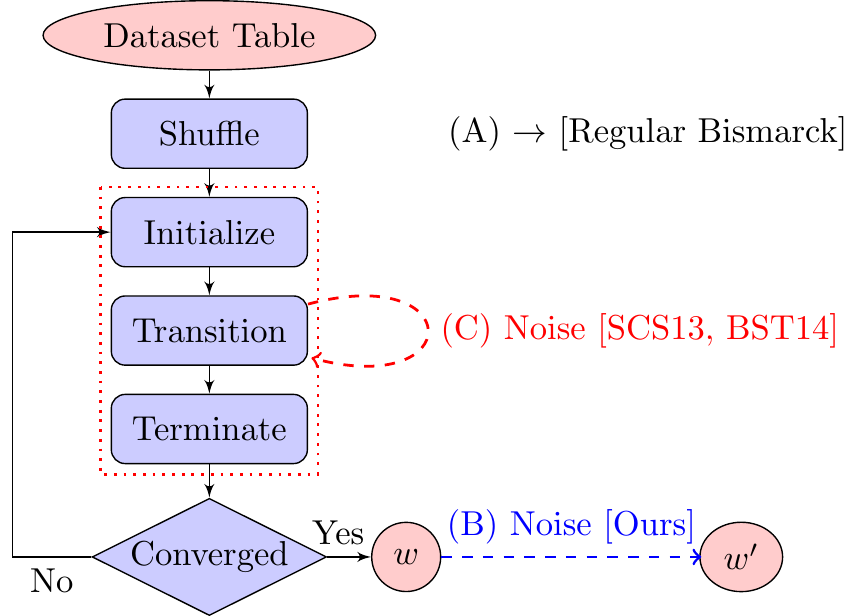}
  \caption{(A) System architecture of regular \Bismarck.
    (B) Extension to implement our algorithms.
    (C) Extension to implement any of SCS13 and BST14.}
  \label{fig:bismarck}
  \centering
\end{figure}

Figure~\ref{fig:bismarck} (A) gives an overview of \Bismarck's architecture.
The dataset is stored as a table in PostgreSQL.
\Bismarck{} permutes the table using an SQL query with a shuffling clause, viz.,
{\tt ORDER BY RANDOM()}. A pass (or epoch, which is used more often in practice)
of SGD is implemented as a C UDA and this UDA is invoked with an SQL query
for each epoch. A front-end controller in Python issues the SQL queries
and also applies the convergence test for SGD after each epoch.
The developer has to provide implementations of three functions in the UDA's C API:
$initialize$, $transition$, and $terminate$,
all of which operate on the $aggregation~state$,
which is the quantity being computed.

To explain how this works, we compare SGD with a standard SQL aggregate:
{\small\tt AVG}. The state for {\small\tt AVG} is the 2-tuple $(sum, count)$,
while that for SGD is the model vector $w$.
The function $initialize$ sets $(sum,count)=(0,0)$
for {\small\tt AVG}, while for SGD,
it sets $w$ to the value given by the Python controller
(the previous epoch's output model).
The function $transition$ updates the state based on a single tuple (one
example). For example, given a tuple with value $x$,
the state update for {\small\tt AVG} is as follows:
$(sum,count)\ \text{+=}\ (x,1)$.
For SGD, $x$ is the feature vector and the update is the update rule
for SGD with the gradient on $x$.
If mini-batch SGD is used, the updates are made to
a temporary accumulated gradient that is part of the aggregation state
along with counters to track the number of examples and mini-batches seen so far.
When a mini-batch is over, the $transition$ function
updates $w$ using the accumulated  gradient for that mini-batch
using an appropriate step size.
The function $terminate$ computes $sum/count$
and outputs it for {\small\tt AVG}, while for SGD,
it simply returns $w$ at the end of that epoch.

It is easy to see that our private SGD algorithm requires
almost no change to \Bismarck{} -- simply add noise to the final $w$ output
after all epochs, as illustrated in Figure~\ref{fig:bismarck} (B).
Thus, our algorithm does not modify any of the RDBMS-related C UDA code.
In fact, we were able to implement our algorithm in about 10 lines of code (LOC)
in Python within  the front-end Python controller.
In contrast, both SCS13 and BST14 require deeper changes to the UDA's
$transition$ function because they need to add noise
at the end of each mini-batch update.
Thus, implementing them required adding dozens of LOC in C
to implement their noise addition procedure
within the $transition$ function,
as illustrated in Figure~\ref{fig:bismarck}  (C).
Furthermore, Python's scipy library already provides the sophisticated distributions
needed for sampling the noise (gamma and multivariate normal), which our algorithm's
implementation exploits. But for both SCS13 and BST14, we need to implement
some of these distributions in C so that it can be used in the UDA.\footnote{\scriptsize
One could use the Python-based UDAs in PostgreSQL but that incurs a significant
runtime performance penalty compared to C UDAs.}

%%% Local Variables:
%%% mode: latex
%%% TeX-master: "make"
%%% End:

\subsection{Experimental Method and Datasets}
\label{sec:experiments:method-datasets}
\noindent We now describe our experimental method and datasets.

\noindent\textbf{Test Scenarios}.
We consider four main scenarios to evaluate the algorithms:
(1) Convex, $\varepsilon$-differential privacy,
(2) Convex, $(\varepsilon,\delta)$-differential privacy,
(3) Strongly Convex, $\varepsilon$-differential privacy, and finally
(4) Strongly Convex, $(\varepsilon,\delta)$-differential privacy.
Note that BST14 only supports $(\varepsilon,\delta)$-differential privacy.
Thus for tests (1) and (3) we compare non-private algorithm,
our algorithms, and SCS13.
For tests (2) and (4), we compare non-private algorithm, our algorithms,
SCS13 and BST14.
For each scenario, we train models on test datasets
and measure the test accuracy of the resulting models.
We evaluate both {\em logistic regression} and {\em Huber support vector machine}
(Huber SVM) (due to lack of space, the results on Huber SVM
are put to Section~\ref{sec:results-huber-svm}).
We use the standard logistic regression for the convex case (Tests (1) and (2)),
and $L_2$-regularized logistic regression for the strongly convex case (Tests (3) and (4)).
We now give more details.
\begin{table}[!htp]
  \centering
  \begin{tabular}{l||l|l|l|l}
    \hline
    Dataset & Task & Train Size & Test Size & \#Dimensions \\
    \hline
    MNIST
    & 10 classes & 60000 & 10000 & 784 (50) $\textcolor{red}{[\ast]}$\\
    \hline
    Protein
    & Binary & 72876 & 72875 & 74 \\
    \hline
    Forest
    & Binary & 498010 & 83002 & 54 \\
    \hline
  \end{tabular}
  \caption{Datasets. Each row gives the name of the dataset,
    number of classes in the classification task, sizes of training and test sets,
    and finally the number of dimensions.
    $\textcolor{red}{[\ast]}$: For MNIST, it originally has 784 dimensions,
    which is difficult for $\varepsilon$-differential privacy as sampling
    from (\ref{output-perturbation:noise-distribution}) makes the magnitude of noise
    depends linearly on the number of dimensions $d$.
    Therefore we randomly project it to $50$ dimensions.
    All data points are normalized to the unit sphere.
  }
  \label{table:datasets}
\end{table}

\noindent\textbf{Datasets}. We consider three standard benchmark datasets:
MNIST\footnote{\scriptsize\url{http://yann.lecun.com/exdb/mnist/}.},
Protein\footnote{{\scriptsize\url{http://osmot.cs.cornell.edu/kddcup/datasets.html}}.},
and Forest Covertype\footnote{{\scriptsize\url{https://archive.ics.uci.edu/ml/datasets/Covertype}}.}.
MNIST is a popular dataset used for image classification.
MNIST poses a challenge to differential privacy for three reasons:
(1) Its number of dimensions is relatively higher than others.
To get meaningful test accuracy we thus use Gaussian Random Projection to randomly project to $50$ dimensions.
This random projection only incurs very small loss in test accuracy,
and thus the performance of non-private SGD on $50$ dimensions will serve the baseline.
(2) MNIST is of medium size and differential privacy is known to be more difficult
for medium or small sized datasets.
(3) MNIST is a multiclass classification (there are $10$ digits),
we built ``one-vs.-all'' multiclass logitstic regression models.
This means that we need to construct $10$ binary models (one for each digit).
Therefore, one needs to split the privacy budget across sub-models.
We used the simplest composition theorem~\cite{DR14},
and divide the privacy budget evenly.

For Protein dataset, because its test dataset does not have labels,
we randomly partition the training set into halves to form train and test datasets.
Logistic regression models have very good test accuracy on it.
Finally, Forest Covertype is a large dataset with 581012 data points,
almost 6 times larger than previous ones.
We split it to have 498010 training points and 83002 test points.
We use this large dataset for two purposes: First, in this case,
one may expect that privacy will follow more easily.
We test to what degree this holds for different private algorithms.
Second, since training on such large datasets is time consuming,
it is desirable to use it to measure runtime overheads of various private algorithms.

\red{
\noindent\textbf{Settings of Hyperparameters}.
The following describes how hyperparameters are set in our experiments.
There are three classes of parameters: Loss function parameters, privacy parameters,
and parameters for running stochastic gradient descent.

\noindent\textbf{\em Loss Function Parameters}.
Given the loss function and $L_2$ regularization parameter $\lambda$, we can derive $L, \beta, \gamma$
as described in Section~\ref{sec:preliminaries}. We privately tune $\lambda$ in $\{0.0001, 0.001, 0.01\}$.

\noindent\textbf{\em Privacy Parameters}.
$\varepsilon, \delta$ are privacy parameters. We vary $\varepsilon$ in $\{0.1, 0.2, 0.5, 1, 2, 4\}$ for MNIST,
and in $\{0.01, 0.02, 0.05, 0.1, 0.2, \\ 0.4\}$ for Protein and Covertype
(as they are binary classification problems and we do not need to divide by 10).
$\delta$ is set to be $1/m^2$ where $m$ is the size of the training set size.

\noindent\textbf{\em SGD Parameters}.
Now we consider $\eta_t$, $b$, and $k$.

\noindent{\bf\em Step Size $\eta_t$}. Step sizes are derived from theoretical analyses of SGD algorithms.
In particular the step sizes only depend on the loss function parameters and the time stamp $t$ during SGD.
Table~\ref{table:step-size} summarizes step sizes for different settings.

\begin{table}[!htp]
  \centering
  \begin{tabular}{l||l|l|l|l}
    \hline
           & Non-private & Ours & SCS13 & BST14  \\
    \hline
    C + $\varepsilon$-DP
           & $\frac{1}{\sqrt{m}}$ & $\frac{1}{\sqrt{m}}$ & $\frac{1}{\sqrt{t}}$ & $\times$ \\
    \hline
    C + $(\varepsilon,\delta)$-DP
           & $\frac{1}{\sqrt{m}}$ & $\frac{1}{\sqrt{m}}$ & $\frac{1}{\sqrt{t}}$
           & Alg.~\ref{alg:convex-bst14-plus} \\
    \hline
    SC + $\varepsilon$-DP
           & $\frac{1}{\gamma t}$ & $\min(\frac{1}{\beta}, \frac{1}{\gamma t})$
           & $\frac{1}{\sqrt{t}}$ & $\times$ \\
    \hline
    SC + $(\varepsilon, \delta)$-DP
           & $\frac{1}{\gamma t}$ & $\min(\frac{1}{\beta}, \frac{1}{\gamma t})$
           & $\frac{1}{\sqrt{t}}$ & Alg.~\ref{alg:strongly-convex-bst14-plus} \\
    \hline
  \end{tabular}
  \caption{Step Sizes for different settings. C: Convex, SC: Strongly Convex.
    For SCS13 we follow in\protect\cite{SCS13} and set step size to be $\protect 1/\sqrt{t}$.
  }
  \vskip 12pt
  \label{table:step-size}
\end{table}

\noindent{\bf \em Mini-batch Size $b$}.
We are not aware of a first-principled way in literature to set mini-batch size
(note that convergence proofs hold even for $b=1$).
In practice mini-batch size typically depends on the system constraints
(e.g. number of CPUs) and is set to some number from $10$ to $100$.
We set $b = 50$ in our experiments for fair comparisons with SCS13 and BST14,
which shows that our algorithms enjoy both efficiency and substantially better test accuracy.

Note that increasing $b$ could reduce noise but makes the gradient step more expensive 
and might require more passes. In general, a good practice is to set $b$ to be reasonably large
without hurting performance too much. To assess the impact of this setting further,
we include an experiment on varying the batch size in Appendix~\ref{sec:acc-vs-mbsize}.
We leave for future work the deeper questions on formally identifying the sweet spot
among efficiency, noise, and accuracy.

\noindent{\bf \em Number of Passes $k$.} For fair comparisons in the experiments below with SCS13 and BST14,
for all algorithms tested we privately tune $k$ in $\{5, 10\}$.
However, for {\em our algorithms} there is a simpler strategy to set $k$ in the {\em strongly convex} case.
Since our algorithms run vanilla SGD as a black box, one can set a convergence tolerance threshold $\mu$
and set a large $K$ as the threshold on the number of passes. Since in the strongly convex case
{\em the noise injected in our algorithms (Alg.~\ref{alg:p-sc-psgd}) does not depend on $k$},
we can run the vanilla SGD until either the decrease rate of training error is smaller than $\mu$,
or the number of passes reaches $K$, and inject noise at the end.

Note that this strategy does {\em not} work for SCS13 or BST14 because in either convex or strongly convex case,
their noise injected in each step depends on $k$, so they must have $k$ fixed beforehand.
Moreover, since they inject noise at each SGD iteration,
it is likely that they will run out of the pass threshold.

The above discussion demonstrates {\em an additional advantage} of our algorithms using output perturbation:
In the strongly convex case the number of passes $k$ is {\em oblivious} to private SGD.

\final{
\noindent{\bf \em Radius $R$}. Recall that for strongly convex optimization the hypothesis space
needs to a have bounded norm (due to the use of $L_2$ regularization). We adopt the practices
in~\cite{SCS13} and set $R = 1/\lambda$.}
}                               %end \red for revision.

\noindent\textbf{Experimental Environment}.
All the experiments were run on a machine with Intel Xeon E5-2680 2.50GHz CPUs
($48$-core) and $64$GB RAM running Ubuntu $14.04.4$.

%%% Local Variables:
%%% mode: latex
%%% TeX-master: "make"
%%% End:

\subsection{Runtime Overhead and Scalability}
\label{sec:experiments:bismarck}
Using output perturbation trivially addresses runtime and scalability concerns.
We confirm this experimentally in this section.

\noindent\textbf{Runtime Overheads}.
We compare the runtime overheads of our private SGD algorithms
against the noiseless version and the other algorithms.
The key parameters that affect runtimes are the number of epochs and the batch sizes.
Thus, we vary each of these parameters, while fixing the others.
The runtimes are the average of $4$ warm-cache runs and all datasets fit in the buffer
cache of PostgreSQL. The error bars represent $90\%$ confidence intervals.
The results are plotted in Figure~\ref{fig:runtimes} (a)--(c) and
Figure~\ref{fig:runtimes} (d)--(f) (only the results of strongly convex,
$(\varepsilon, \delta$)-differential privacy are reported;
the other results are similar and thus, we skip them here for brevity).

The first observation is that our algorithm incurs virtually no runtime overhead over
noiseless \Bismarck, which is as expected because our algorithm only adds noise once
at the end of all epochs.
In contrast, both SCS13 and BST14 incur significant runtime overheads
in all settings and datasets.
In terms of runtime performance for $20$ epochs and a batch size of $10$, 
both SCS13 and BST14 are between {\bf 2X} and {\bf 3X} slower than our algorithm.
The gap grows larger as the batch size is reduced: for a batch size of $1$ and $1$ epoch,
both SCS13 and BST14 are up to {\bf 6X} slower than our algorithm.
This is expected since these algorithms invoke expensive
random sampling code from sophisticated distributions for each mini-batch.
When the batch size is increased to $500$, the runtime gap between these algorithms 
practically disappears as the random sampling code is invoked much less often.
Overall, we find that our algorithms can be significantly faster than the alternatives.

\noindent\textbf{Scalability}.
We compare the runtimes of all the private SGD algorithms as the datasets
are scaled up in size (number of examples). For this experiment, we use the data
synthesizer available in \Bismarck{} for binary classification. We produce two sets 
of datasets for scalability: \textit{in-memory} and \textit{disk-based} (dataset does 
not fit in memory). The results for both are presented in Figure~\ref{fig:scale}. 
We observe linear increase in runtimes for all the algorithms compared in both settings. 
As expected, when the dataset fits in memory, SCS13 and BST14 are much slower
and in particular the runtime overhead increases linearly as data size grows.
This is primarily because CPU costs dominate the runtime. Recall that these algorithms add 
noise to each mini-batch, which makes them computationally more expensive. We also see 
that all runtimes scale linearly with the dataset size even in the disk-based setting. 
An interesting difference is that I/O costs, which are the same for all the algorithms 
compared, dominate the runtime in Figure~\ref{fig:scale}(b). Overall, these results 
demonstrate a key benefit of integrating our private SGD algorithm into an RDBMS-based 
toolkit like \Bismarck: scalability to larger-than-memory data comes for free.

\begin{figure}[!htb]
  \centering
  \begin{tabular}{cc}
    \includegraphics[width=0.3\columnwidth]{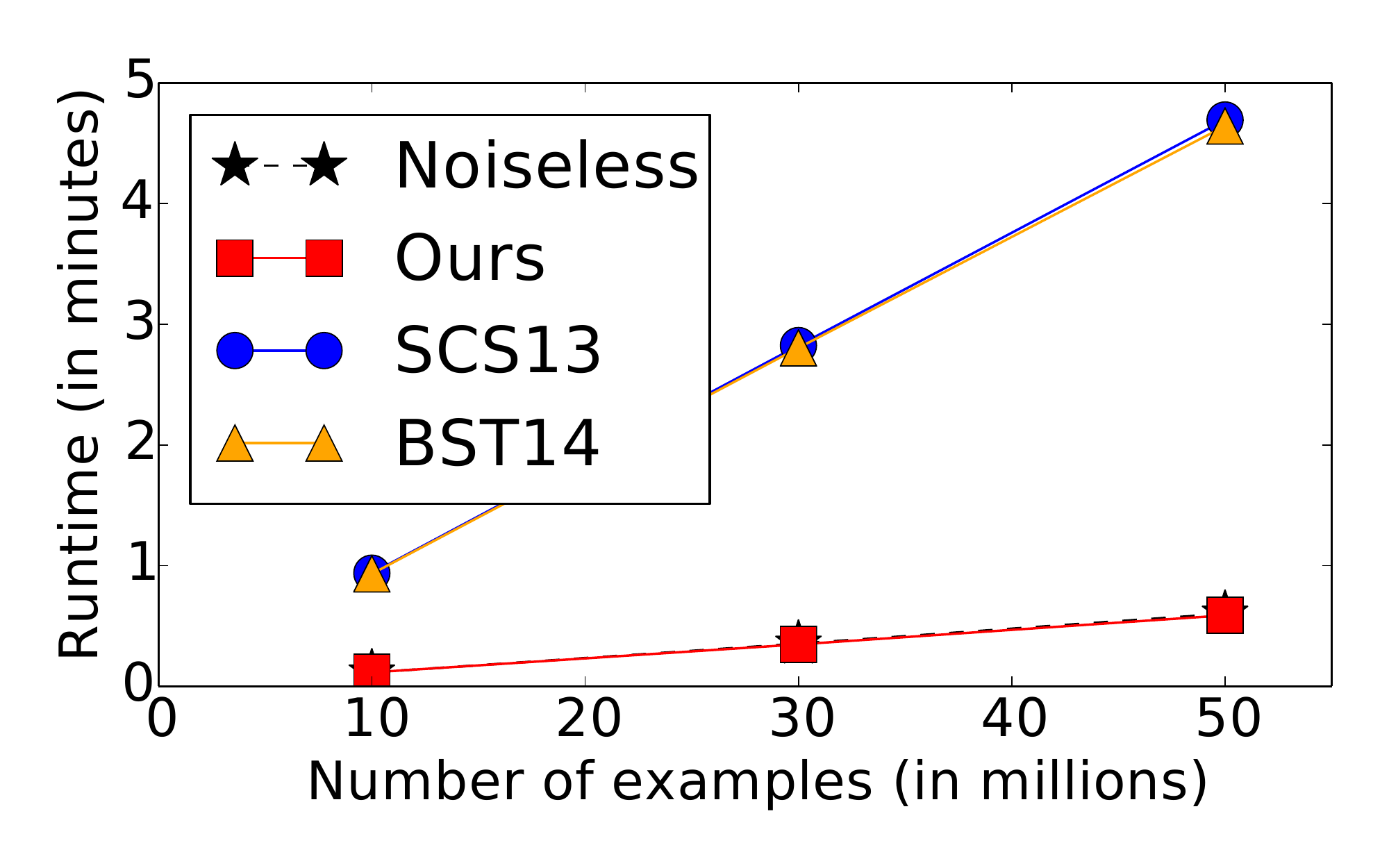} &
    \includegraphics[width=0.3\columnwidth]{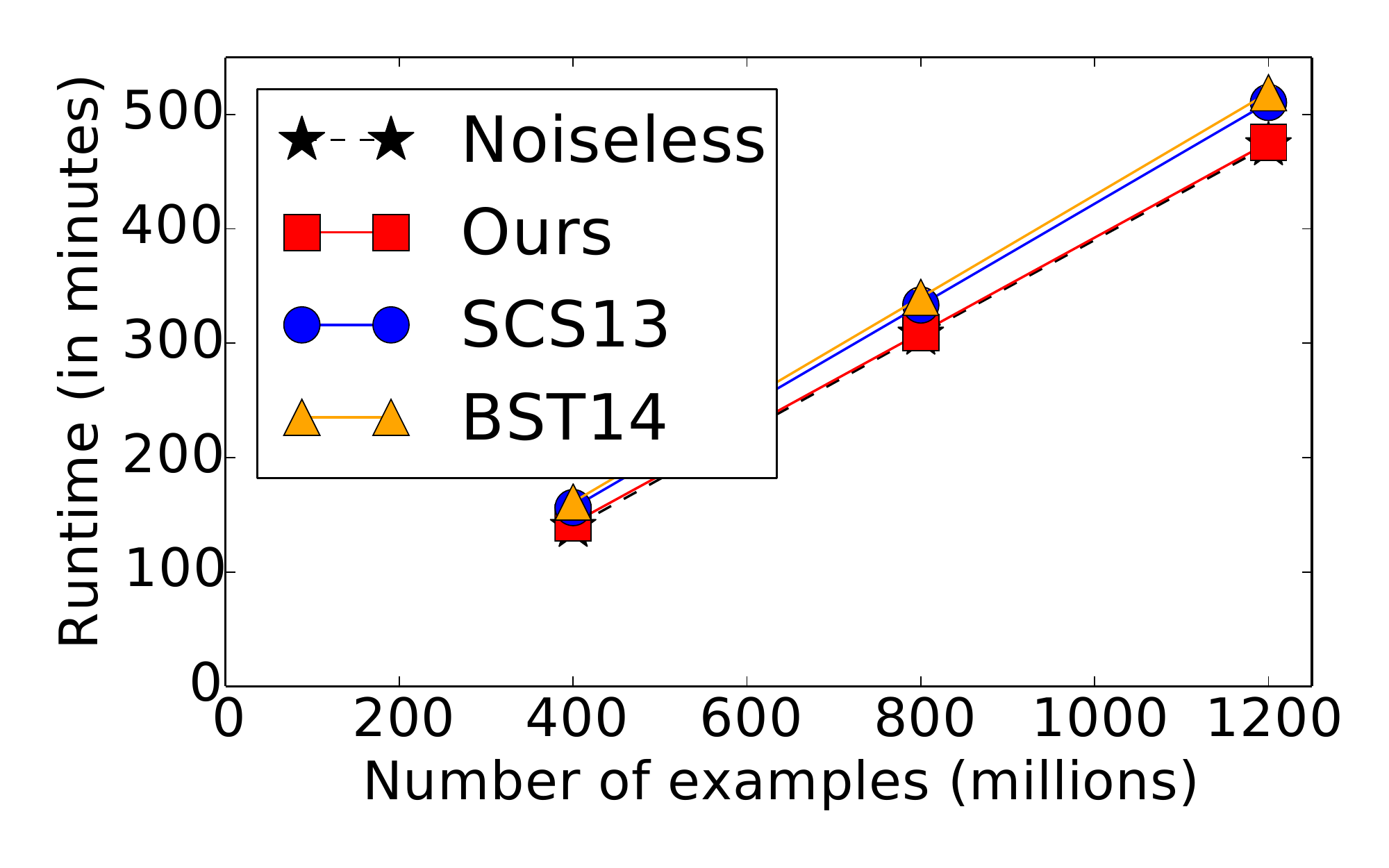} \\
  \end{tabular}
  \caption{{\bf Scalability of $(\epsilon,\delta)$-DP SGD algorithms in \protect\Bismarck{}:
      (a) The dataset fits in memory. (b) The dataset is larger than memory (on disk).
      The runtime per epoch for mini-batch size $=1$ is plotted. All datasets have $d=50$ 
      features. We fix $\epsilon=0.1$ and $\lambda=0.0001$. The dataset sizes 
      vary from $3.7$ GB to $18.6$ GB in (a) and from $149$ GB to $447$ GB in (b).}}
  \label{fig:scale}
\end{figure}

%%% Local Variables:
%%% mode: latex
%%% TeX-master: "make"
%%% End:

\subsection{Accuracy and Effects of Parameters}
\label{sec:experiments:accuracy-and-parameter-effect}
\begin{figure*}[!htb]
  \centering
  \subfloat{
    \includegraphics[width=0.24\columnwidth]
    {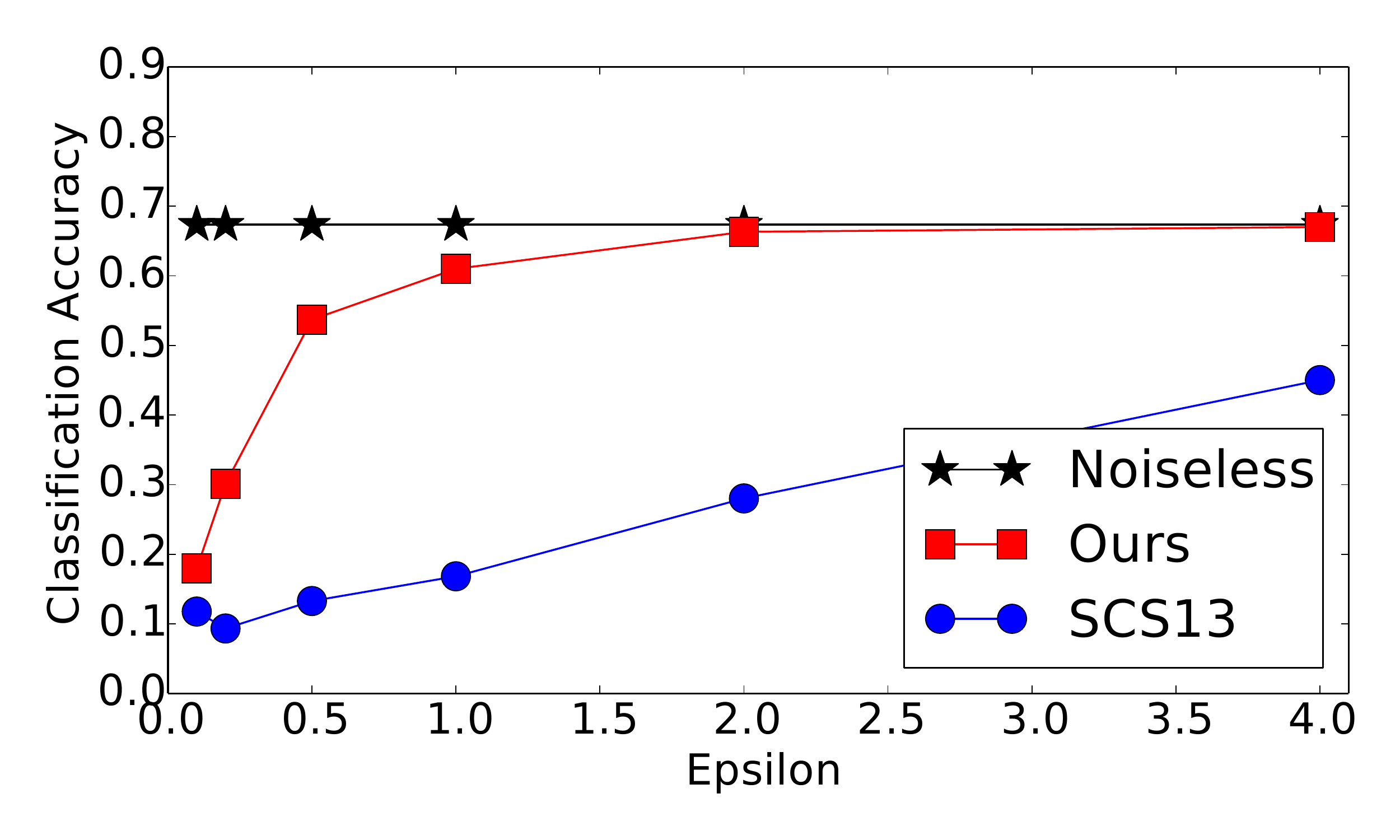}
  }
  \subfloat{
    \includegraphics[width=0.24\columnwidth]
    {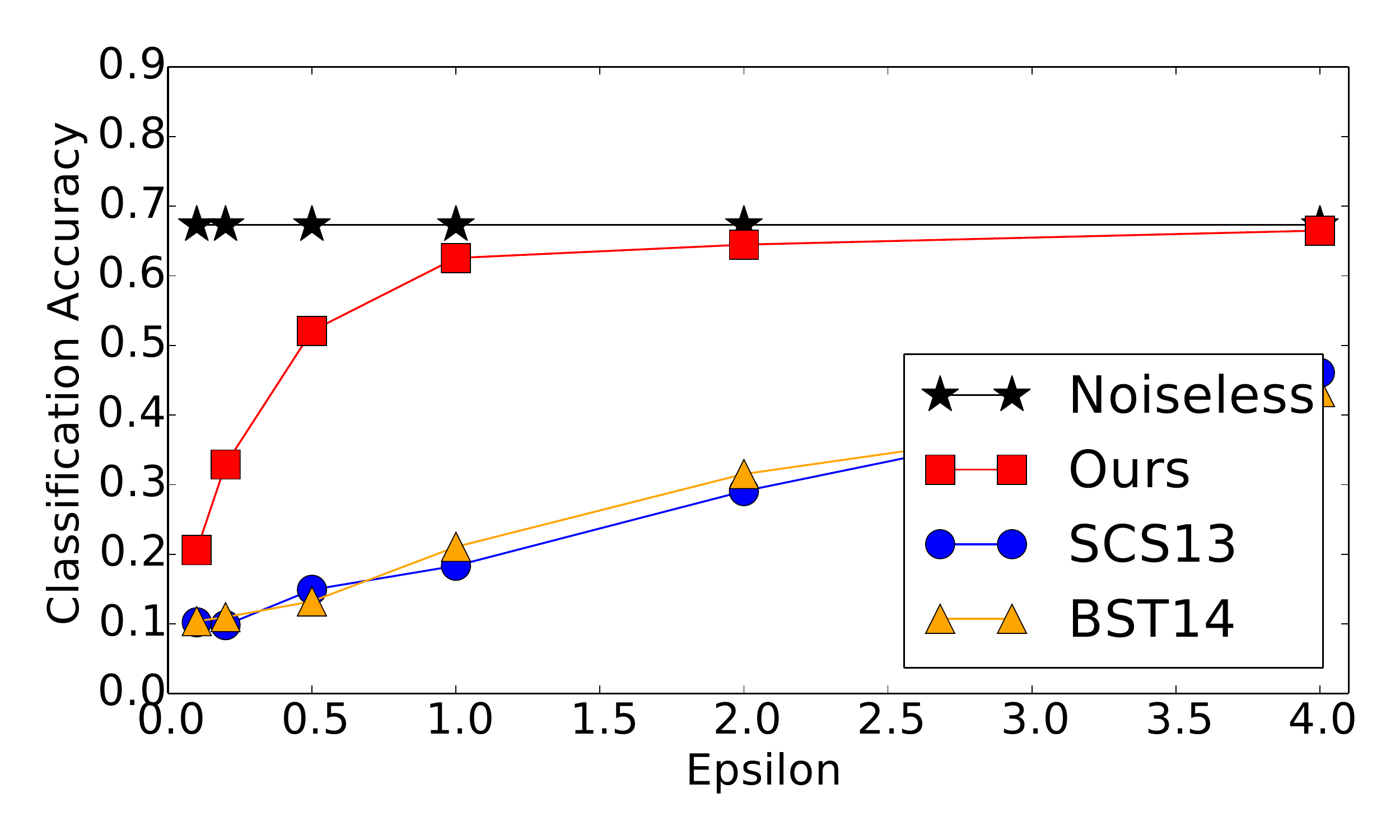}
  }
  \subfloat{
    \includegraphics[width=0.24\columnwidth]
    {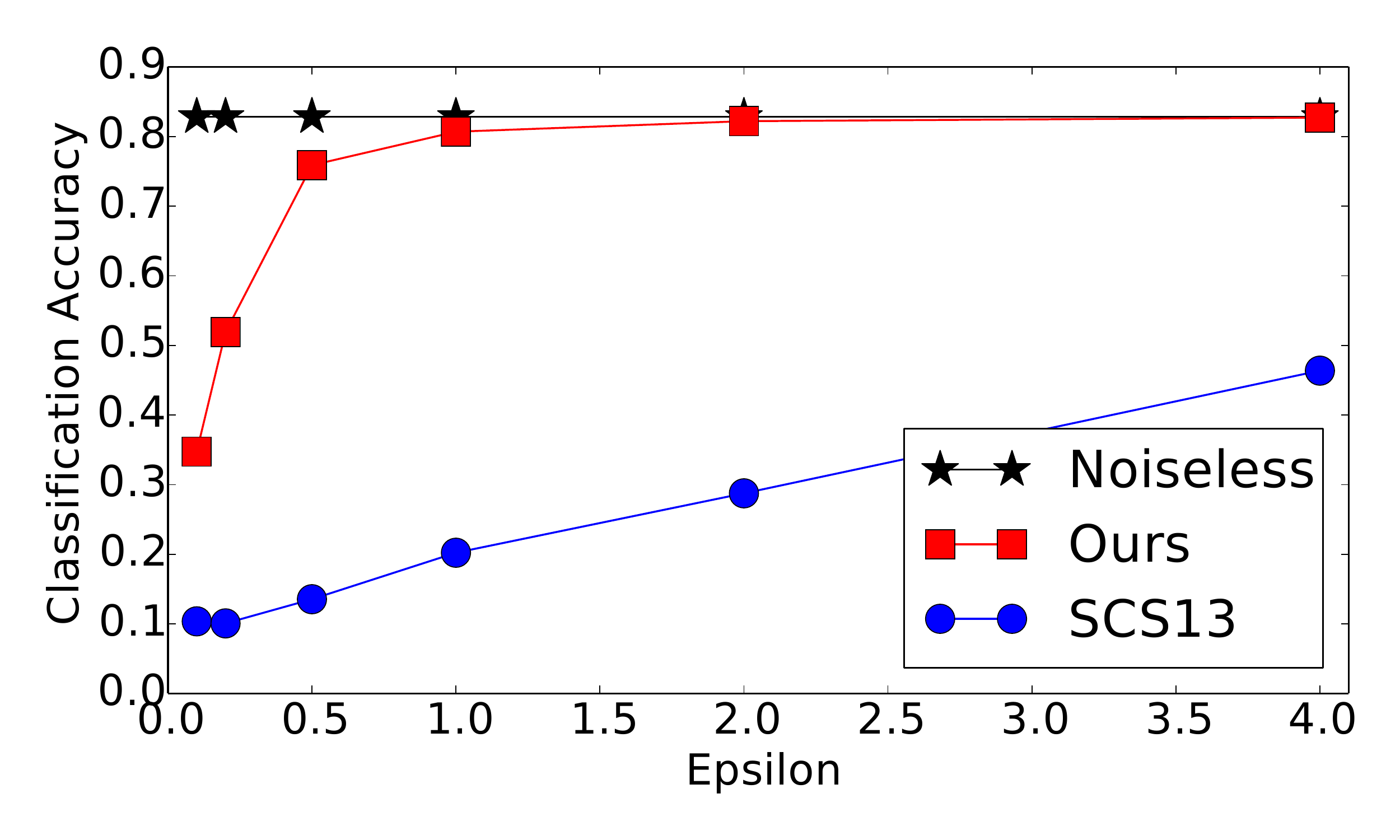}
  }
  \subfloat{
    \includegraphics[width=0.24\columnwidth]
    {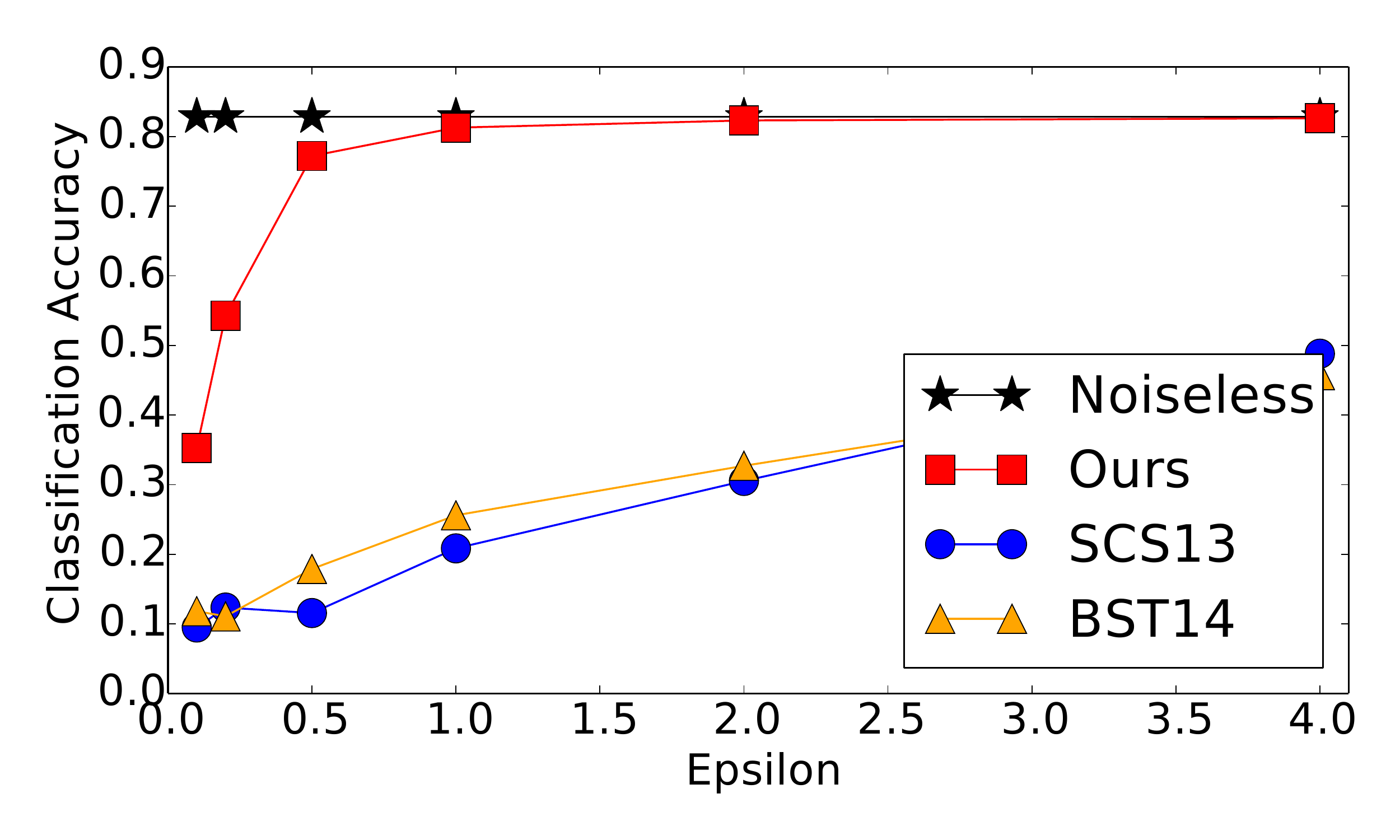}
  } \\[-3ex]
  \subfloat{
    \includegraphics[width=0.24\columnwidth]
    {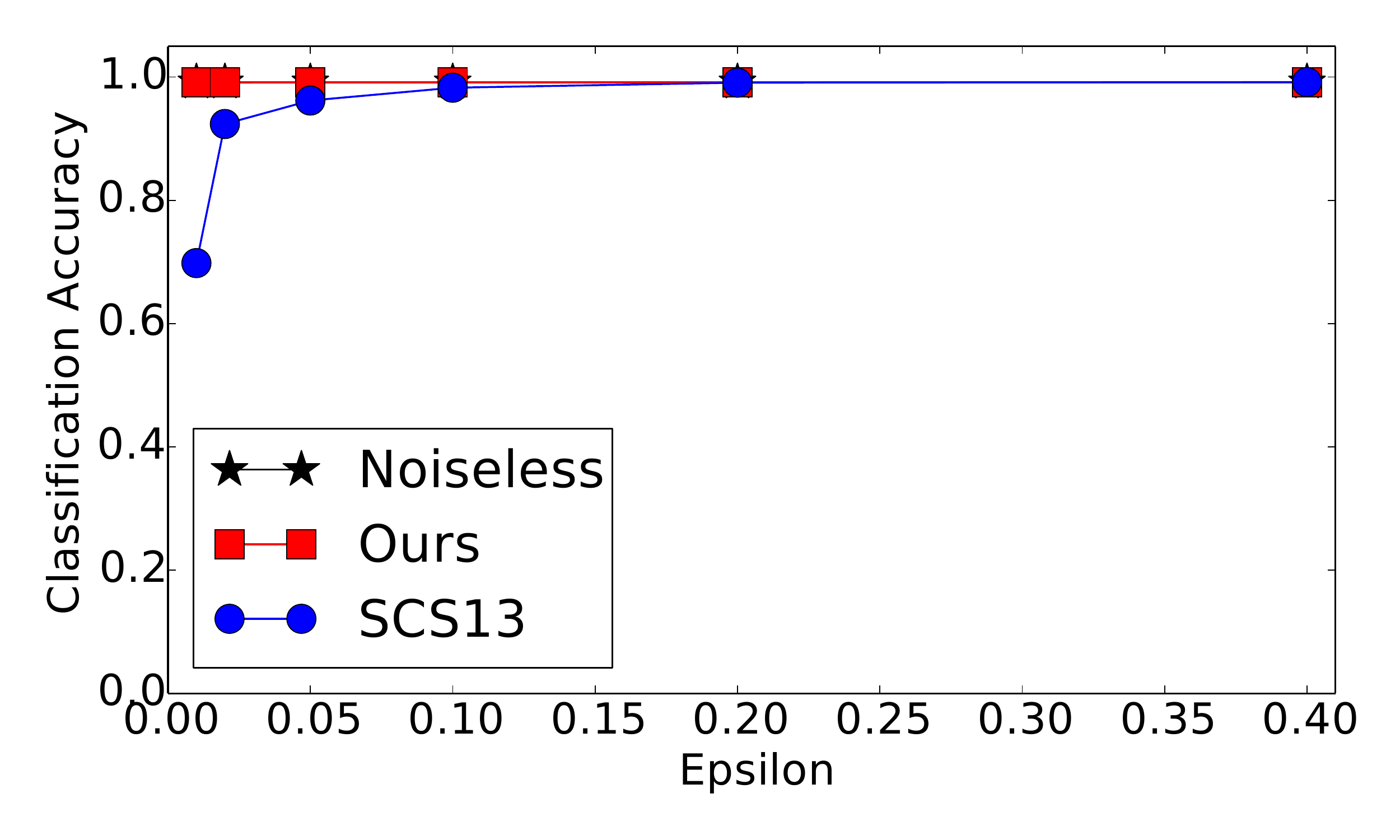}
  }
  \subfloat{
    \includegraphics[width=0.24\columnwidth]
    {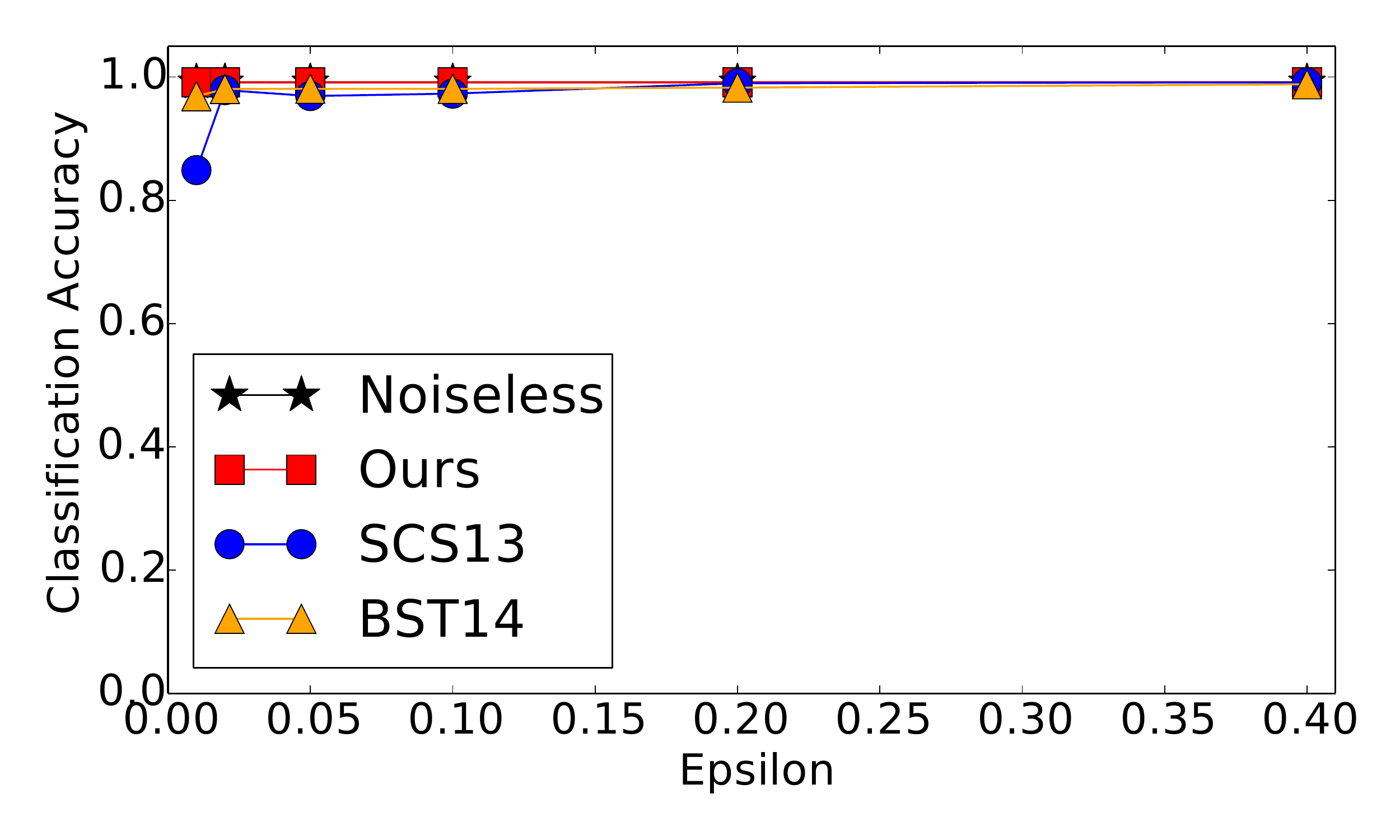}
  }
  \subfloat{
    \includegraphics[width=0.24\columnwidth]
    {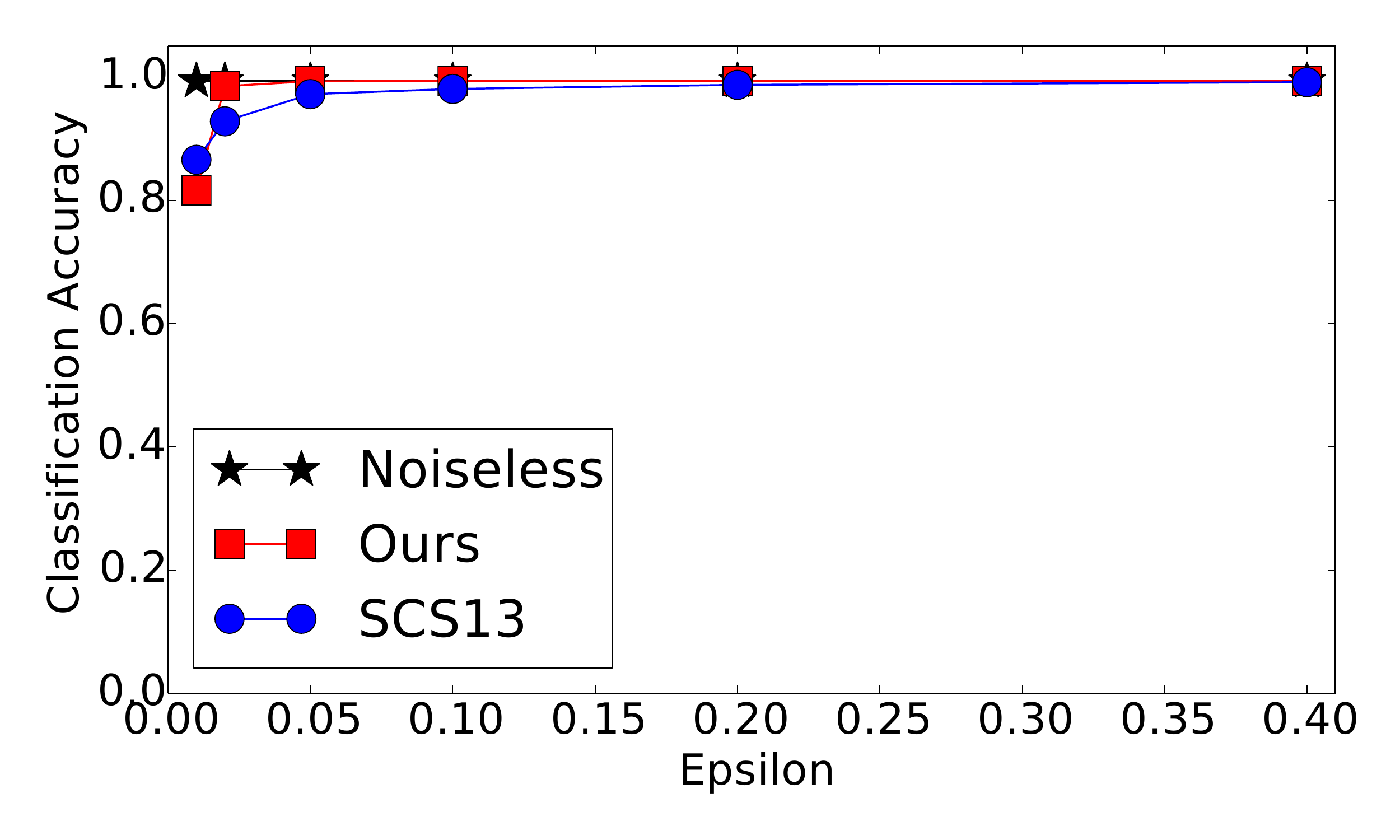}
  }
  \subfloat{
    \includegraphics[width=0.24\columnwidth]
    {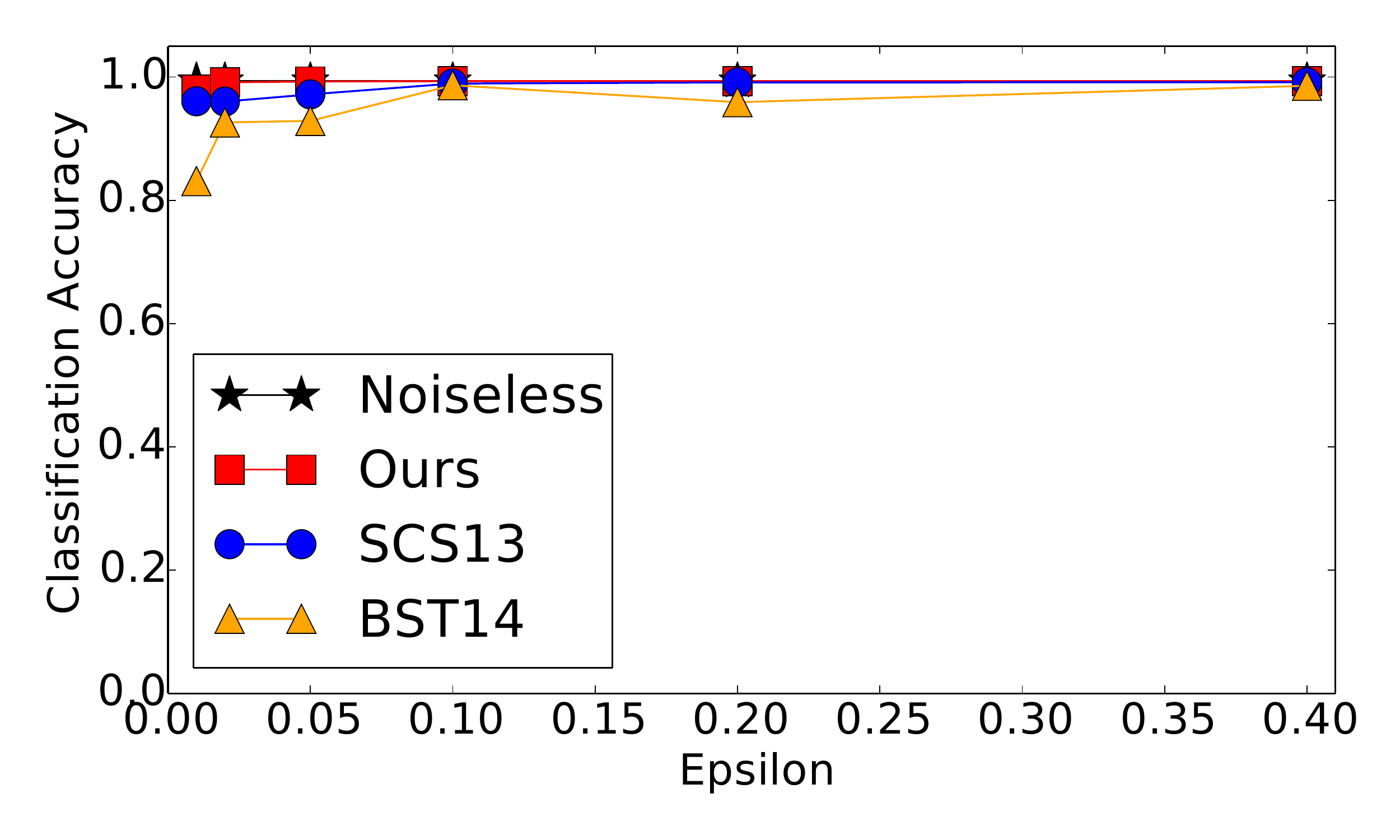}
  } \\[-3ex]
  \subfloat{
    \includegraphics[width=0.24\columnwidth]
    {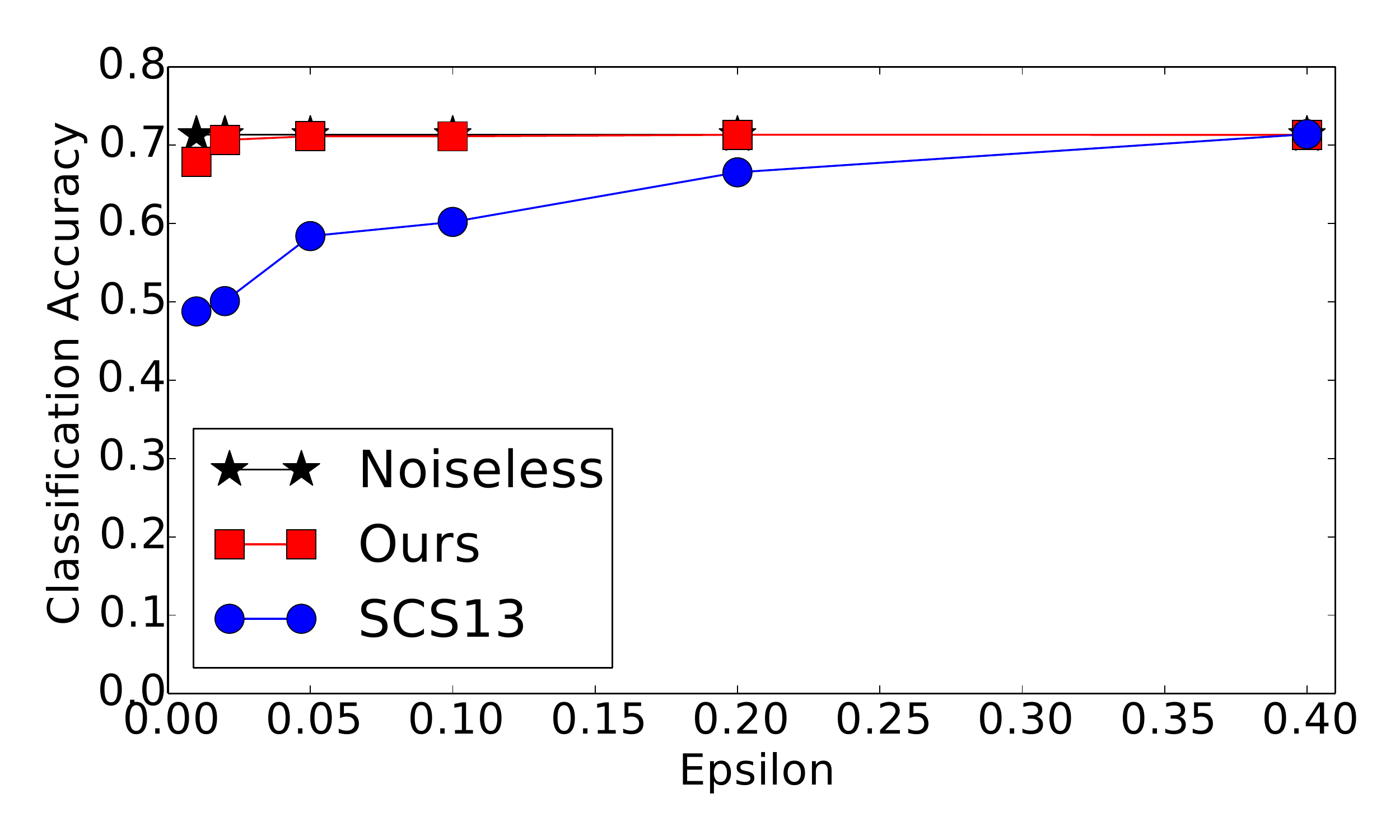}
  }
  \subfloat{
    \includegraphics[width=0.24\columnwidth]
    {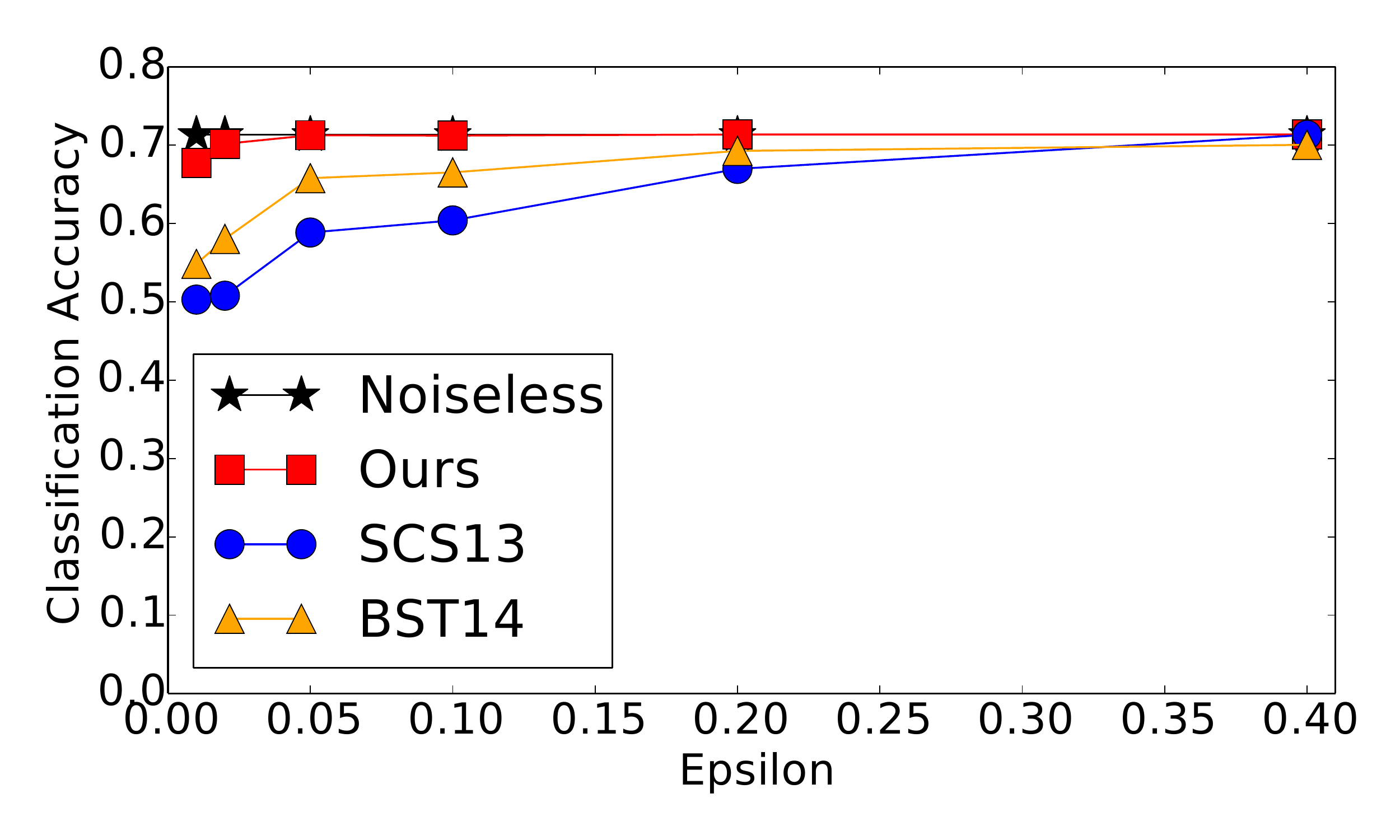}
  }
  \subfloat{
    \includegraphics[width=0.24\columnwidth]
    % {plots/covertype/test3_10_pass.pdf}
    {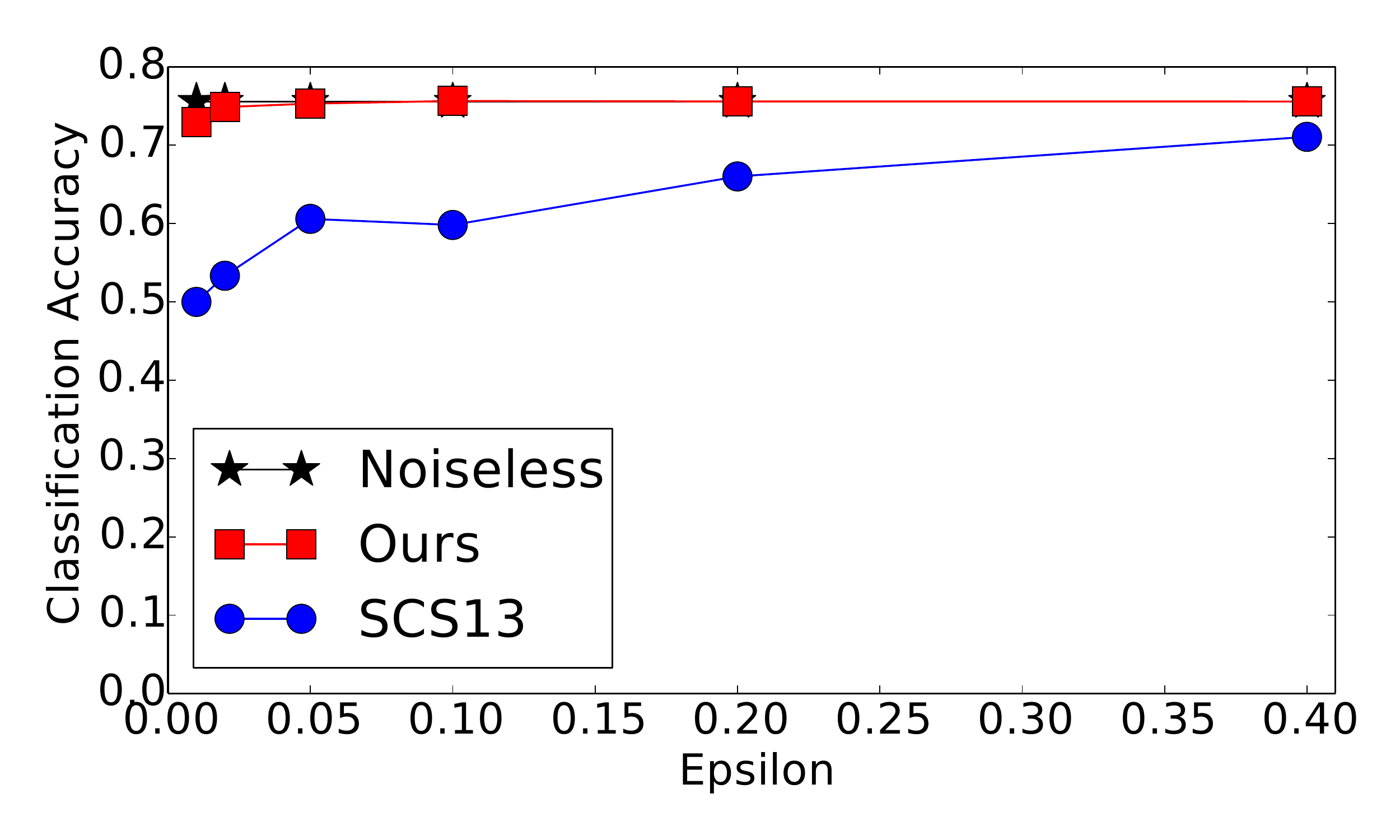}
  }
  \subfloat{
    \includegraphics[width=0.24\columnwidth]
    {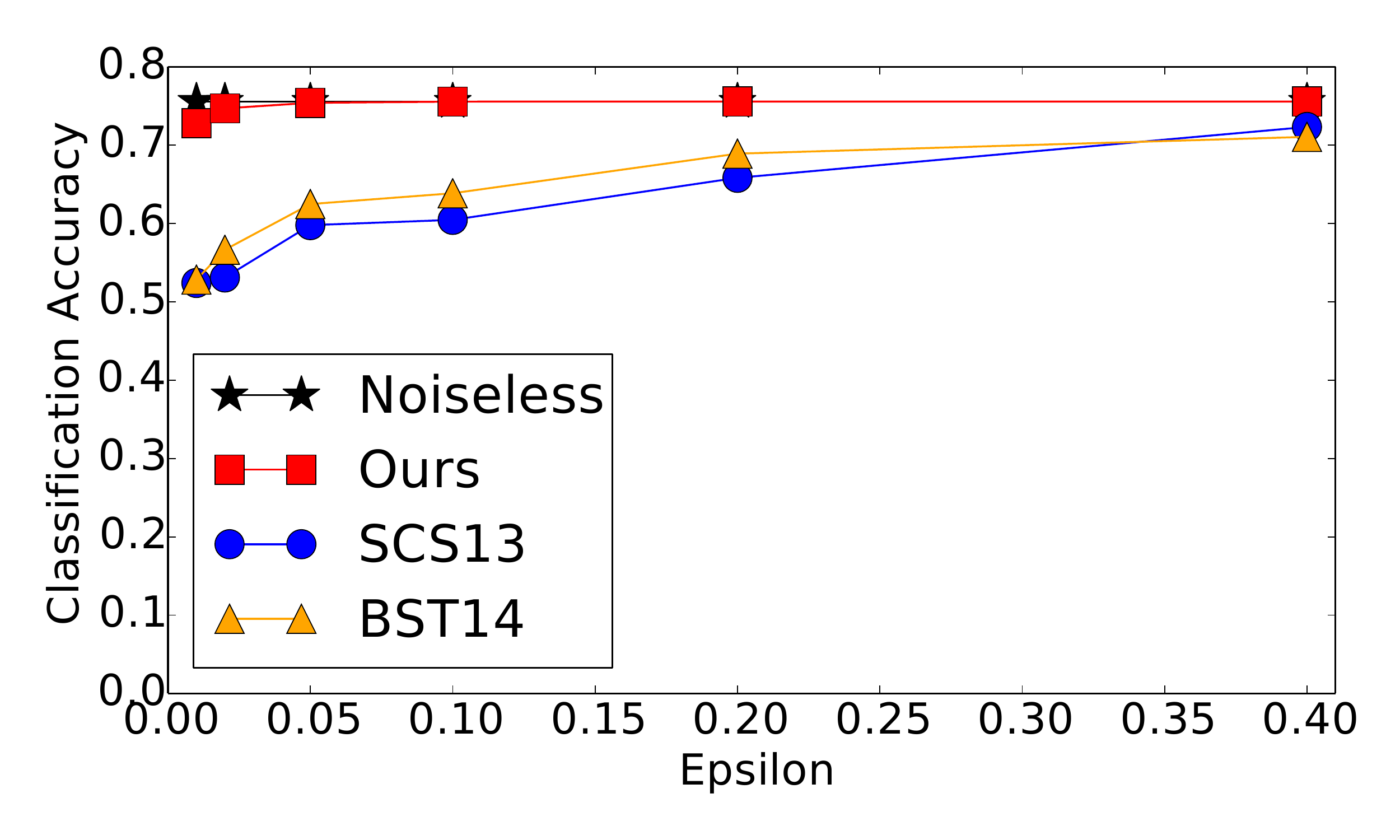}
  }
  \caption{
    {\bf Tuning using Public Data.}
    Row 1 is MNIST, row 2 is Protein and row 3 is Forest Covertype.
    Each row gives the test accuracy results of 4 tests:
    Test 1 is Convex, $\protect (\varepsilon, 0)$-DP,
    Test 2 is Convex, $\protect (\varepsilon, \delta)$-DP,
    Test 3 is Strongly Convex, $\protect (\varepsilon, 0)$-DP,
    and Test 4 is Strongly Convex, $\protect (\varepsilon, \delta)$-DP.
    For Test 1 and 3, we compare Noiseless, our algorithm and SCS13.
    For Test 2 and 4, we compare all four algorithms.
    The mini-batch size $b = 50$. For strongly convex optimization we set $\protect R = 1/\lambda$,
    otherwise we report unconstrained optimization for the convex case.
    Each point is the test accuracy of the model trained with $10$ passes and $\lambda = 0.0001$,
    where applicable.
  }
  \label{fig:notune:tests_50_mb_10_passes}
\end{figure*}
\noindent Finally, we report test accuracy and analyze the parameters.

\noindent\textbf{Test Accuracy using Public Data}.
Figure~\ref{fig:notune:tests_50_mb_10_passes} reports the test accuracy results
if one can tune parameters using public data.
{\em For all tests our algorithms give significantly better accuracy,
up to {\bf 4X} better than SCS13 and BST14.}
Besides better absolute performance, we note that our algorithms are more stable
in the sense that it converges more quickly to noiseless performance
(at smaller $\varepsilon$).

\noindent\textbf{Test Accuracy using a Private Tuning Algorithm}.
Figure~\ref{fig:accuracy:tests_50_mb_10_passes} gives the test accuracy results of
MNIST, Protein and Covertype for all 4 test scenarios using a private tuning algorithm
(Algorithm~\ref{alg:private-tuning}).
{\em For all tests we see that our algorithms give
  significantly better accuracy, up to {\bf 3.5X} better than BST14
  and up to {\bf 3X} better than SCS13.}

SCS13 and BST14 exhibit much better accuracy on Protein than on MNIST,
since logistic regression fits well to the problem. 
Specifically, BST14 has very close accuracy as our algorithms,
though our algorithms still consistently outperform BST14.
The accuracy of SCS13 decreases significantly with smaller $\varepsilon$.

For Covertype, even on this large dataset,
SCS13 and BST14 give much worse accuracy compared to ours.
The accuracy of our algorithms is close to the baseline at around $\varepsilon=0.05$.
The accuracy of SCS13 and BST14 slowly improves with more passes over the data.
Specifically, the accuracy of BST14 approaches the baseline only after $\varepsilon=0.4$.

\begin{figure*}[!htb]
  \centering
  \subfloat[Convex, Test 1]{
    \includegraphics[width=0.3\columnwidth]
    % {plots/mnist/convex_num_passes.pdf}
    {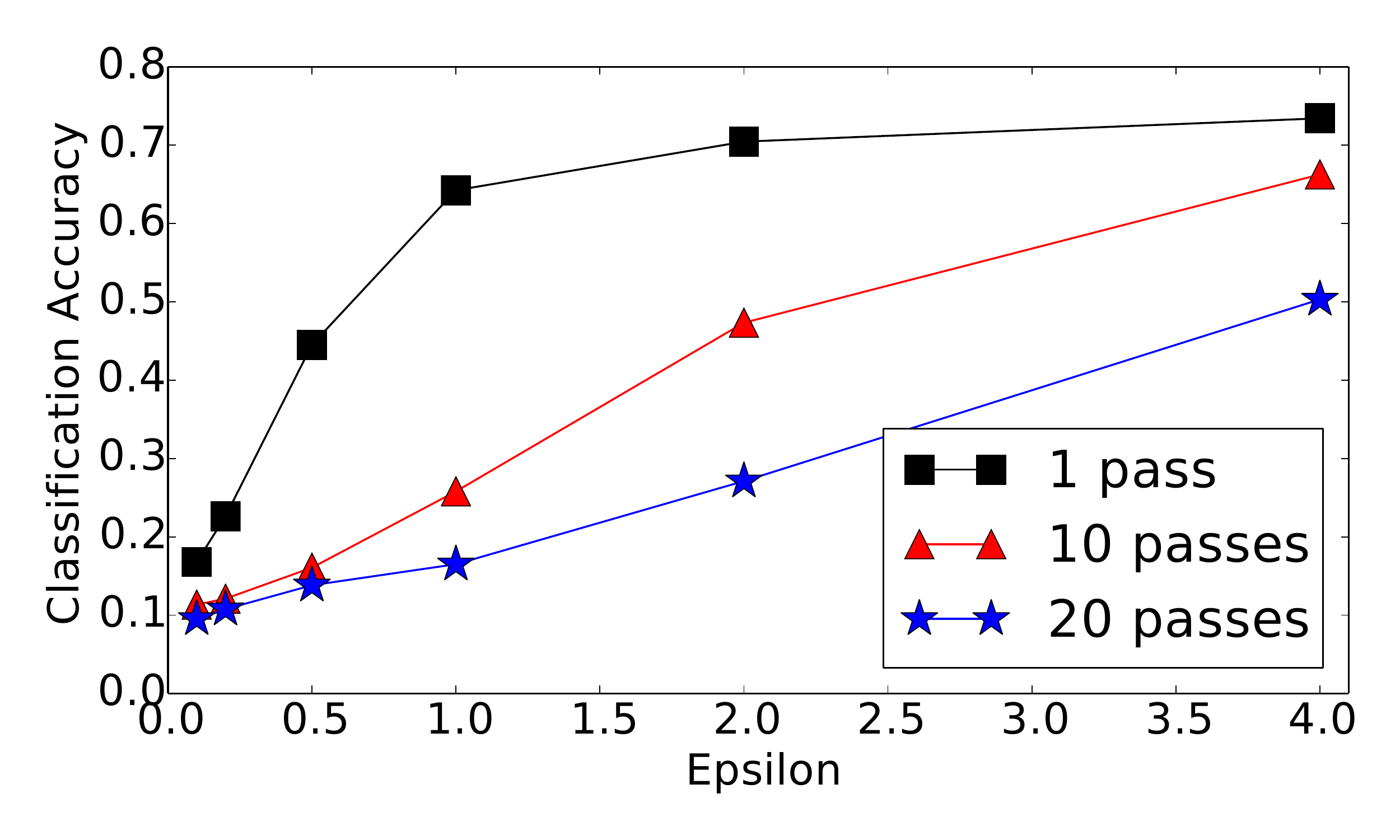}
  }
  \subfloat[Strongly Convex, Test 3]{
    \includegraphics[width=0.3\columnwidth]
    % {plots/mnist/strongly_convex_num_passes.pdf}
    {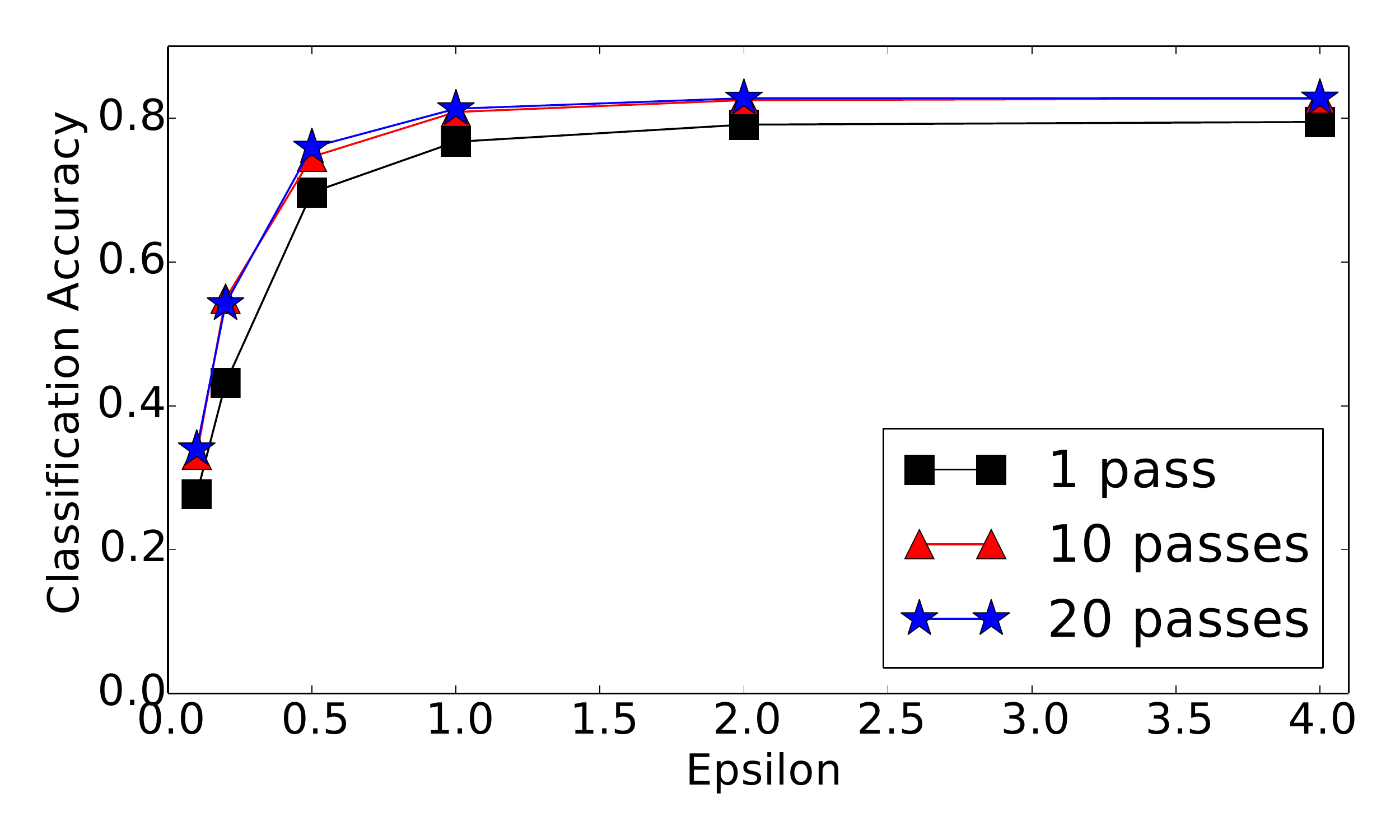}
  }
  \subfloat[Convex, Test 1]{
    \includegraphics[width=0.3\columnwidth]
    % {plots/mnist/convex_mb.pdf}
    {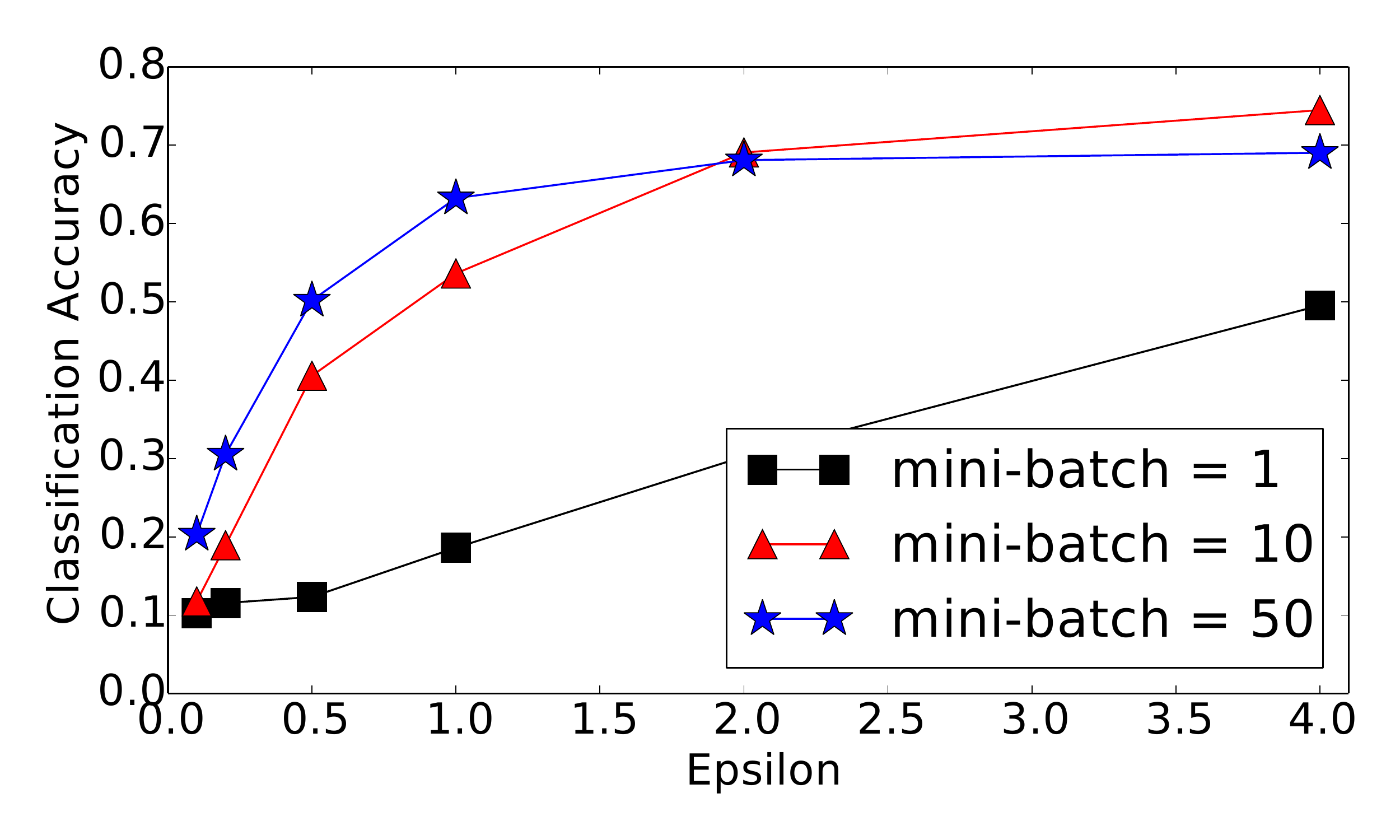}
  }
  \caption{{\bf (a), (b) The effect of number of passes}:
    We report the results on MNIST dataset.
    We contrast Test 1 (Convex $\protect\varepsilon$-DP) using mini-batch size 1,
    with Test 3 (Strongly Convex $\protect\varepsilon$-DP) using mini-batch size 50.
    In the former case, more passes through the data introduces more noise due to
    privacy and thus results in worse test accuracy. In the latter case,
    more passes improves the test accuracy as it helps convergence
    while no more noise is needed for privacy. {\bf (c) The effect of mini-batch size}.
    We run again Test 1 with 20 passes through the data, and vary mini-batch size in $\protect\{1, 10, 50\}$.
    As soon as mini-batch size increases to 10 the test accuracy drastically
    improves from 0.45 to 0.71.}
  \label{fig:parameters}
\end{figure*}
\begin{figure*}[!htb]
  \centering
  \subfloat{
    \includegraphics[width=0.3\columnwidth]
    {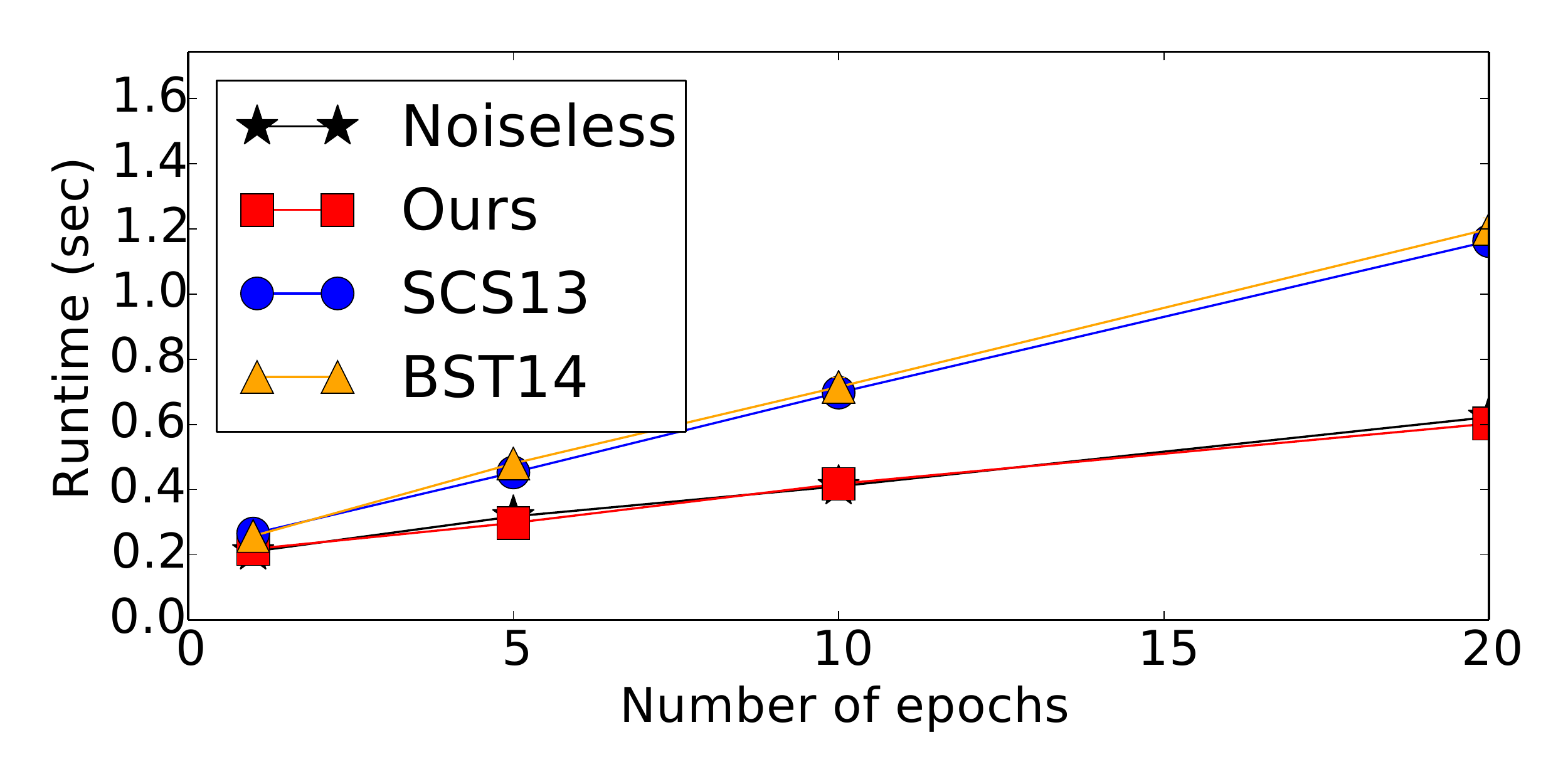}
  }
  \subfloat{
    \includegraphics[width=0.3\columnwidth]
    {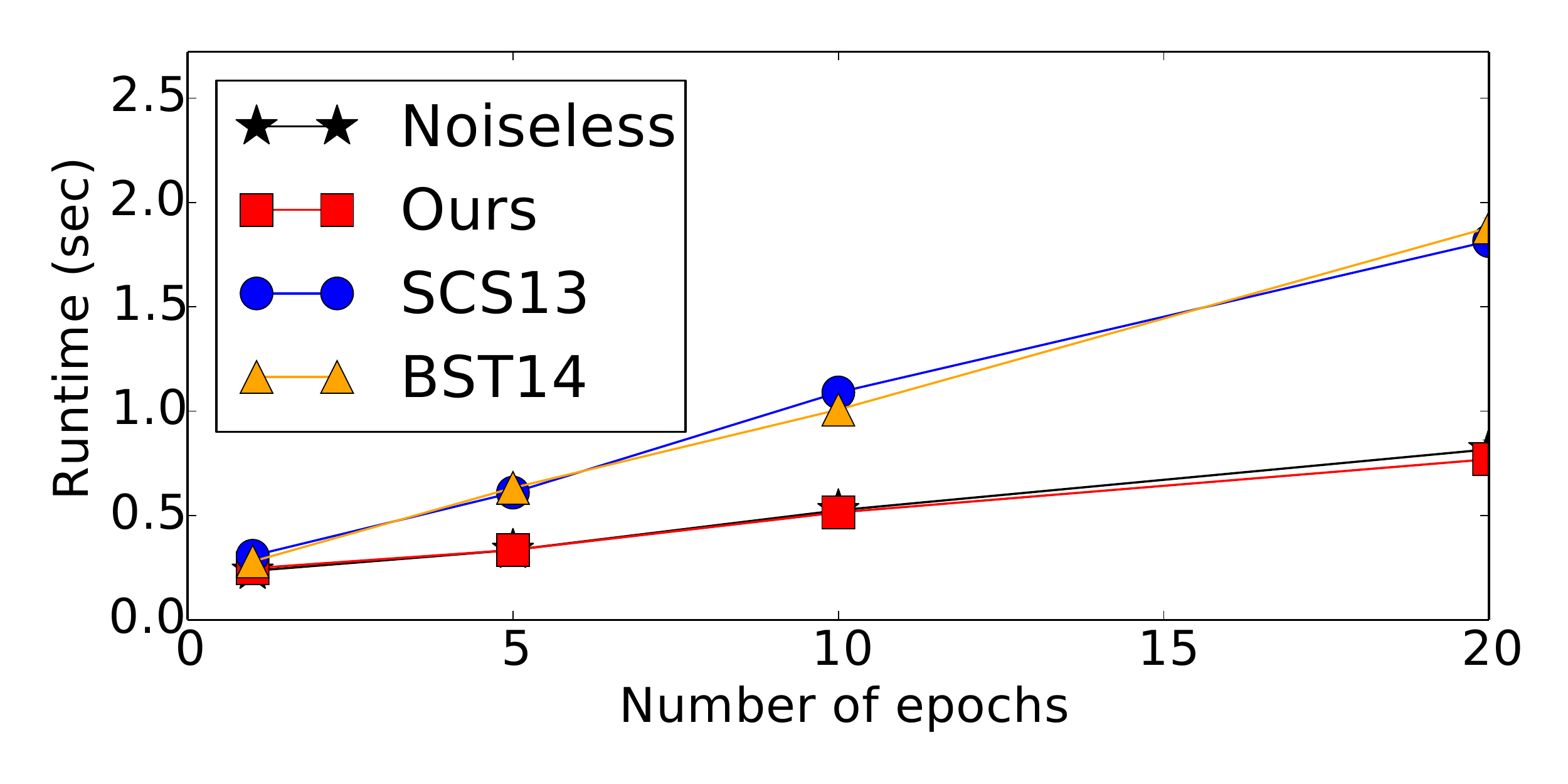}
  }
  \subfloat{
    \includegraphics[width=0.3\columnwidth]
    {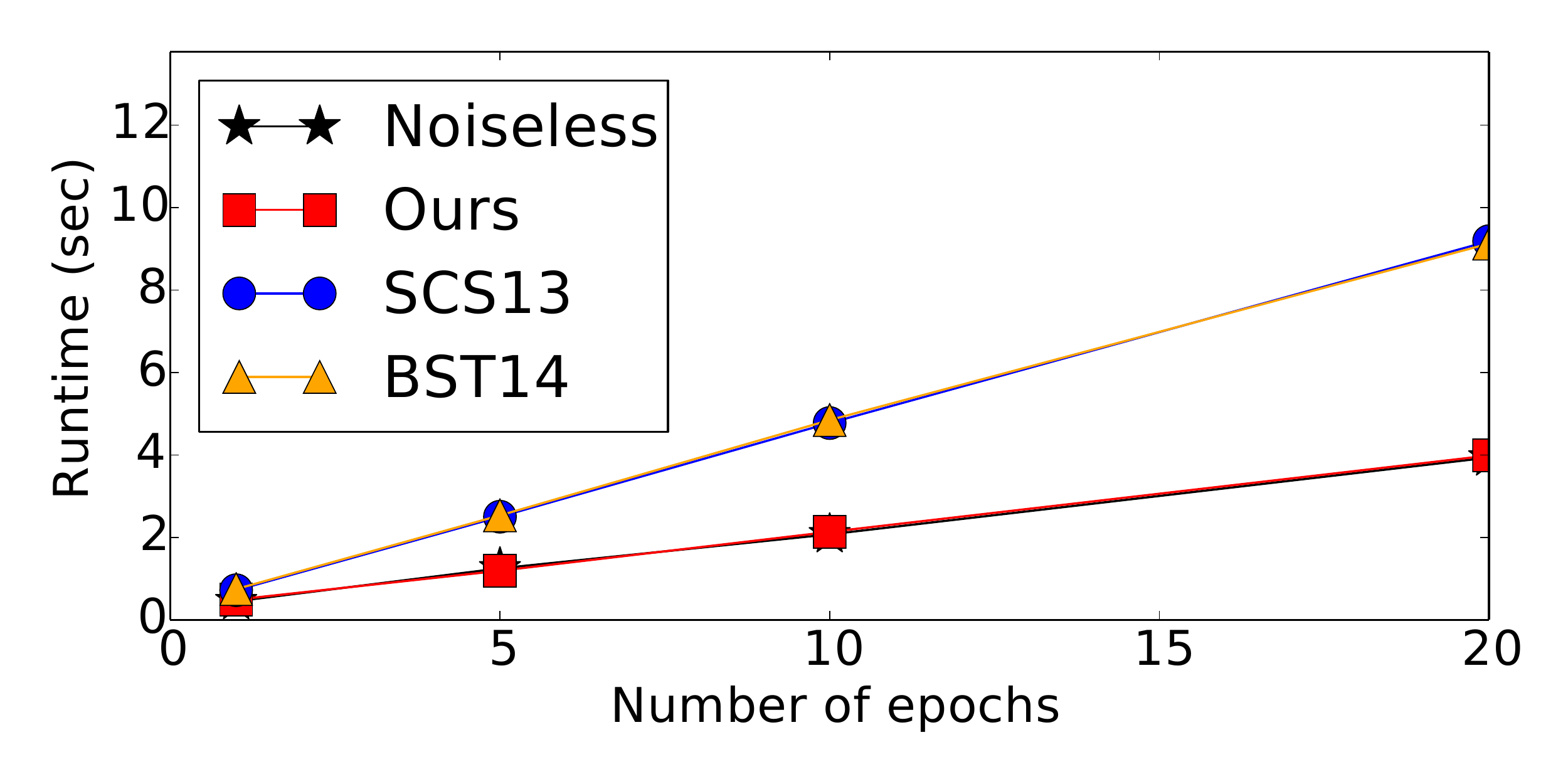}
  } \\[-3ex]
  \subfloat{
    \includegraphics[width=0.3\columnwidth]
    {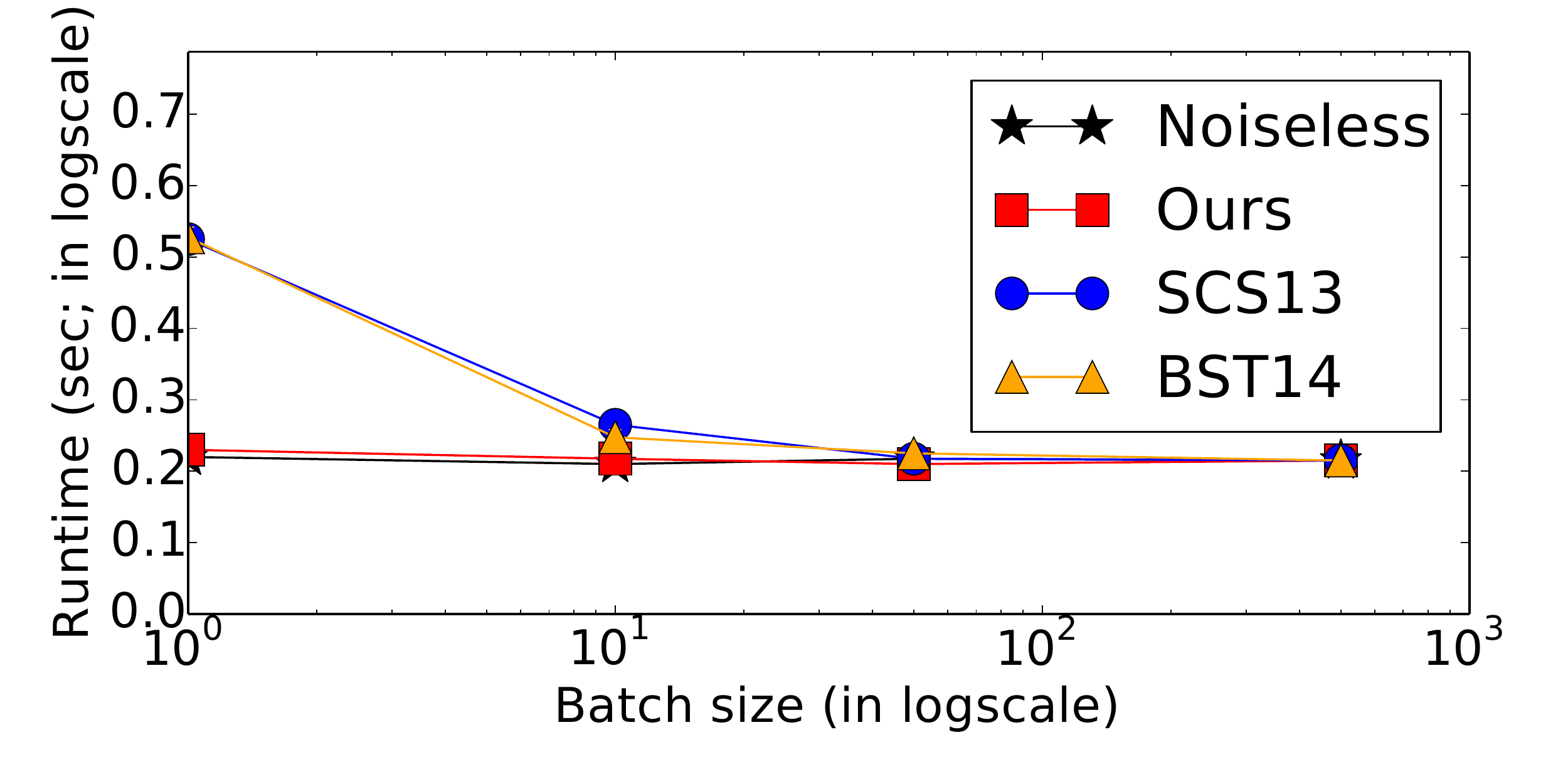}
  }
  \subfloat{
    \includegraphics[width=0.3\columnwidth]
    {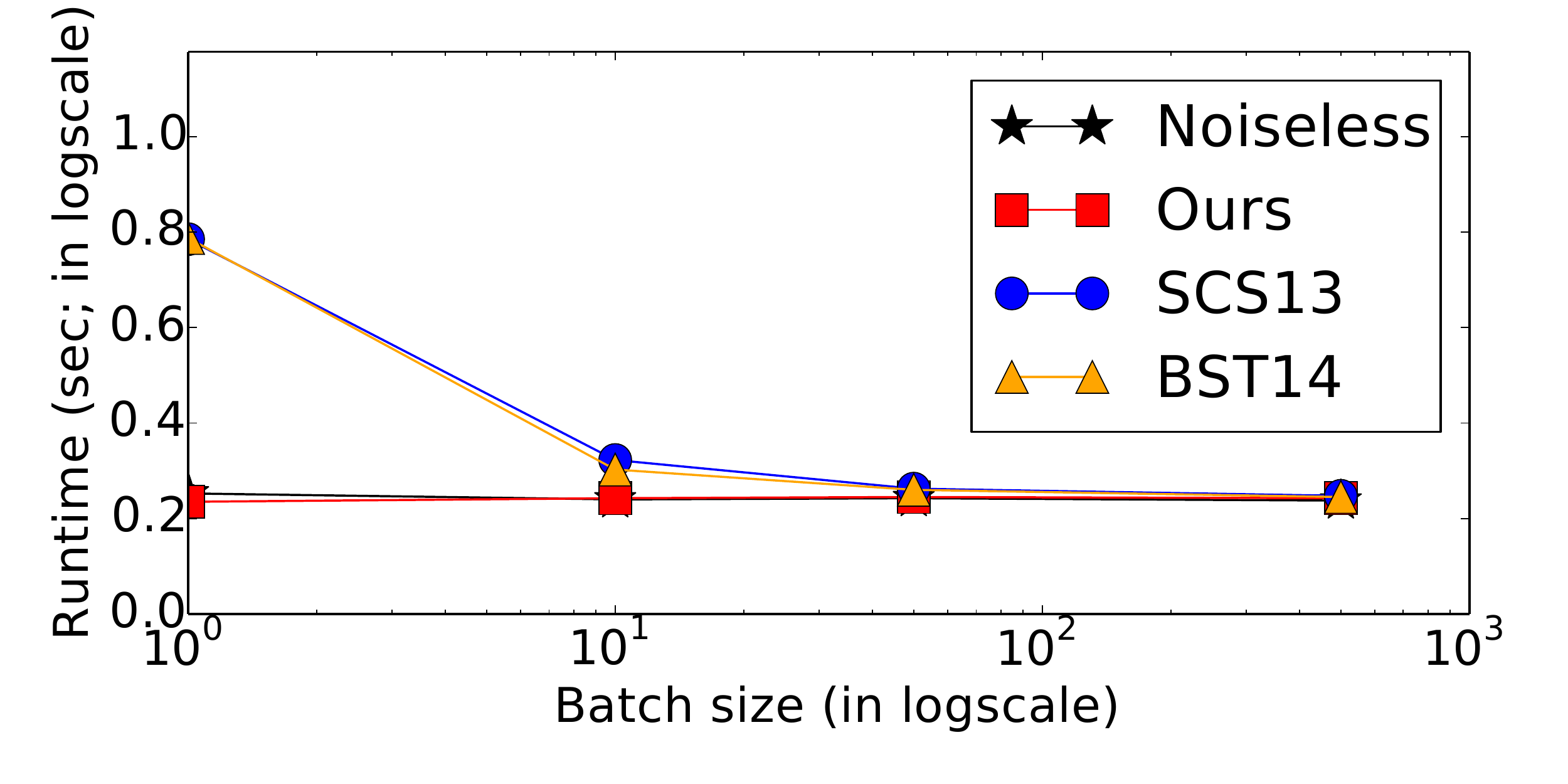}
  }
  \subfloat{
    \includegraphics[width=0.3\columnwidth]
    {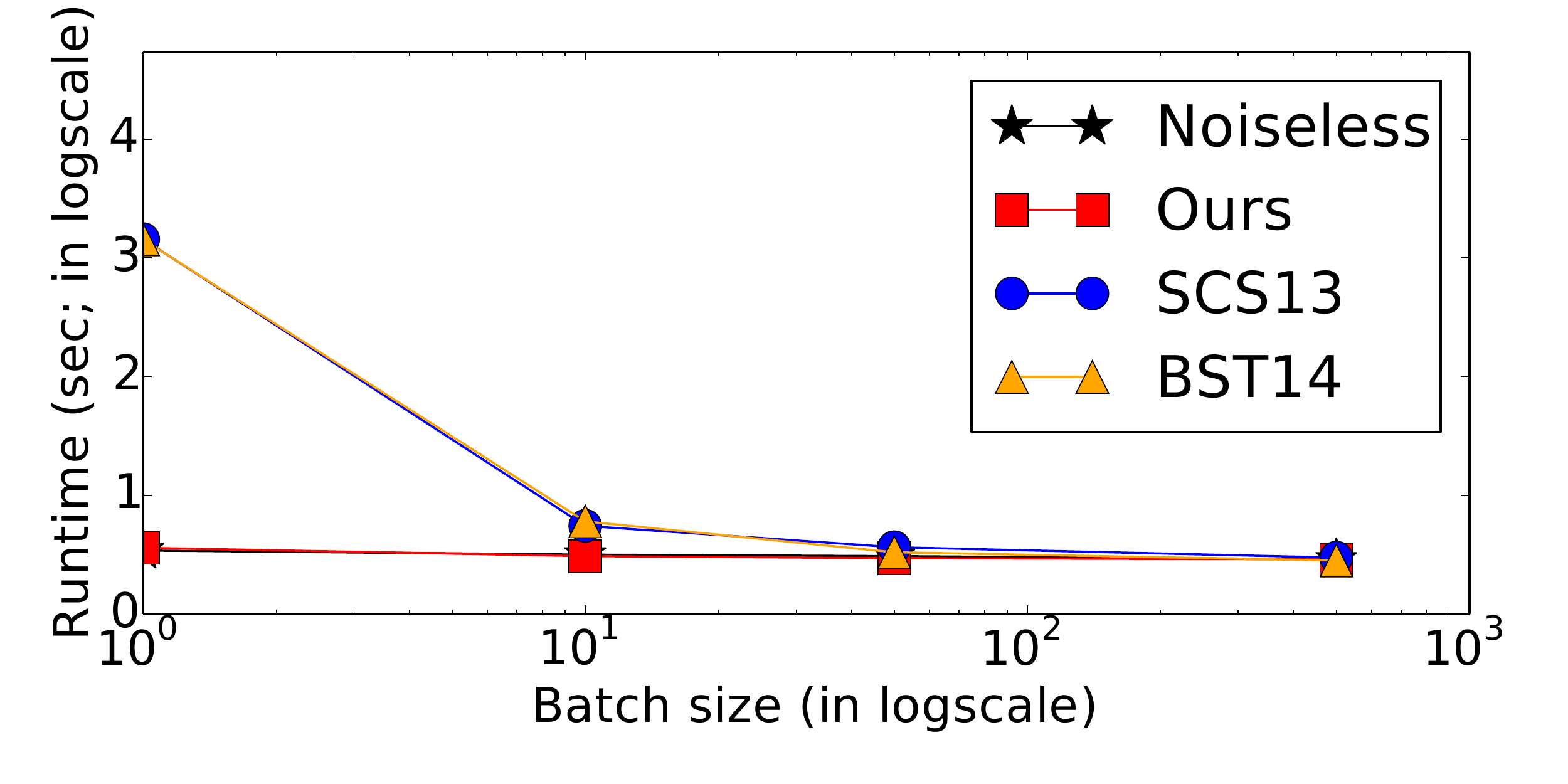}
  }
  \caption{
    {\bf Runtime of the implementations on \protect\Bismarck{}.}
    Row 1 gives the runtime results of varying the number of epochs with mini-batch size $=\protect10$,
    on MNIST, Protein and Forest Covertype, respectively.
    Row 2 gives the runtime results of Varying the mini-batch size for a single epoch,
    on MNIST, Protein and Forest Covertype, respectively.
    Only the results of Strongly Convex, $\protect (\varepsilon, \delta)$-DP are
    reported, and other settings have very similar trends.
    Noiseless is the regular mini-batch SGD in \protect\Bismarck.
    We fix  $\protect \varepsilon = 0.1$.
  }
  \label{fig:runtimes}
\end{figure*}

\noindent\textbf{Number of Passes (Epochs)}.
In the case of convex optimization, more passes through the data indicates larger
noise magnitude according to our theory results. This translates empirically to
worse test accuracy as we perform more passes.
Figure~\ref{fig:parameters} (a) reports test accuracy in the convex case
as we run our algorithm $1$ pass, $10$ passes and $20$ passes through the MNIST data.
The accuracy drops from 0.71 to 0.45 for $\varepsilon =  4.0$.
One should contrast this with results reported
in Figure~\ref{fig:parameters} (b) where doing more passes actually
{\em improves} the test accuracy.
This is because in the strongly convex case more passes will not introduce more noise
for privacy while it can potentially improve the convergence.

\noindent\textbf{Mini-batch Sizes}. We find that slightly enlarging the mini-batch size
can effectively reduce the noise and thus allow the private algorithm to run more
passes in the convex setting. This is useful since it is very common in practice to
adopt a mini-batch size at around 10 to 50. To illustrate the effect of mini-batch size
we consider the same test as we did above for measuring the effect of number of passes:
We run Test 1 with $20$ passes through the data, but vary mini-batch sizes in
$\{1, 10, 50\}$. Figure~\ref{fig:parameters} (c) reports the test
accuracy for this experiment: As soon as we increase mini-batch size
to $10$ the test accuracy already improves drastically from 0.45 to 0.71.

% Logistic regression with private tuning.
\begin{figure*}[!htb]
  \centering
  \subfloat{
    \includegraphics[width=0.24\columnwidth]
    % {plots/mnist/test1_10_pass_50_mb.pdf}
    {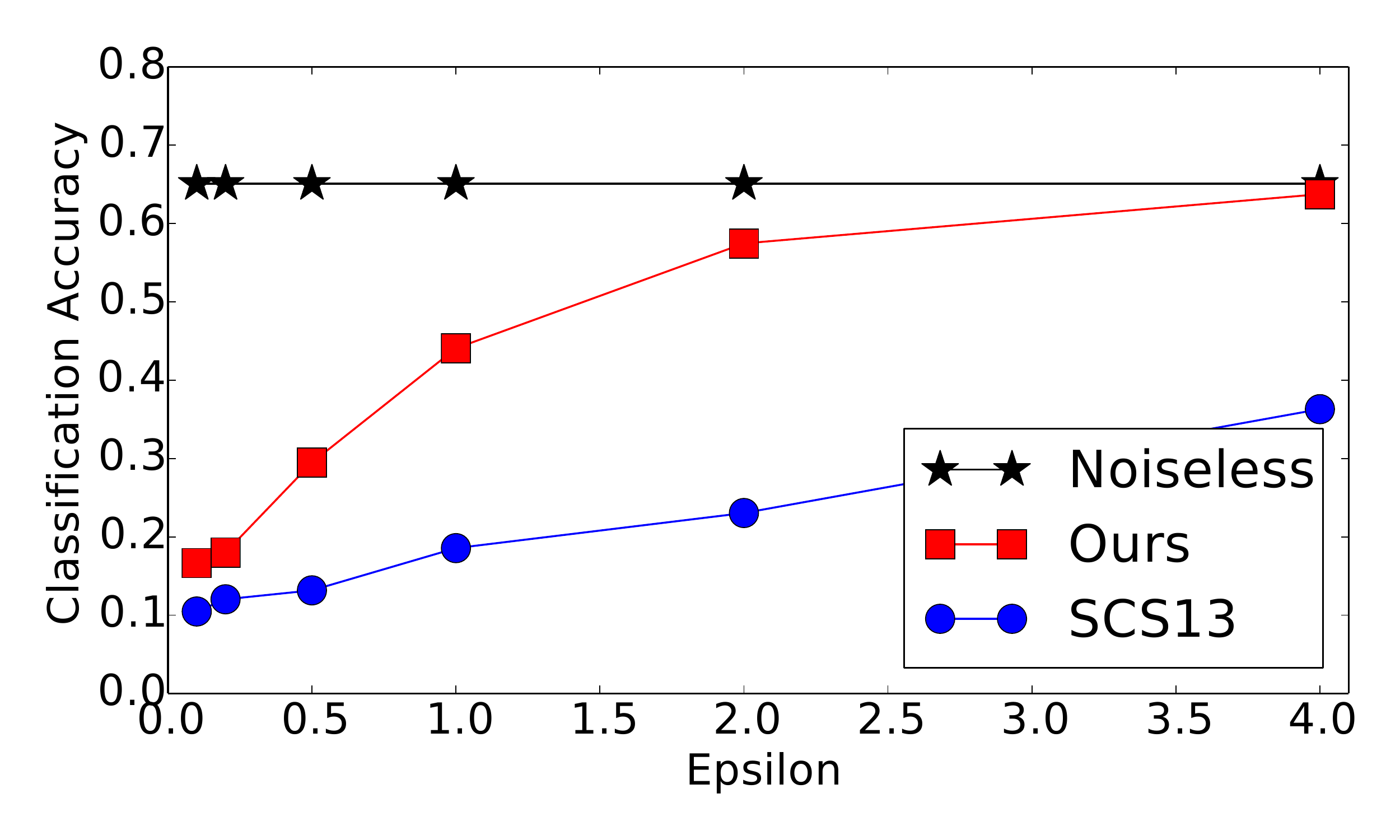}
  }
  \subfloat{
    \includegraphics[width=0.24\columnwidth]
    % {plots/mnist/test2_10_pass_50_mb.pdf}
    {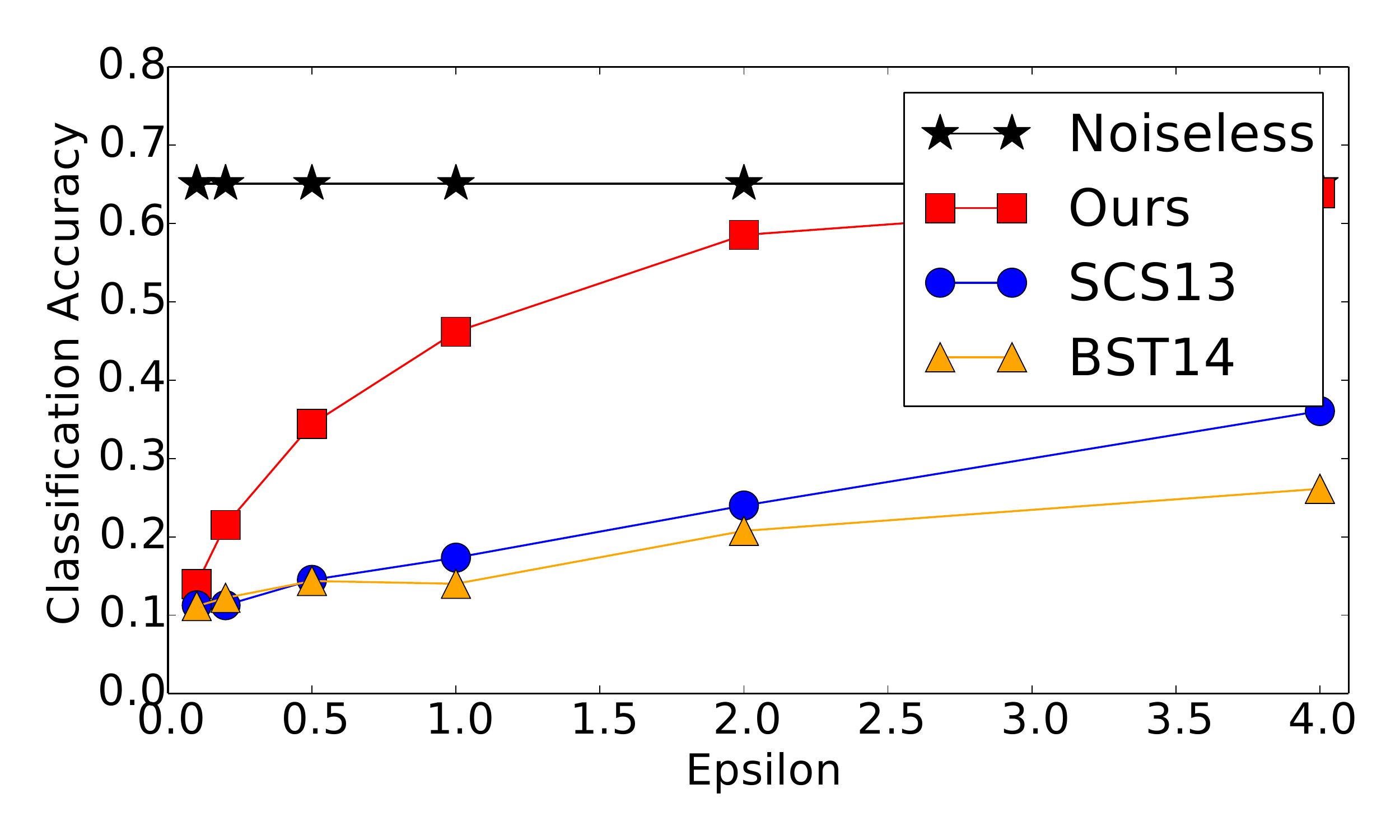}
  }
  \subfloat{
    \includegraphics[width=0.24\columnwidth]
    % {plots/mnist/test3_10_pass_50_mb.pdf}
    {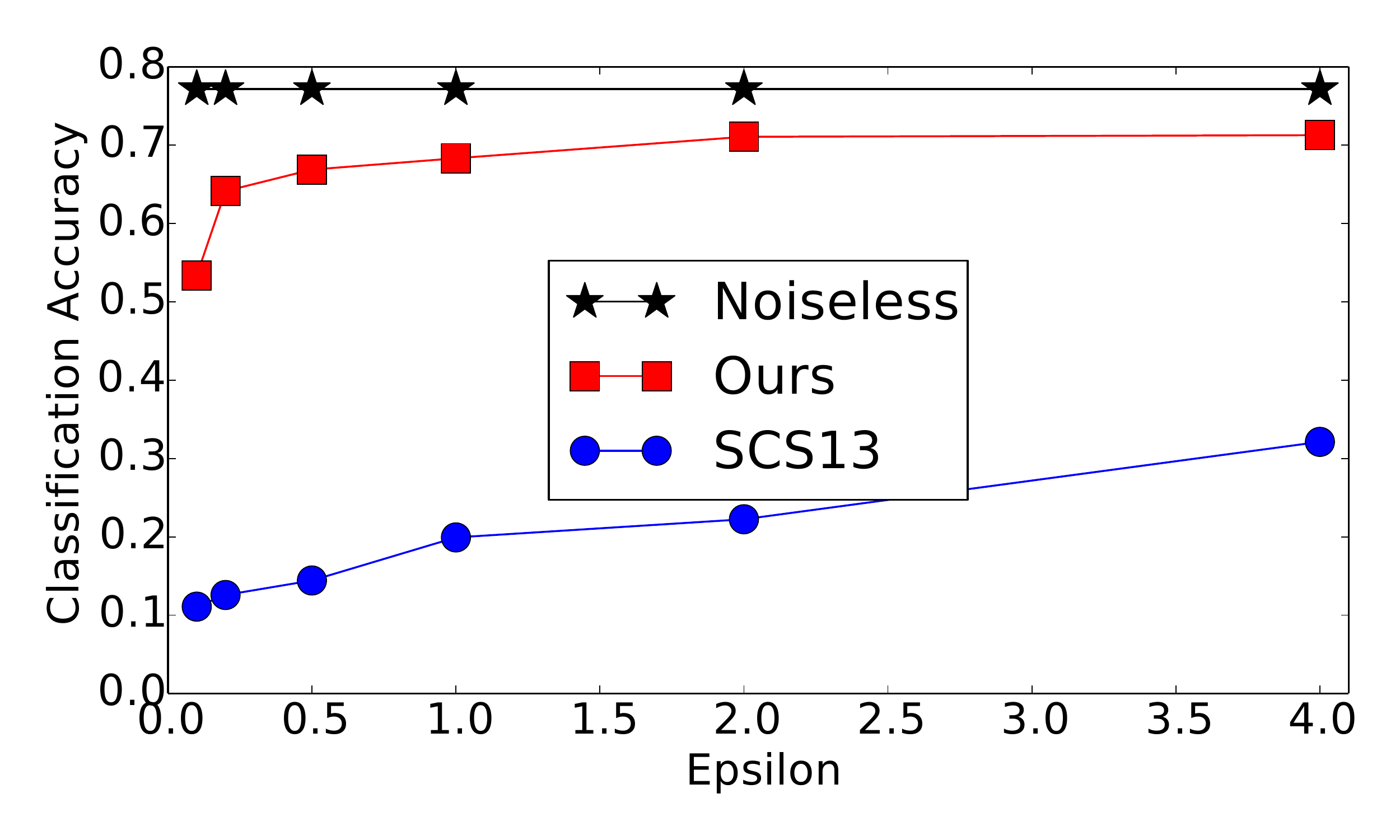}
  }
  \subfloat{
    \includegraphics[width=0.24\columnwidth]
    % {plots/mnist/test4_10_pass_50_mb.pdf}
    {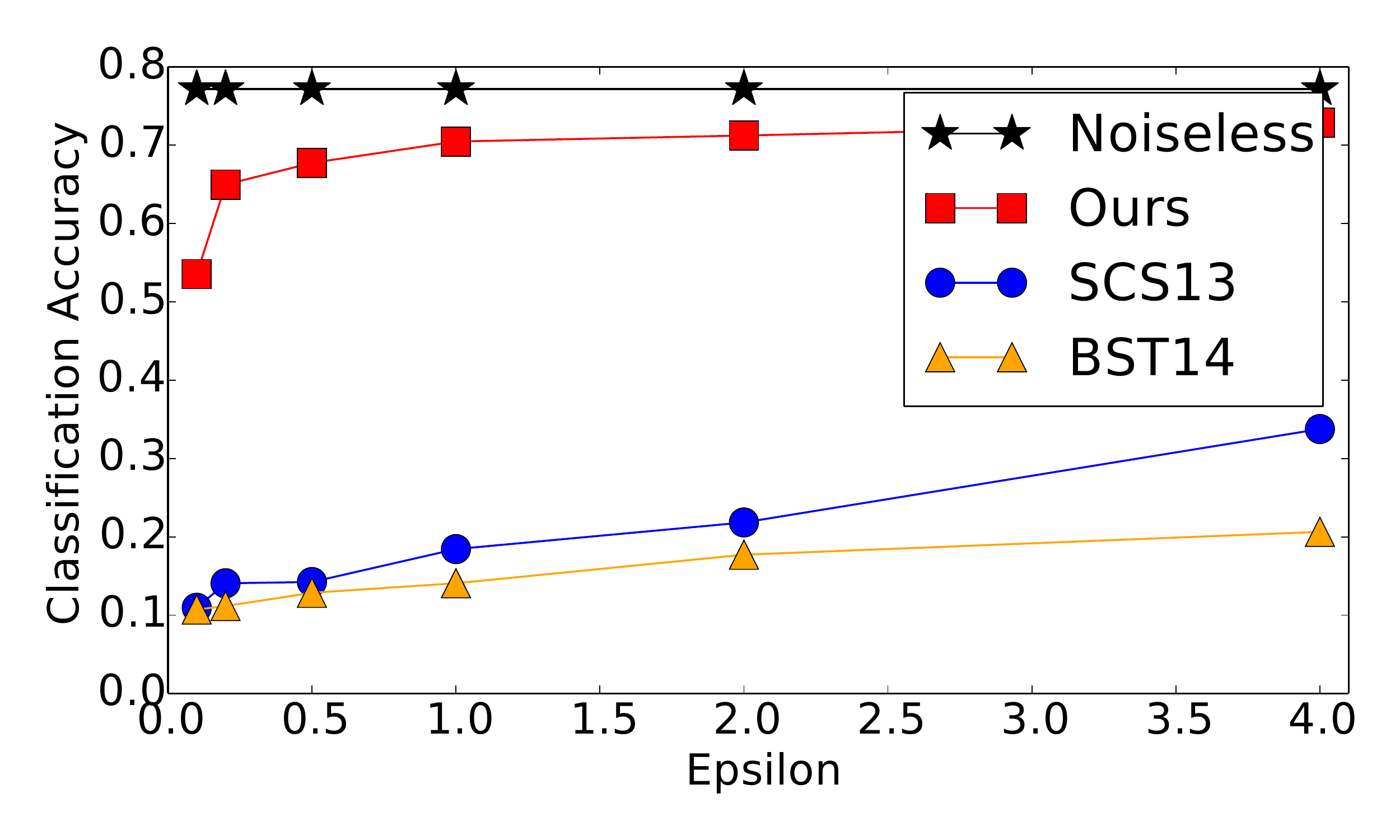}
  } \\[-3ex]
  \subfloat{
    \includegraphics[width=0.24\columnwidth]
    % {plots/protein/test1_10_pass.pdf}
    {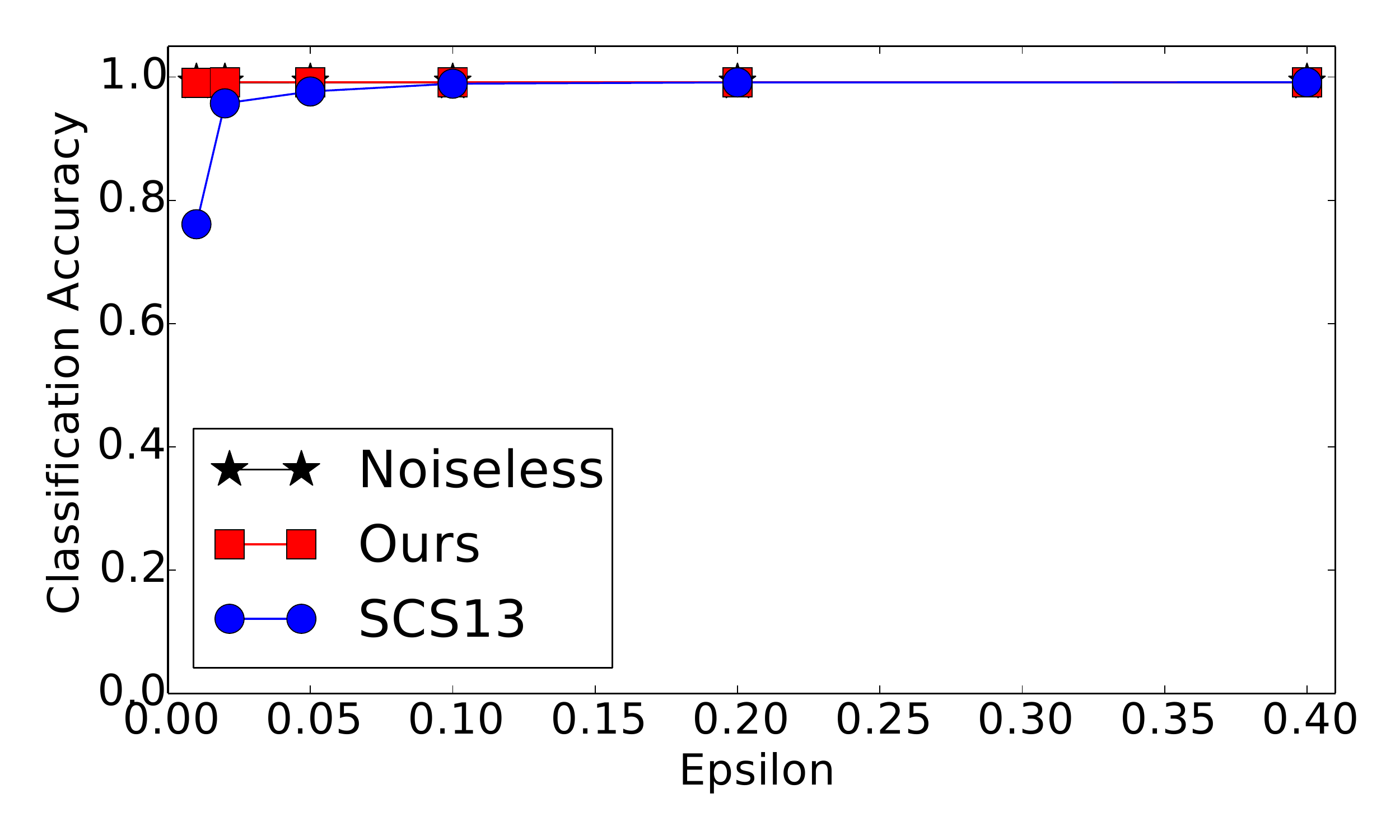}
  }
  \subfloat{
    \includegraphics[width=0.24\columnwidth]
    % {plots/protein/test2_10_pass.pdf}
    {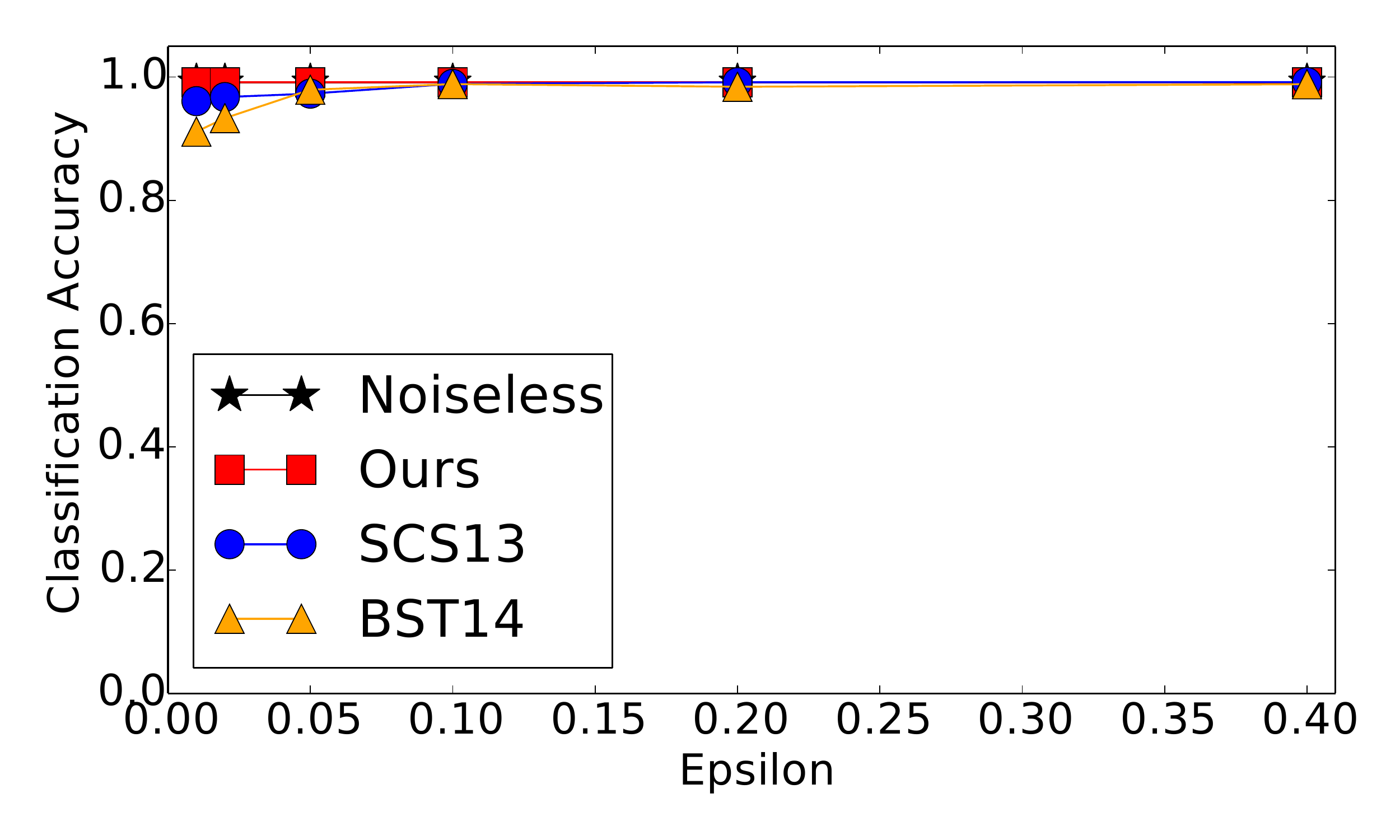}
  }
  \subfloat{
    \includegraphics[width=0.24\columnwidth]
    % {plots/protein/test3_10_pass.pdf}
    {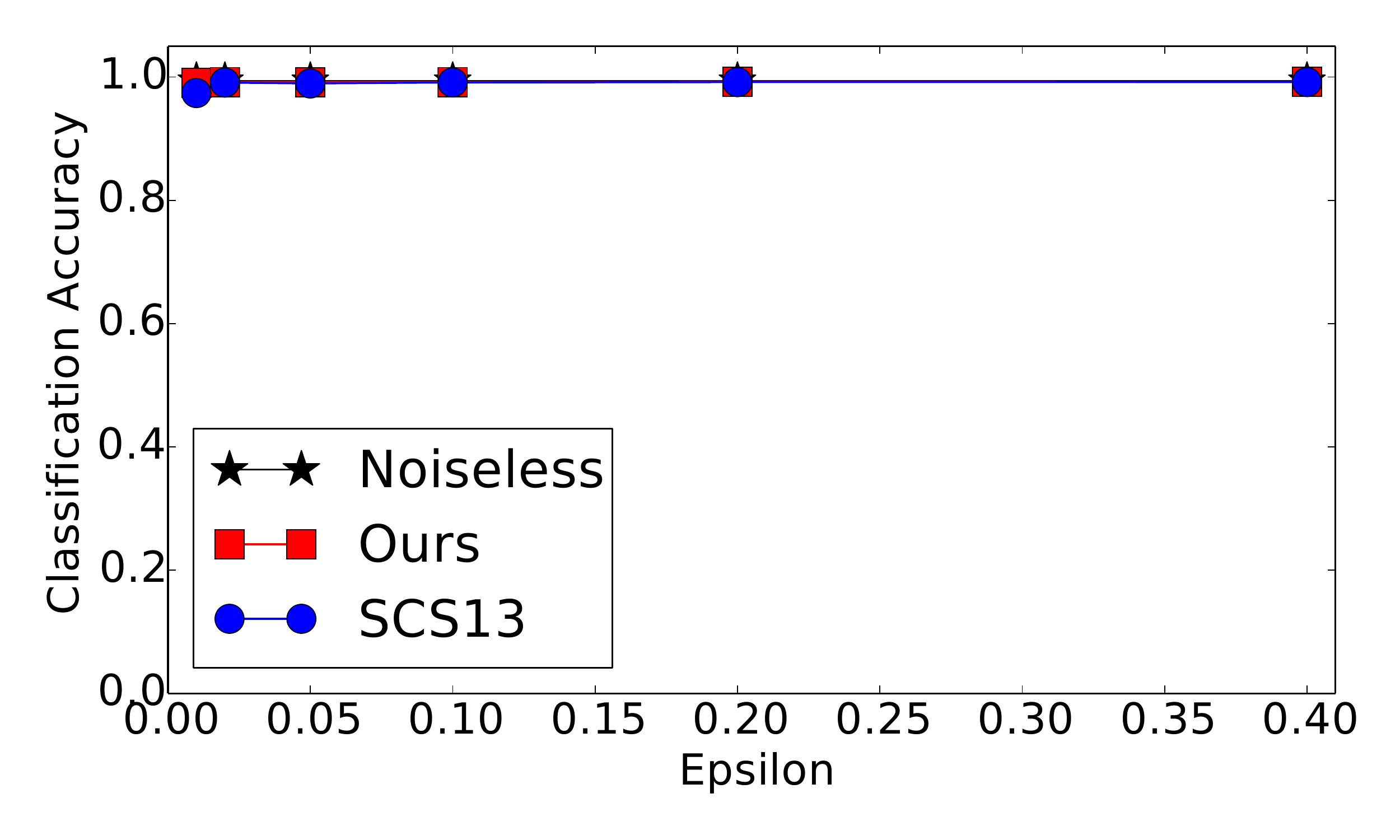}
  }
  \subfloat{
    \includegraphics[width=0.24\columnwidth]
    % {plots/protein/test4_10_pass.pdf}
    {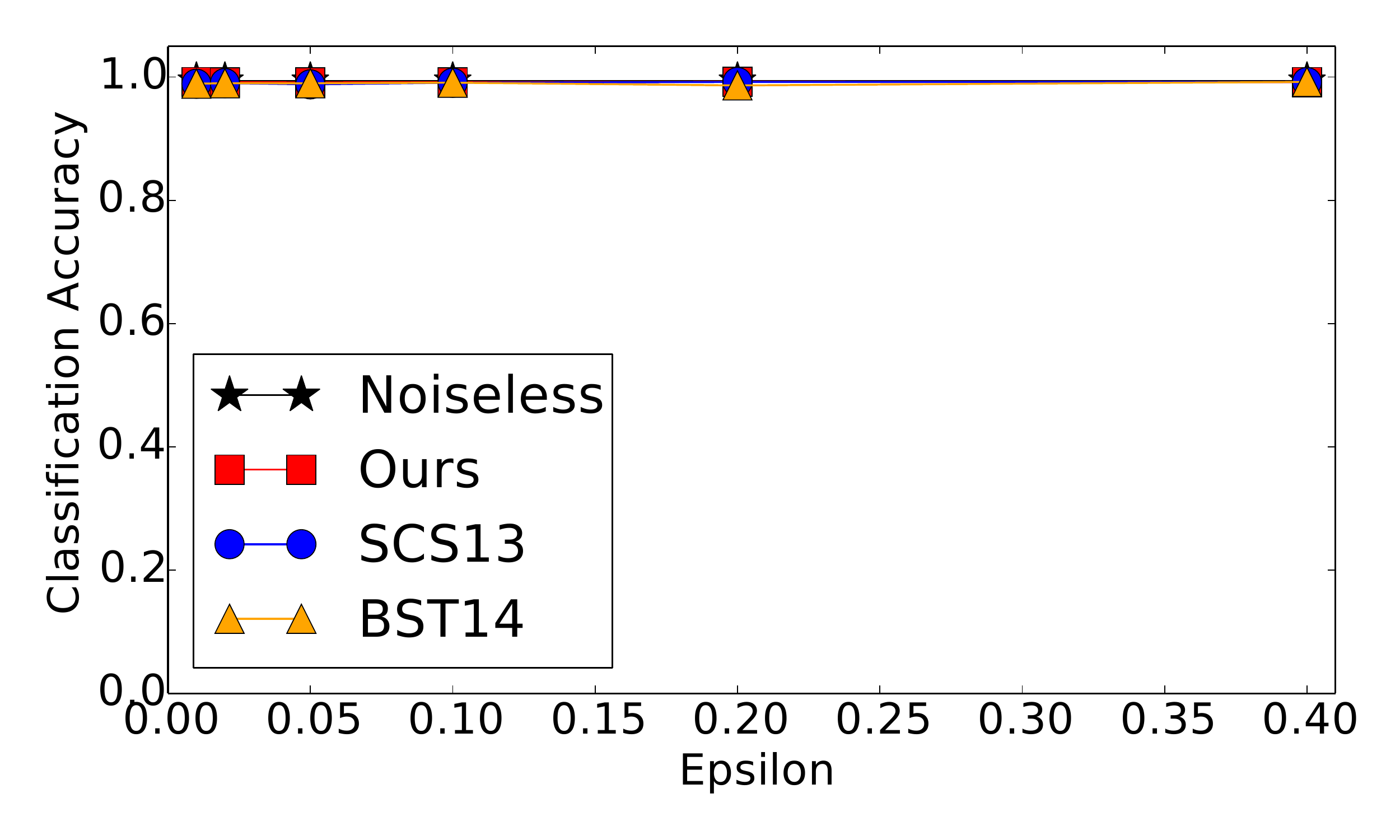}
  } \\[-3ex]
  \subfloat{
    \includegraphics[width=0.24\columnwidth]
    % {plots/covertype/test1_10_pass.pdf}
    {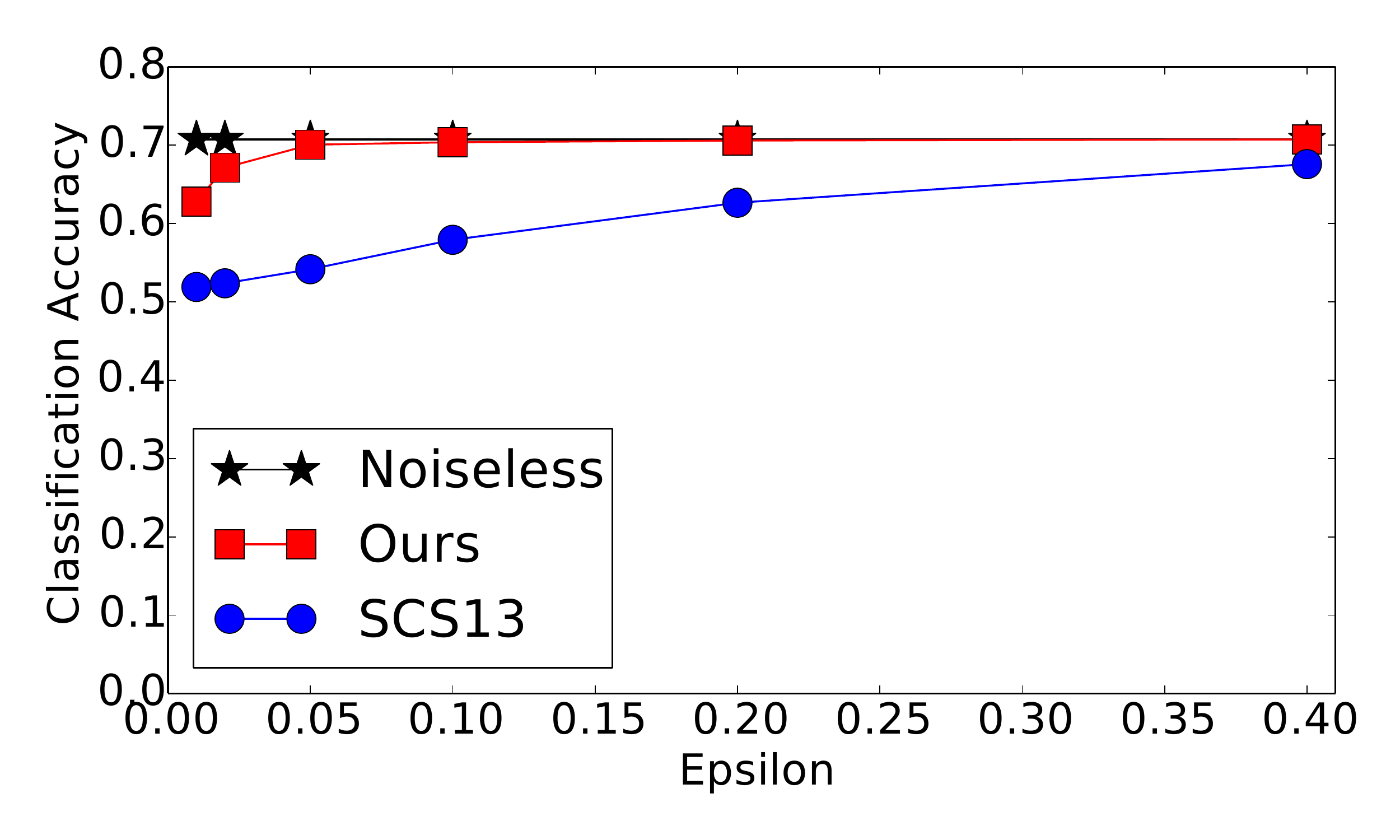}
  }
  \subfloat{
    \includegraphics[width=0.24\columnwidth]
    % {plots/covertype/test2_10_pass.pdf}
    {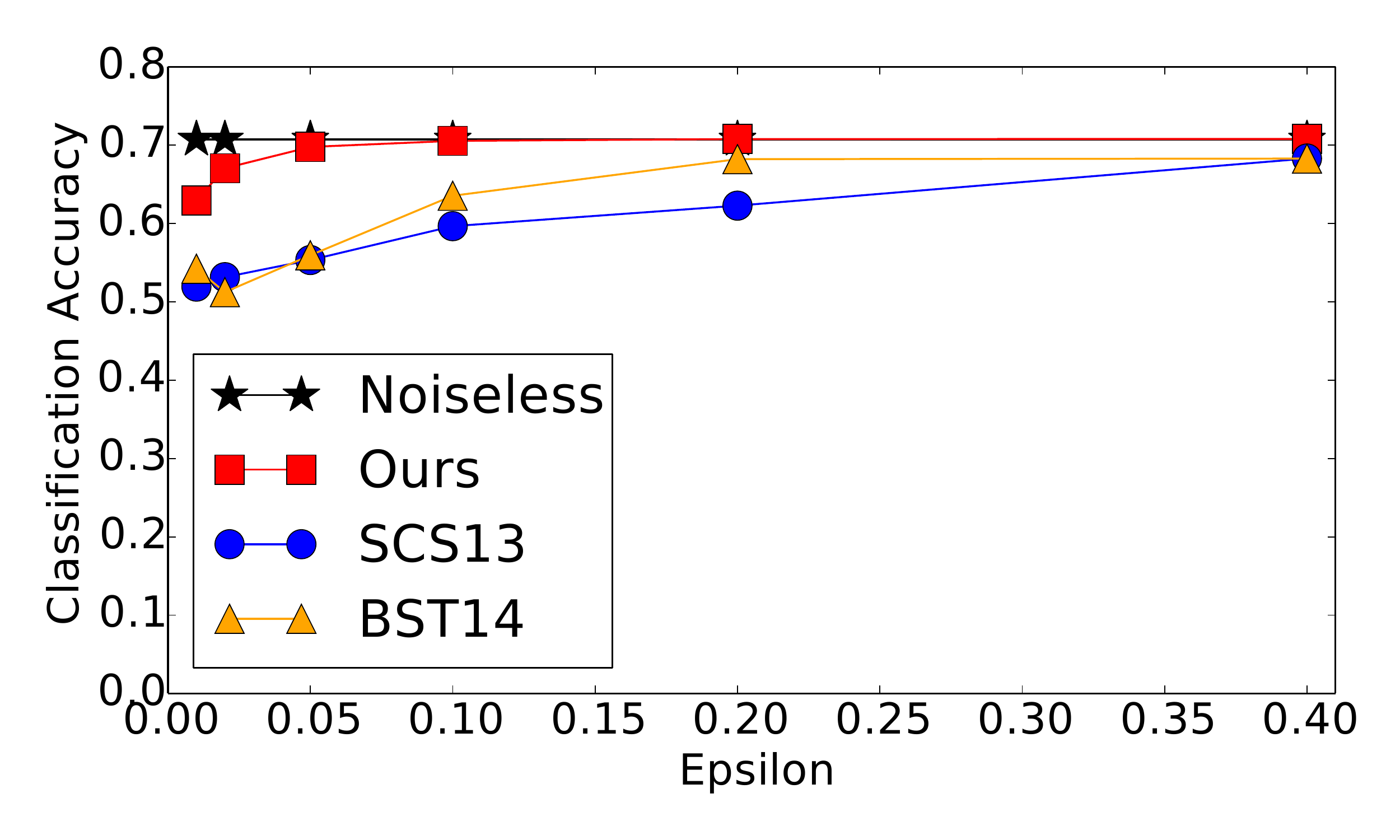}
  }
  \subfloat{
    \includegraphics[width=0.24\columnwidth]
    % {plots/covertype/test3_10_pass.pdf}
    {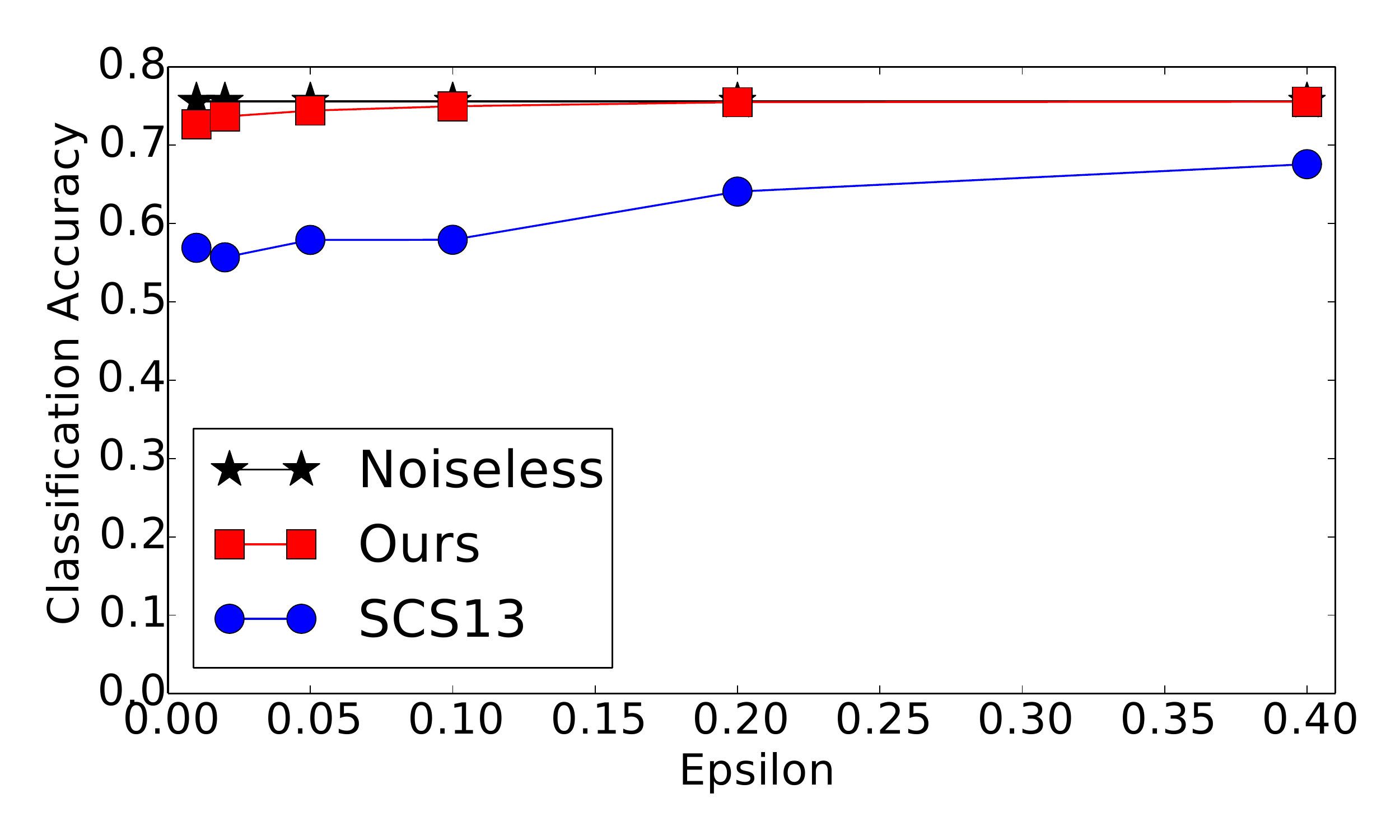}
  }
  \subfloat{
    \includegraphics[width=0.24\columnwidth]
    % {plots/covertype/test4_10_pass.pdf}
    {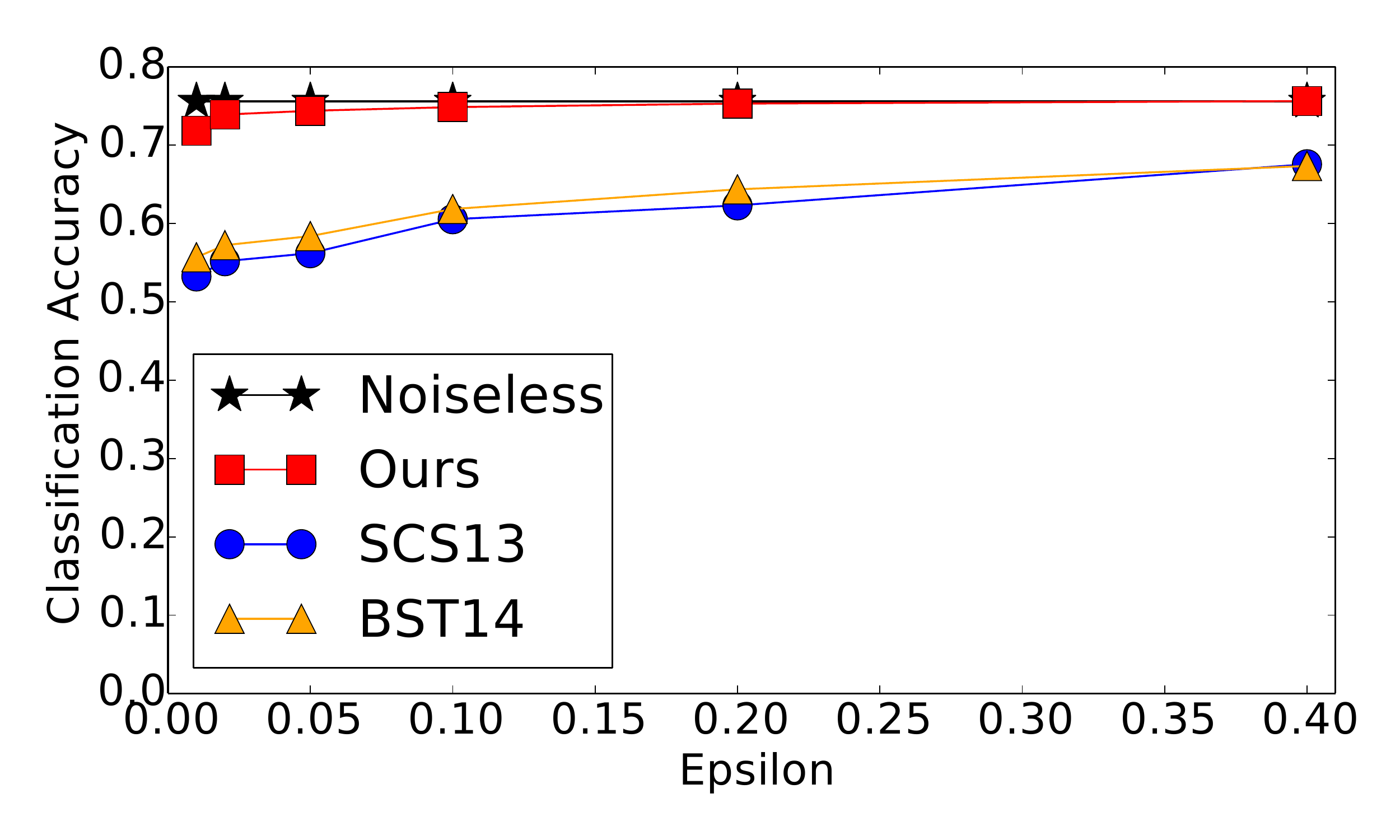}
  }
  \caption{
    {\bf Tuning using a Private Tuning Algorithm}.
    Row 1 is MNIST, row 2 is Protein and row 3 is Forest Covertype.
    Each row gives the test accuracy results of 4 tests:
    Test 1 is Convex, $\protect (\varepsilon, 0)$-DP,
    Test 2 is Convex, $\protect (\varepsilon, \delta)$-DP,
    Test 3 is Strongly Convex, $\protect (\varepsilon, 0)$-DP,
    and Test 4 is Strongly Convex, $\protect (\varepsilon, \delta)$-DP.
    For Test 1 and 3, we compare Noiseless, our algorithm and SCS13.
    For Test 2 and 4, we compare all four algorithms.
    The mini-batch size $b = 50$.
    For strongly convex optimization we set $\protect R = 1/\lambda$,
    otherwise we unconstrained  optimization for the convex case.
    The hyper-parameters were tuned using Algorithm~\ref{alg:private-tuning}
    with a standard ``grid search'' with $2$ values for $k$ ($5$ and $10$)
    and $3$ values for $\lambda$ ($0.0001$,  $0.001$, $0.01$), where applicable.
  }
  \label{fig:accuracy:tests_50_mb_10_passes}
\end{figure*}
%%% Local Variables:
%%% mode: latex
%%% TeX-master: "make"
%%% End:

\red{
\subsection{Lessons from Our Experiments}
\label{sec:lessons}
\noindent The experimental results demonstrate that our private SGD algorithms produce much 
more accurate models than prior work. Perhaps more importantly, our algorithms are also more 
{\em stable} with small privacy budgets, which is important for practical applications.

Our algorithms are also {\em easier to tune} in practice than SCS13 and BST14.
In particular, the only parameters that one needs to tune for our algorithms are mini-batch size $b$ and
$L_2$ regularization parameter $\lambda$; other parameters can either be derived from the loss function
or can be fixed based on our theoretical analysis. In contrast, SCS13 and BST14 require more attention 
to the number of passes $k$. For $b$, we recommend setting it as a value between $50$ and $200$,
noting that too large a value may make gradient steps more expensive and might require more passes.
For $\lambda$, we recommend using private parameter tuning with candidates chosen from a typical range
of $(10^{-5}, 10^{-2})$ (e.g., choose $\{10^{-5}, 10^{-4}, 10^{-3}, 10^{-2}\}$).

From the larger perspective of building differentially private analytics systems, however, we note that this paper 
addresses how to answer one ``query'' privately and effectively; in some applications, one might want to 
answer multiple such queries. Studying tradeoffs such as how to split the privacy budget across 
multiple queries is largely orthogonal to our paper's focus although they are certainly important.
Our work can be plugged into existing frameworks that attempt to address this requirement. 
That said, we think there is still a large gap between theory and practice for differentially private 
analytics systems. 

%%% Local Variables: 
%%% mode: latex
%%% TeX-master: "make"
%%% End: 

}

%%% Local Variables: 
%%% mode: latex
%%% TeX-master: t
%%% End: 

\section{Related Work}
\label{sec:related}
There has been much prior work on differentially private convex optimization.
There are three main styles of algorithms --
output perturbation~\cite{BST14, CMS11, JKT12, RBHT09},
objective perturbation~\cite{CMS11, KST12}
and online algorithms~\cite{BST14, DJW13, JKT12, SCS13}.
Output perturbation works by finding the exact convex minimizer
and then adding noise to it, while objective perturbation {\em exactly}
solves a randomly perturbed optimization problem.
Unfortunately, the privacy guarantees provided by both styles often assume that
the exact convex minimizer can be found, which usually does not hold in practice.

There are also a number of online approaches.~\cite{JKT12} provides an online algorithm
for strongly convex optimization based on a {\em proximal algorithm}
(e.g. Parikh and Boyd~\cite{PB14}),
which is harder to implement than SGD.
They also provide an offline version (Algorithm 6) for the strongly convex case
that is similar to our approach.
SGD-style algorithms were provided by~\cite{BST14, DJW13, JKT12, SCS13}.
\red{
There has also been a recent work on deep learning (non-convex optimization)
with differential privacy~\cite{ACGMMT016}. Unfortunately, no convergence result
is known for private non-convex optimization, and they also can only guarantee
$(\varepsilon, \delta)$-differential privacy due to the use of advanced composition
of $(\varepsilon, \delta)$-differential privacy.

Finally, our results are not directly comparable with systems such as RAPPOR~\cite{EPK14}.
In particular, RAPPOR assumes a different privacy model (local differential privacy)
where there is no trusted centralized agency who can compute on the entire raw dataset.
}

%%% Local Variables: 
%%% mode: latex
%%% TeX-master: "make"
%%% End: 

\section{Conclusion and Future Work}
\label{sec:conclusion}
Scalable and differentially private convex ERM have each received significant attention in the past decade.
Unfortunately, little previous work has examined the private ERM problem
in scalable systems such as in-RDBMS analytics systems. This paper takes a step to bridge this gap.
There are many intriguing future directions to pursue. We need to better understand 
the convergence behavior of private SGD when only a constant number of passes can be afforded.
BST14~\cite{BST14} provides a convergence bound for private SGD when $O(m)$ passes are made.
SCS13~\cite{SCS13} does not provide a convergence proof; however, the work of~\cite{DJW13},
which considers {\em local differential privacy}, a privacy model where data providers
do not even trust the data collector, can be used to conclude convergence for SCS13,
though at a very slow rate. Finally, while our method converges very well in practice
with multiple passes, we can only prove convergence with {\em one pass}.
Can we prove convergence bounds of our algorithms for multiple-passes and match the bounds of BST14?

%%% Local Variables: 
%%% mode: latex
%%% TeX-master: t
%%% End: 

{%\small
  \bibliographystyle{abbrv}
  \bibliography{paper}

\begin{thebibliography}{10}

\bibitem{mahout}
{Apache Mahout}.
\newblock \url{mahout.apache.org}.

\bibitem{spark}
{Apache Spark}.
\newblock \url{https://en.wikipedia.org/wiki/Apache_Spark}.

\bibitem{gridsearch}
{Grid Search}.
\newblock \url{http://scikit-learn.org/stable/modules/grid_search.html}.

\bibitem{hinge-loss}
{Hinge Loss and Smoothed Variants}.
\newblock \url{https://en.wikipedia.org/wiki/Hinge_loss}.

\bibitem{eps-dp-vs-eps-delta-dp}
{How many secrets do you have?}
\newblock
  \url{https://github.com/frankmcsherry/blog/blob/master/posts/2017-02-08.md}.

\bibitem{ml-preprocessing}
{Preprocessing data in machine learning}.
\newblock \url{http://scikit-learn.org/stable/modules/preprocessing.html}.

\bibitem{random-projection}
{Random Projection}.
\newblock \url{https://en.wikipedia.org/wiki/Random_projection}.

\bibitem{sample-uniform-random-unit-vector}
{random unit vector in multi-dimensional space}.
\newblock
  \url{http://stackoverflow.com/questions/6283080/random-unit-vector-in-multi-dimensional-space}.

\bibitem{ACGMMT016}
M.~Abadi, A.~Chu, I.~J. Goodfellow, H.~B. McMahan, I.~Mironov, K.~Talwar, and
  L.~Zhang.
\newblock Deep learning with differential privacy.
\newblock In {\em Proceedings of the 2016 {ACM} {SIGSAC} Conference on Computer
  and Communications Security, Vienna, Austria, October 24-28, 2016}, pages
  308--318, 2016.

\bibitem{BST14}
R.~Bassily, A.~Smith, and A.~Thakurta.
\newblock Private empirical risk minimization: Efficient algorithms and tight
  error bounds.
\newblock In {\em {FOCS}}, 2014.

\bibitem{BV04}
S.~Boyd and L.~Vandenberghe.
\newblock {\em Convex Optimization}.
\newblock Cambridge University Press, New York, NY, USA, 2004.

\bibitem{Bubeck15}
S.~Bubeck.
\newblock Convex optimization: Algorithms and complexity.
\newblock {\em Foundations and Trends in Machine Learning}, 8(3-4):231--357,
  2015.

\bibitem{CMS11}
K.~Chaudhuri, C.~Monteleoni, and A.~D. Sarwate.
\newblock Differentially private empirical risk minimization.
\newblock {\em Journal of Machine Learning Research}, 12:1069--1109, 2011.

\bibitem{CV13}
K.~Chaudhuri and S.~A. Vinterbo.
\newblock A stability-based validation procedure for differentially private
  machine learning.
\newblock In {\em {NIPS}}, 2013.

\bibitem{DJW13}
J.~C. Duchi, M.~I. Jordan, and M.~J. Wainwright.
\newblock Local privacy and statistical minimax rates.
\newblock In {\em {FOCS}}, 2013.

\bibitem{DMNS06}
C.~Dwork, F.~McSherry, K.~Nissim, and A.~Smith.
\newblock Calibrating noise to sensitivity in private data analysis.
\newblock In {\em {TCC}}, 2006.

\bibitem{DR14}
C.~Dwork and A.~Roth.
\newblock The algorithmic foundations of differential privacy.
\newblock {\em Foundations and Trends in Theoretical Computer Science},
  9(3-4):211--407, 2014.

\bibitem{EPK14}
{\'{U}}.~Erlingsson, V.~Pihur, and A.~Korolova.
\newblock {RAPPOR:} randomized aggregatable privacy-preserving ordinal
  response.
\newblock In {\em Proceedings of the 2014 {ACM} {SIGSAC} Conference on Computer
  and Communications Security, Scottsdale, AZ, USA, November 3-7, 2014}, pages
  1054--1067, 2014.

\bibitem{FKRR12}
X.~Feng, A.~Kumar, B.~Recht, and C.~R{\'{e}}.
\newblock Towards a unified architecture for in-rdbms analytics.
\newblock In {\em {SIGMOD}}, 2012.

\bibitem{gray}
J.~Gray et~al.
\newblock {Data Cube: A Relational Aggregation Operator Generalizing Group-By,
  Cross-Tab, and Sub-Totals}.
\newblock {\em Data Min. Knowl. Discov.}, 1(1):29--53, Jan. 1997.

\bibitem{HRS15}
M.~{Hardt}, B.~{Recht}, and Y.~{Singer}.
\newblock {Train faster, generalize better: Stability of stochastic gradient
  descent}.
\newblock {\em {ArXiv e-prints}}, Sept. 2015.

\bibitem{HMMCZ15}
M.~Hay, A.~Machanavajjhala, G.~Miklau, Y.~Chen, and D.~Zhang.
\newblock Principled evaluation of differentially private algorithms using
  dpbench.
\newblock {\em CoRR}, abs/1512.04817, 2015.

\bibitem{madlib}
J.~Hellerstein et~al.
\newblock {{The MADlib Analytics Library or MAD Skills, the SQL}}.
\newblock In {\em {VLDB}}, 2012.

\bibitem{JKT12}
P.~Jain, P.~Kothari, and A.~Thakurta.
\newblock Differentially private online learning.
\newblock In {\em {COLT}}, 2012.

\bibitem{JT13}
P.~Jain and A.~Thakurta.
\newblock Differentially private learning with kernels.
\newblock In {\em {ICML}}, 2013.

\bibitem{JZ13}
R.~Johnson and T.~Zhang.
\newblock Accelerating stochastic gradient descent using predictive variance
  reduction.
\newblock In {\em {NIPS}}, 2013.

\bibitem{KST12}
D.~Kifer, A.~D. Smith, and A.~Thakurta.
\newblock Private convex optimization for empirical risk minimization with
  applications to high-dimensional regression.
\newblock In {\em {COLT}}, 2012.

\bibitem{NY83}
A.~Nemirovsky and D.~Yudin.
\newblock Problem complexity and method efficiency in optimization.
\newblock 1983.

\bibitem{Nesterov04}
Y.~Nesterov.
\newblock {\em Introductory lectures on convex optimization : a basic course}.
\newblock Applied optimization. Kluwer Academic Publ., 2004.

\bibitem{PB14}
N.~Parikh and S.~P. Boyd.
\newblock Proximal algorithms.
\newblock {\em Foundations and Trends in Optimization}, 1(3):127--239, 2014.

\bibitem{Polyak87}
B.~T. Polyak.
\newblock {\em Introduction to optimization}.
\newblock Optimization Software, 1987.

\bibitem{RSB12}
N.~L. Roux, M.~W. Schmidt, and F.~R. Bach.
\newblock A stochastic gradient method with an exponential convergence rate for
  finite training sets.
\newblock In {\em {NIPS}}, 2012.

\bibitem{RBHT09}
B.~I.~P. Rubinstein, P.~L. Bartlett, L.~Huang, and N.~Taft.
\newblock Learning in a large function space: Privacy-preserving mechanisms for
  {SVM} learning.
\newblock {\em {CoRR}}, abs/0911.5708, 2009.

\bibitem{Shamir16}
O.~{Shamir}.
\newblock {Without-Replacement Sampling for Stochastic Gradient Methods:
  Convergence Results and Application to Distributed Optimization}.
\newblock {\em ArXiv e-prints}, Mar. 2016.

\bibitem{SCS13}
S.~Song, K.~Chaudhuri, and A.~D. Sarwate.
\newblock Stochastic gradient descent with differentially private updates.
\newblock In {\em {GlobalSIP}}, 2013.

\bibitem{ZXYZW13}
J.~Zhang, X.~Xiao, Y.~Yang, Z.~Zhang, and M.~Winslett.
\newblock Privgene: differentially private model fitting using genetic
  algorithms.
\newblock In {\em {SIGMOD}}, 2013.

\bibitem{ZZXYW12}
J.~Zhang, Z.~Zhang, X.~Xiao, Y.~Yang, and M.~Winslett.
\newblock Functional mechanism: Regression analysis under differential privacy.
\newblock {\em {PVLDB}}, 2012.

\bibitem{Zinkevich03}
M.~Zinkevich.
\newblock Online convex programming and generalized infinitesimal gradient
  ascent.
\newblock In {\em {ICML}}, 2003.

\end{thebibliography}
}

\appendix
% SVM results with private tuning.
% TODO(wuxi): The following only reports SVM for private tuning case.
%   We need to add tuning with public data case.
\begin{figure*}[!htb]
  \centering
  \subfloat{
    \includegraphics[width=0.24\columnwidth]
    {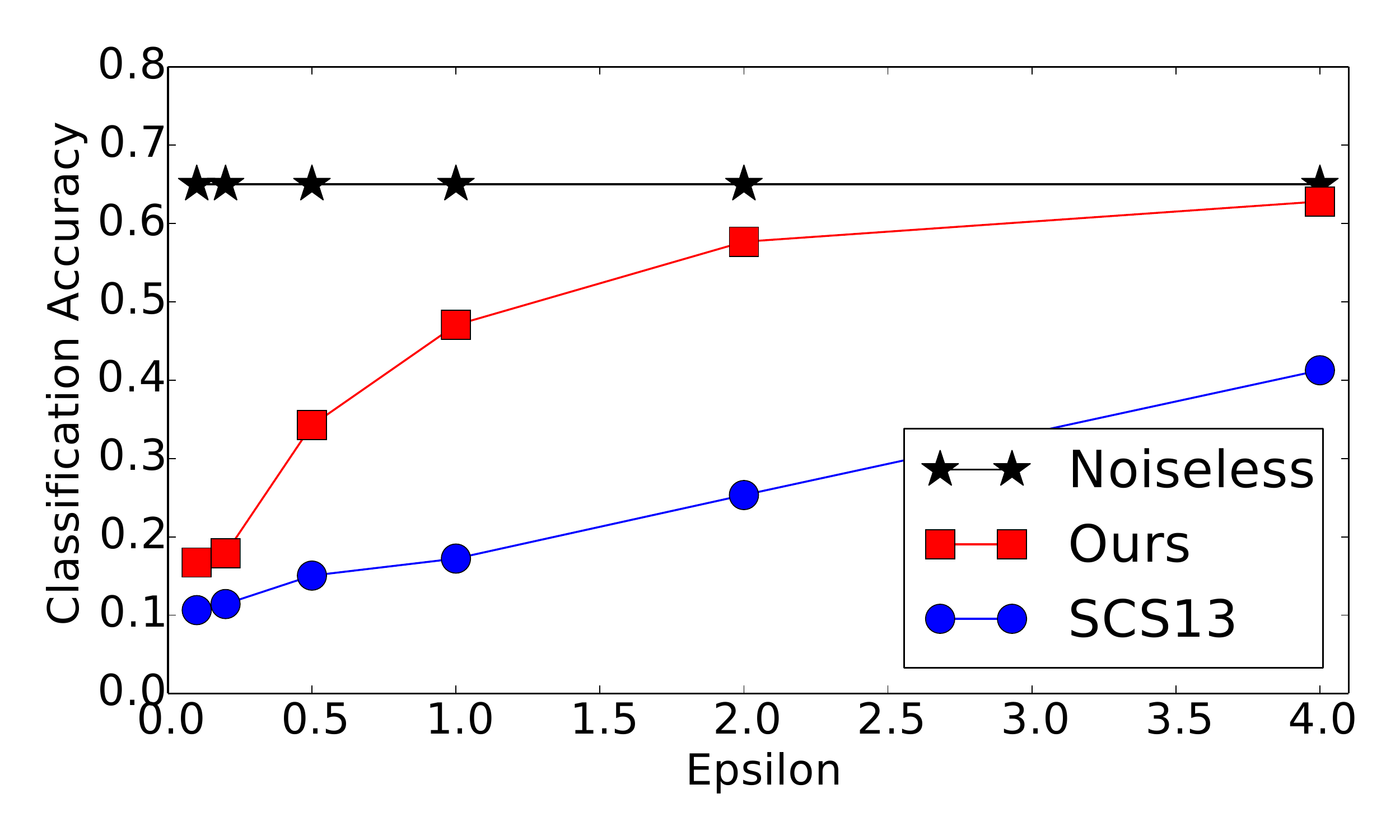}
  }
  \subfloat{
    \includegraphics[width=0.24\columnwidth]
    {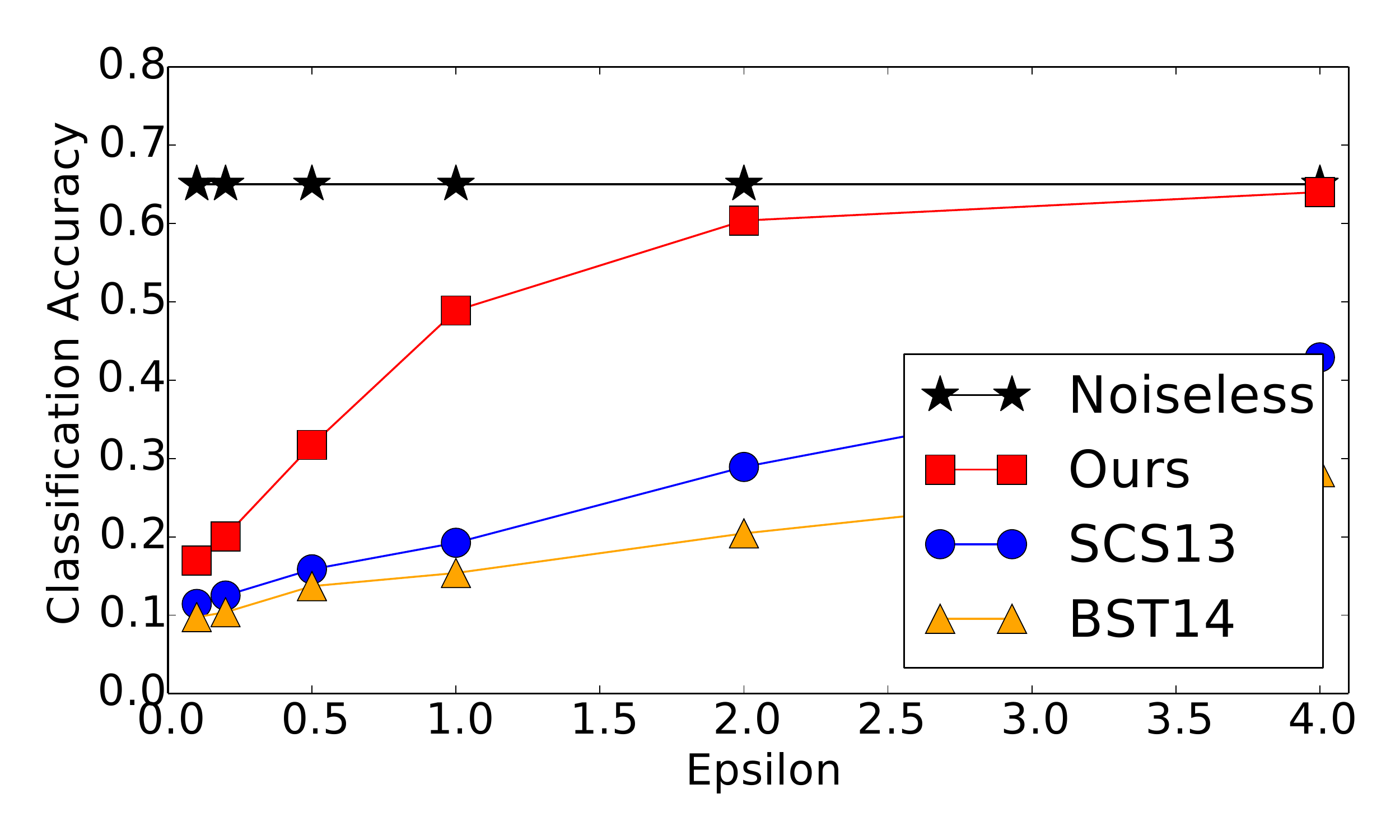}
  }
  \subfloat{
    \includegraphics[width=0.24\columnwidth]
    {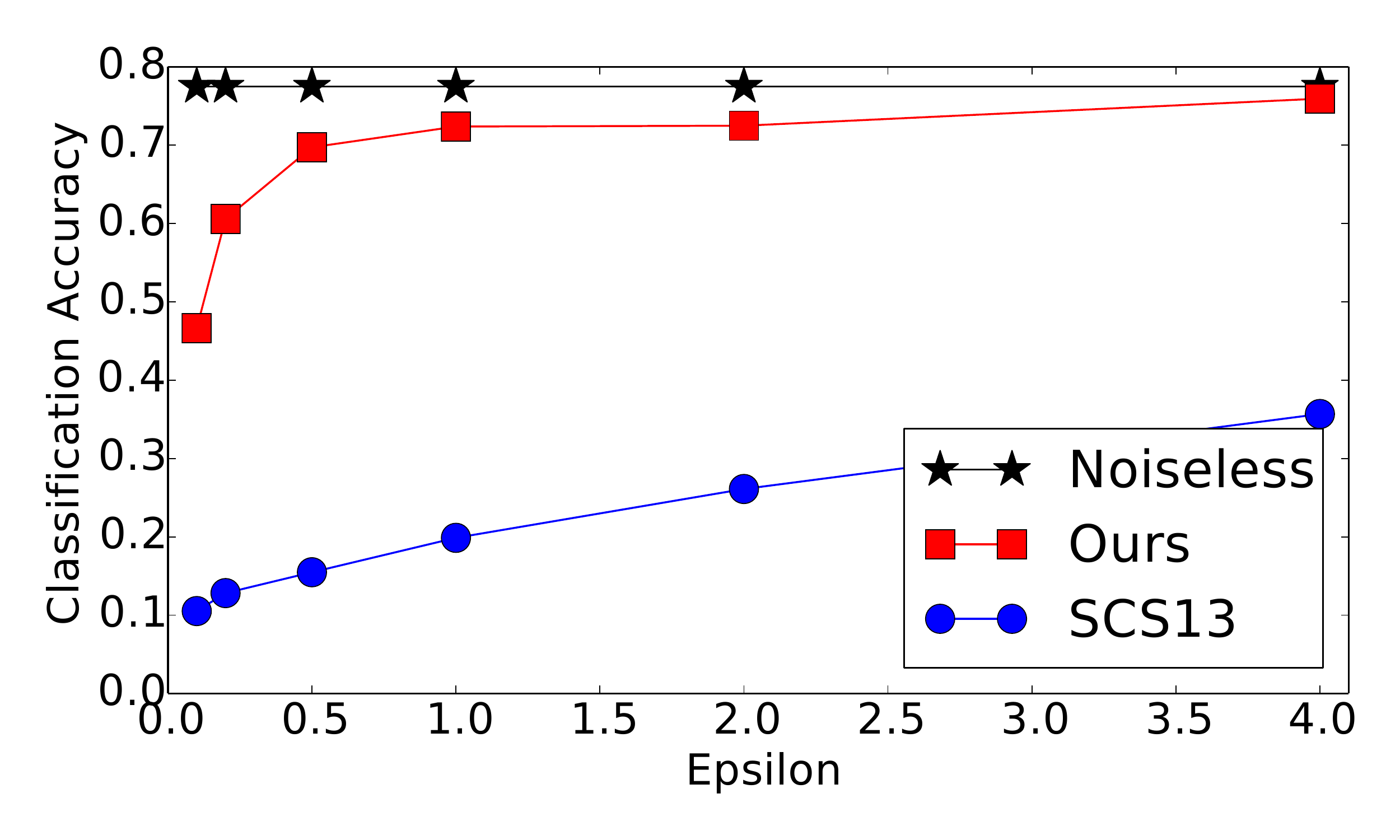}
  }
  \subfloat{
    \includegraphics[width=0.24\columnwidth]
    {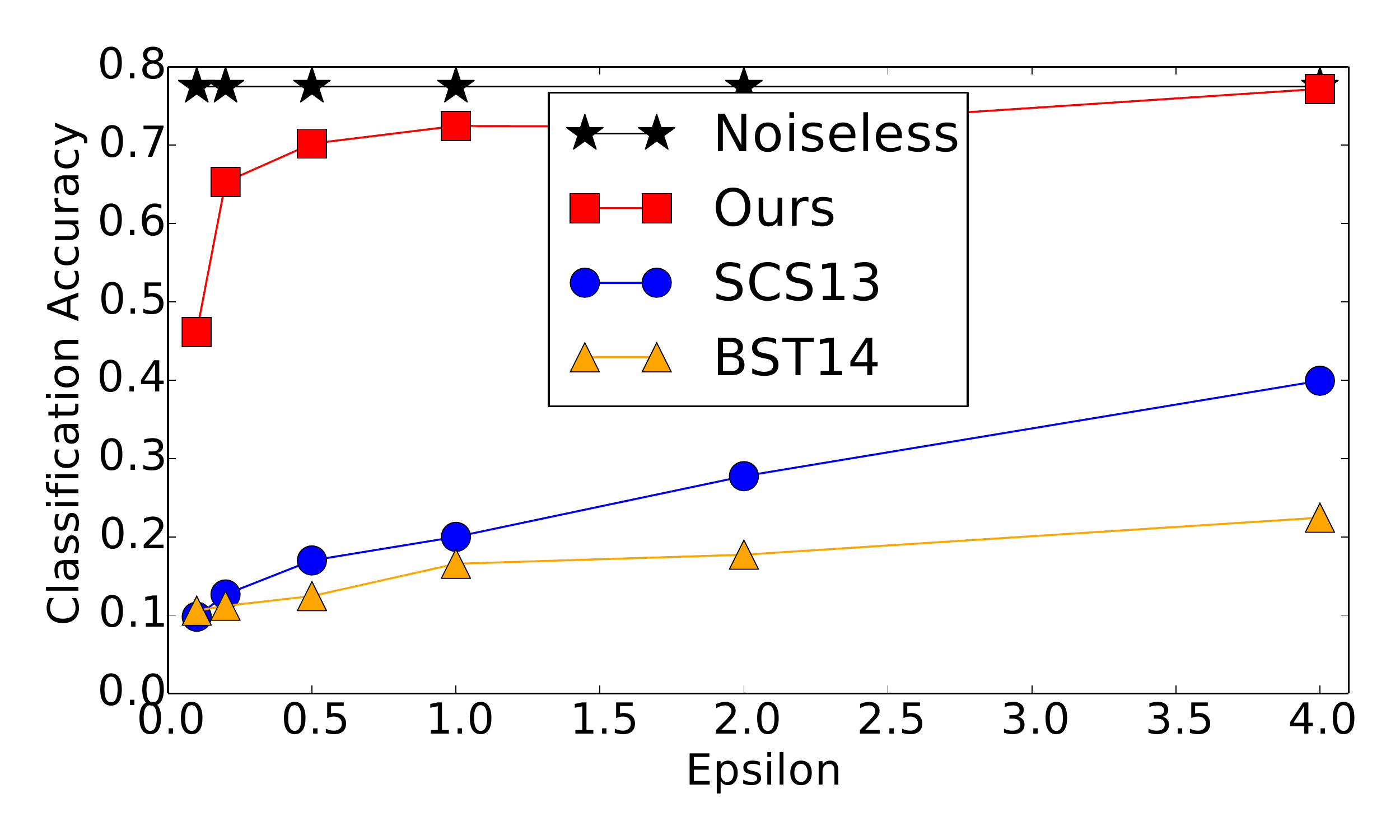}
  } \\[-3ex]
  \subfloat{
    \includegraphics[width=0.24\columnwidth]
    {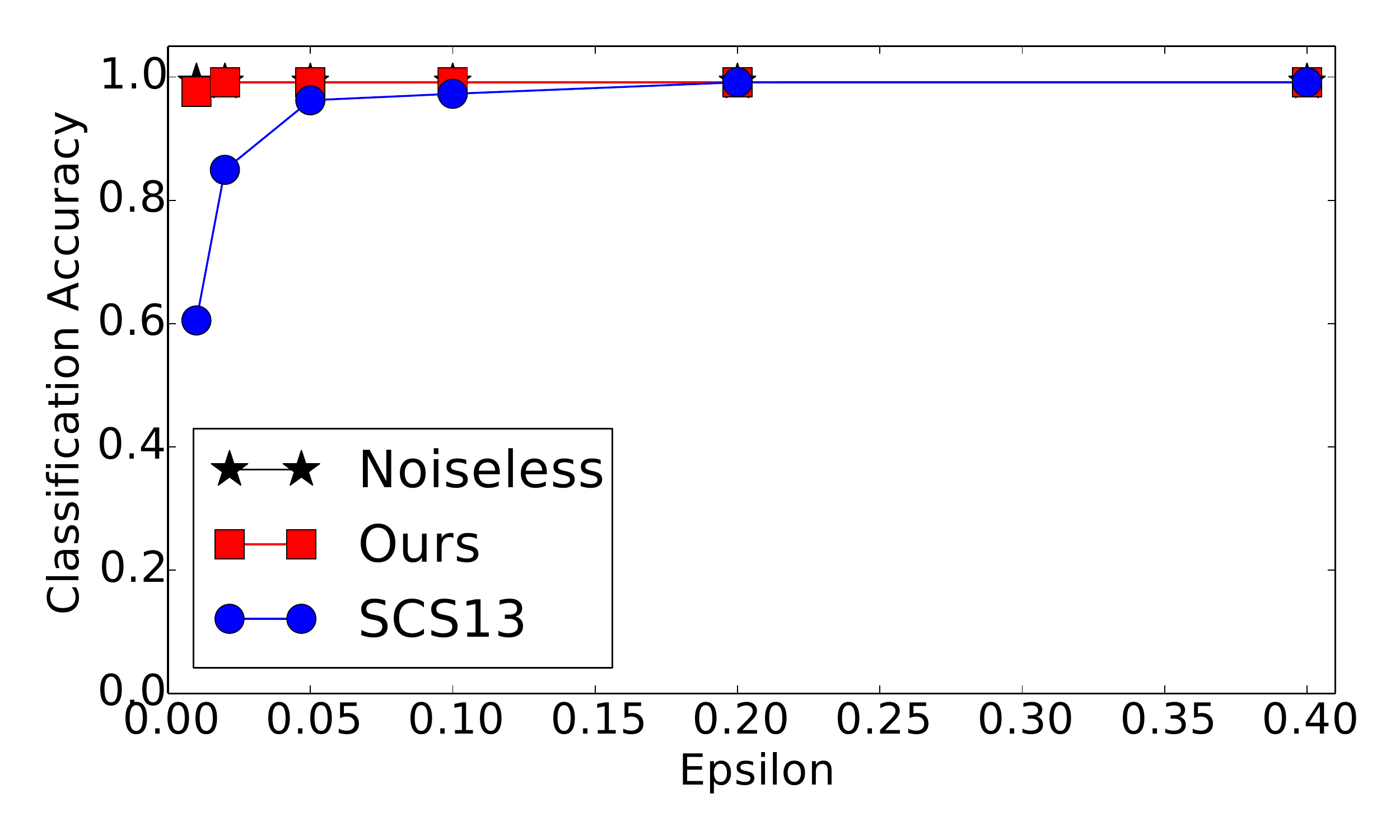}
  }
  \subfloat{
    \includegraphics[width=0.24\columnwidth]
    {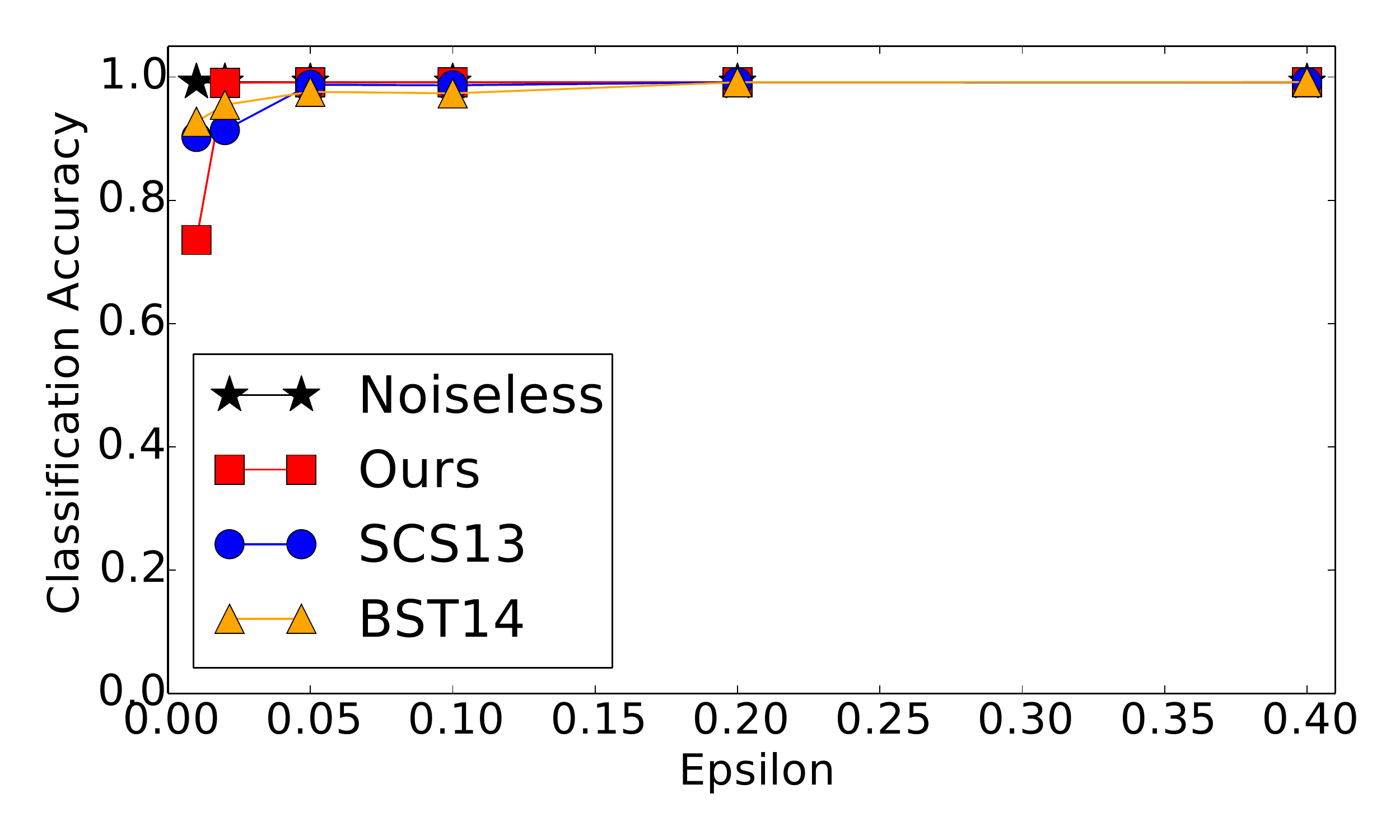}
  }
  \subfloat{
    \includegraphics[width=0.24\columnwidth]
    {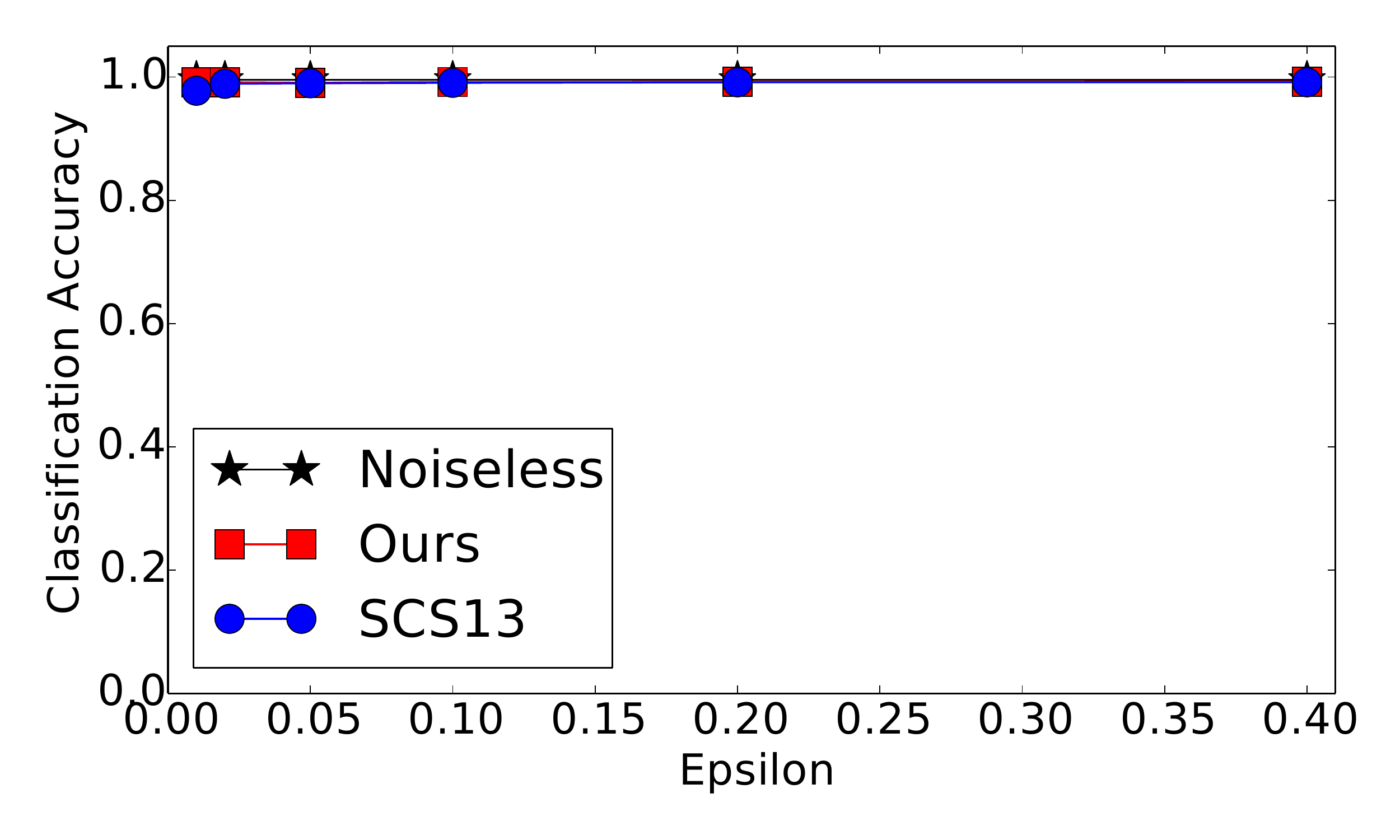}
  }
  \subfloat{
    \includegraphics[width=0.24\columnwidth]
    {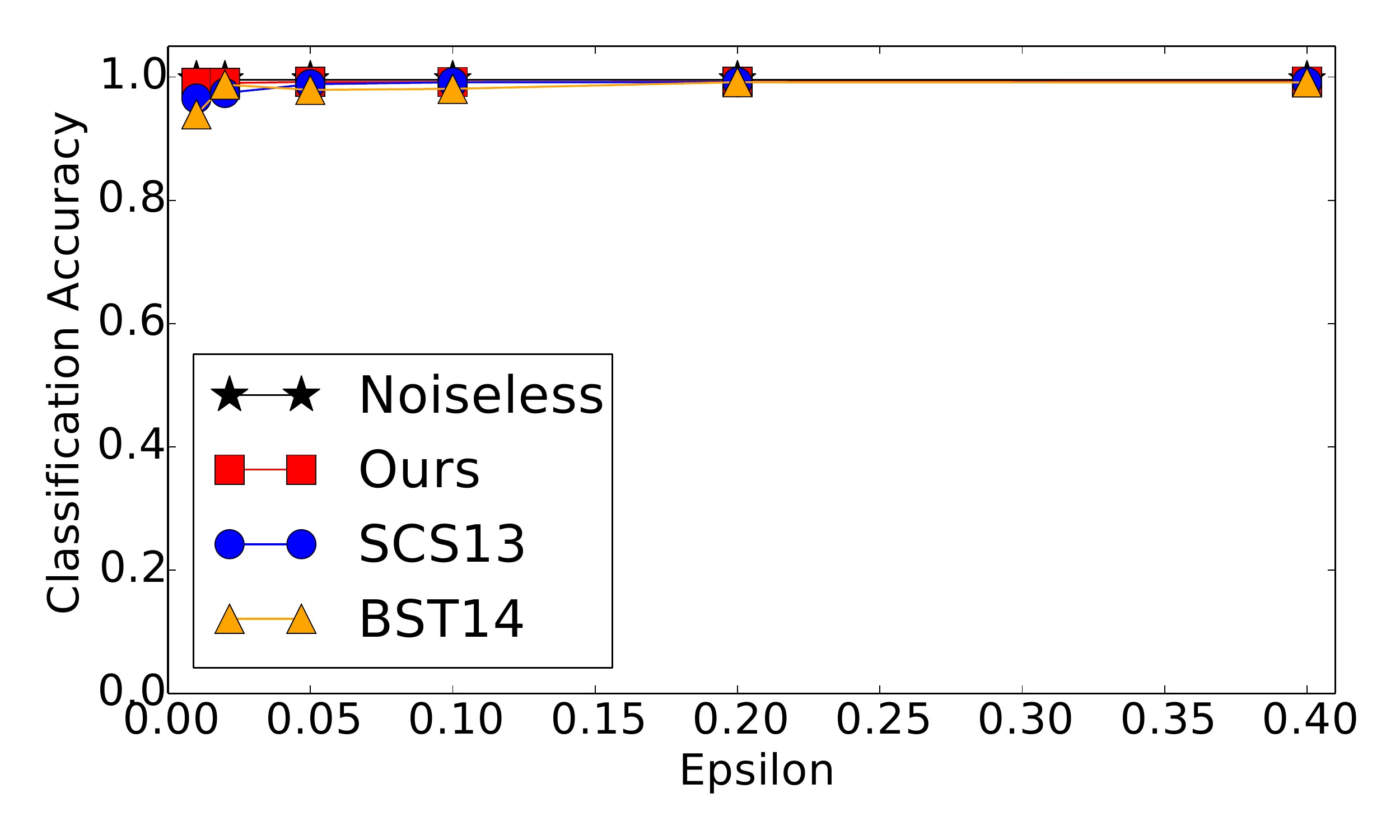}
  } \\[-3ex]
  \subfloat{
    \includegraphics[width=0.24\columnwidth]
    {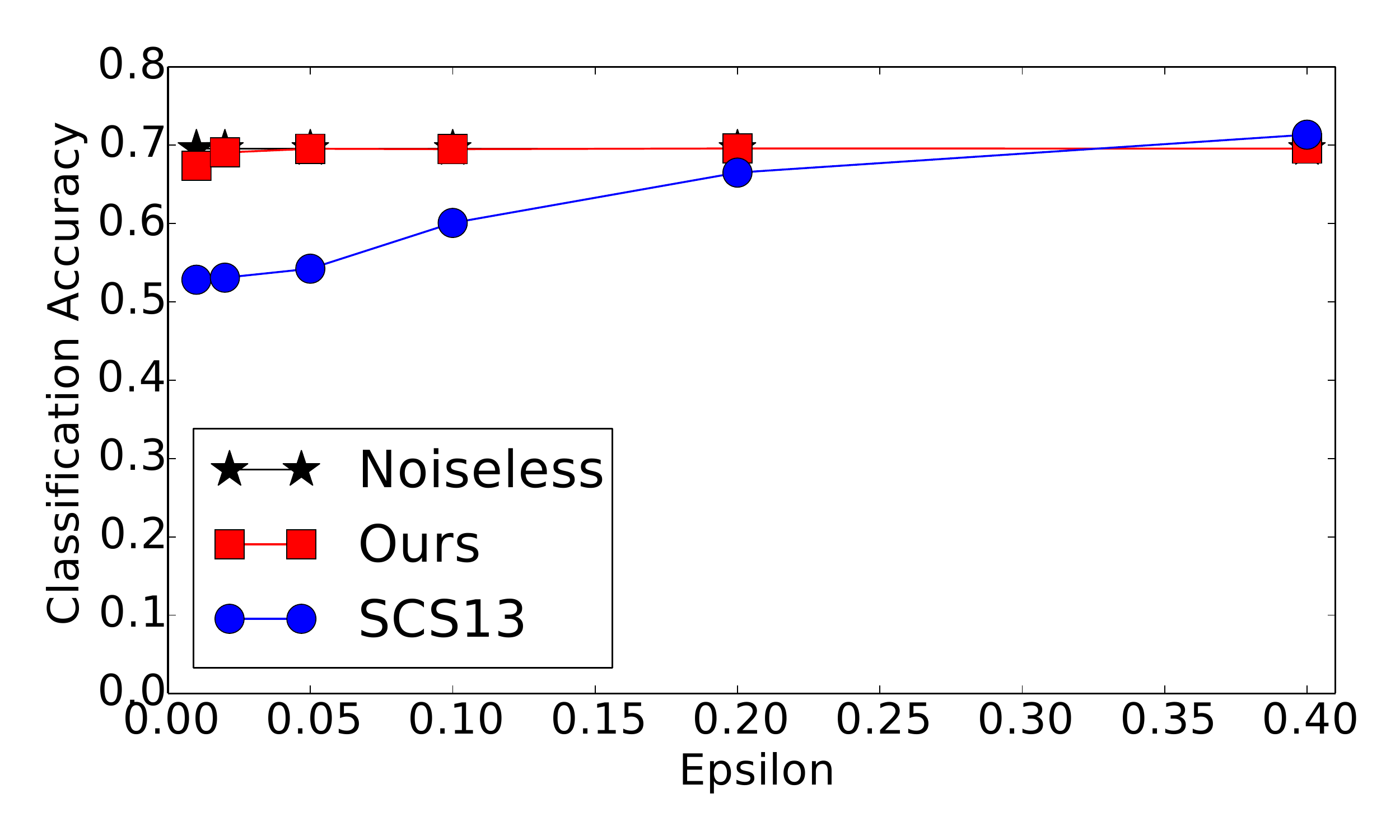}
  }
  \subfloat{
    \includegraphics[width=0.24\columnwidth]
    {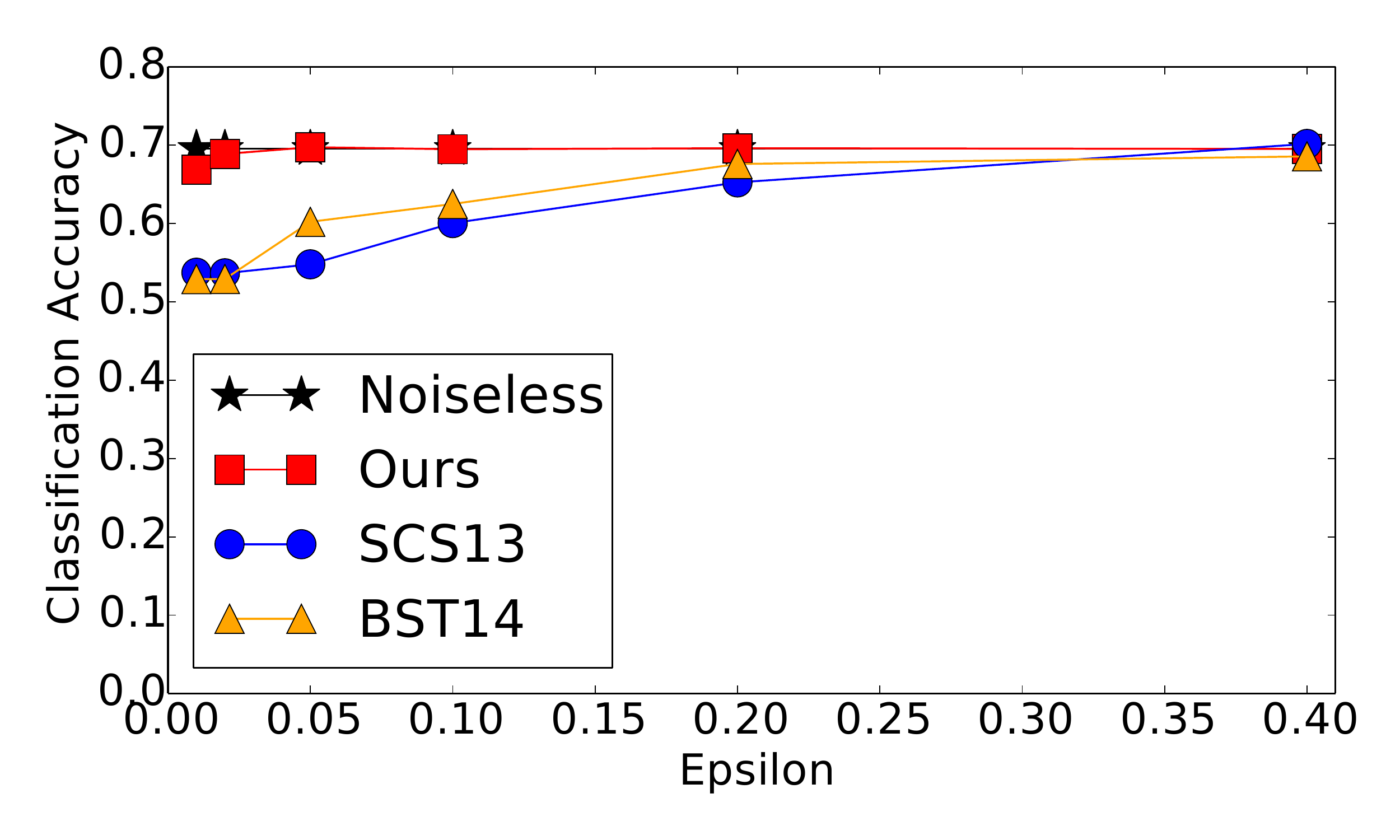}
  }
  \subfloat{
    \includegraphics[width=0.24\columnwidth]
    {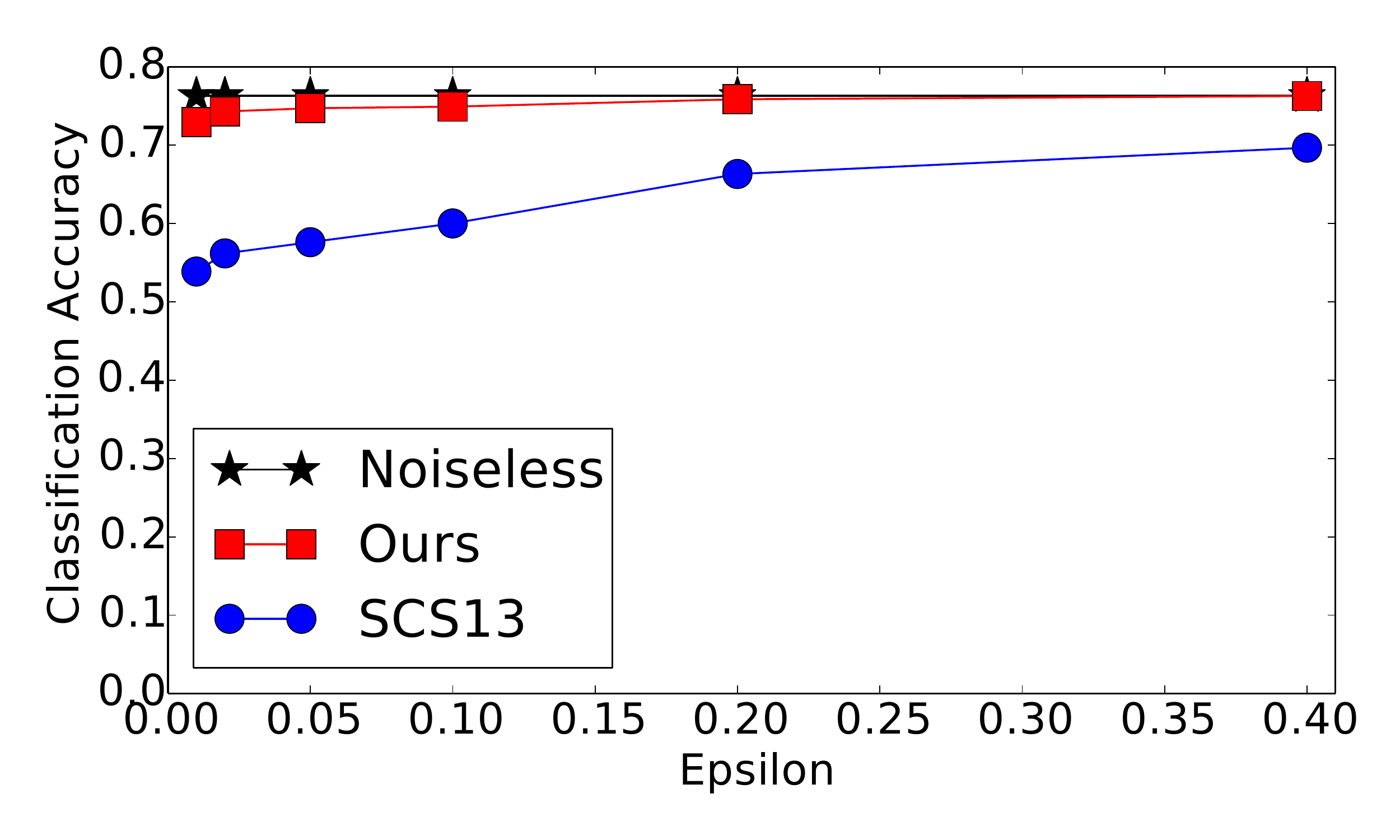}
  }
  \subfloat{
    \includegraphics[width=0.24\columnwidth]
    {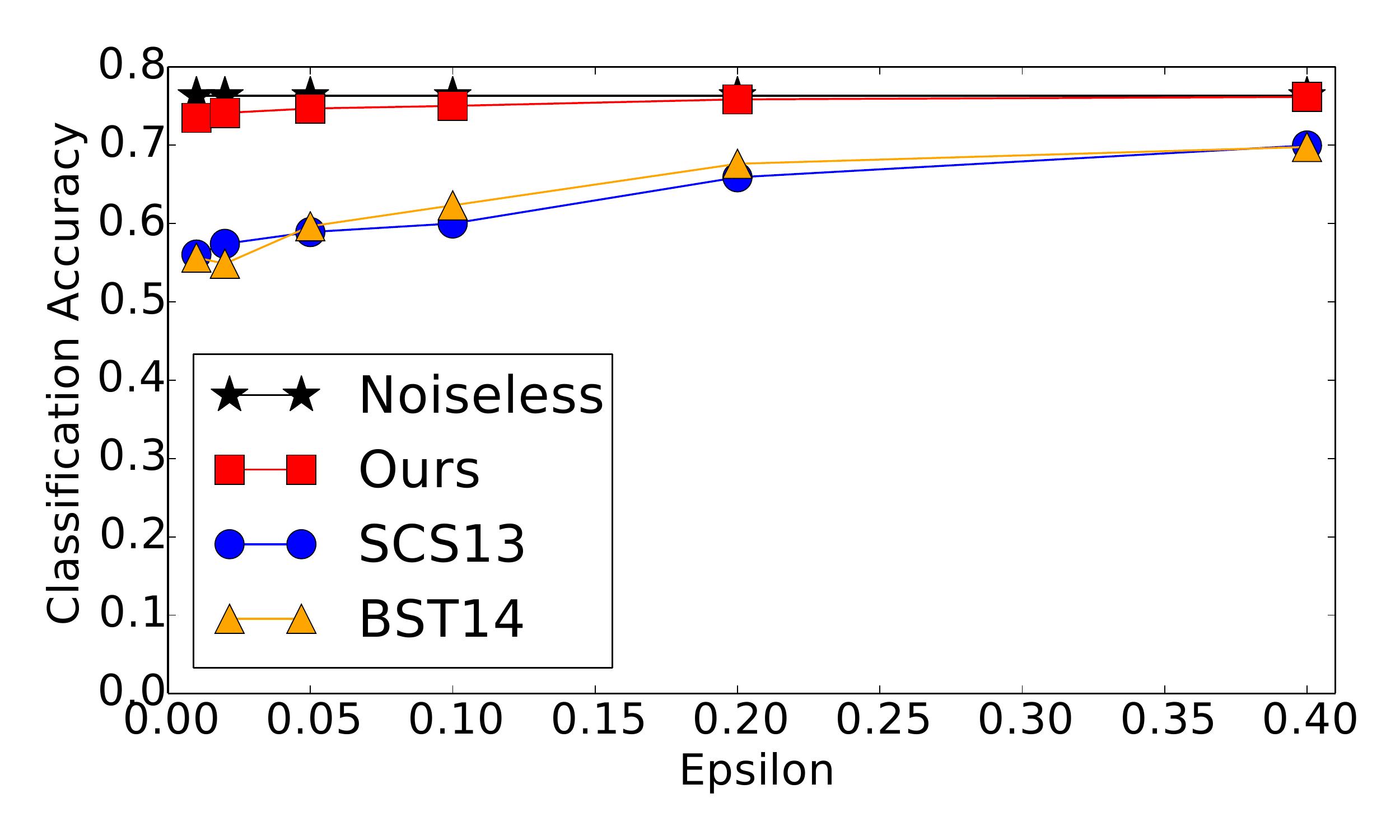}
  }
  \caption{
    {\bf Tuning using a Private Tuning Algorithm for Huber SVM}.
    Row 1 is MNIST, row 2 is Protein and row 3 is Forest Covertype.
    Each row gives the test accuracy results for 4 tests:
    Test 1 is Convex, $\protect (\varepsilon, 0)$-DP,
    Test 2 is Convex, $\protect (\varepsilon, \delta)$-DP,
    Test 3 is Strongly Convex, $\protect (\varepsilon, 0)$-DP,
    and Test 4 is Strongly Convex, $\protect (\varepsilon, \delta)$-DP.
    We compare Noiseless, our algorithms and SCS13 for tests 1 and 3,
    and compare all four algorithms for tests 2 and 4.
    The mini-batch size $b = 50$, and $h=0.1$ for the Huber loss.
    For strongly convex optimization, we set $\protect R = 1/\lambda$,
    otherwise we report unconstrained optimization in the convex case.
    The hyper-parameters were tuned using Algorithm~\ref{alg:private-tuning}
    with a standard ``grid search''  with $2$ values for $k$ ($5$ and $10$)
    and $3$ values for $\lambda$ ($0.0001$, $0.001$, $0.01$), where applicable.
  }
  \label{fig:svm:accuracy:tests_50_mb_10_passes}
\end{figure*}

% Logistic regression results on additional datasets: higgs and kddcup-99.
\begin{figure*}[!htb]
  \centering
  \subfloat{
    \includegraphics[width=0.24\columnwidth]
    {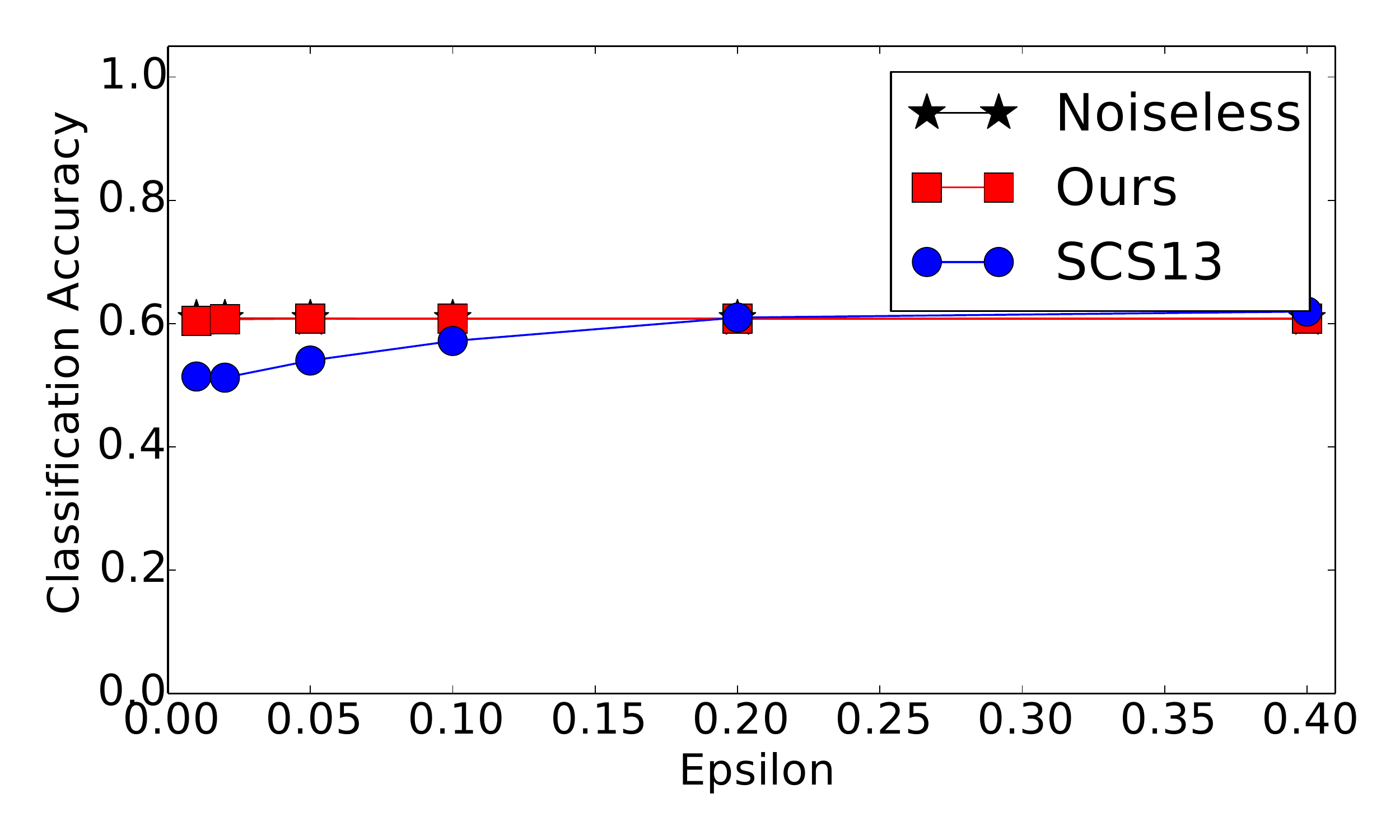}
  }
  \subfloat{
    \includegraphics[width=0.24\columnwidth]
    {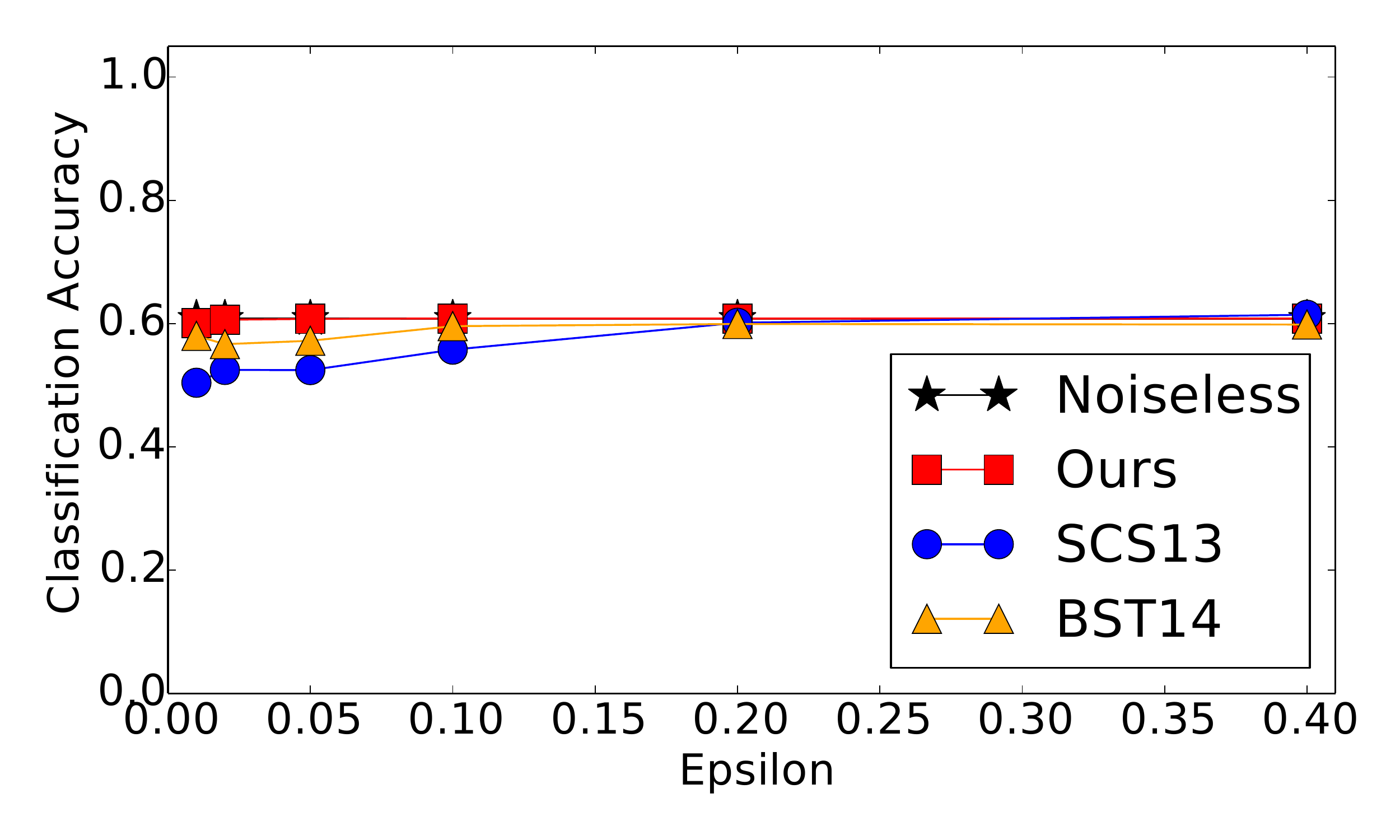}
  }
  \subfloat{
    \includegraphics[width=0.24\columnwidth]
    {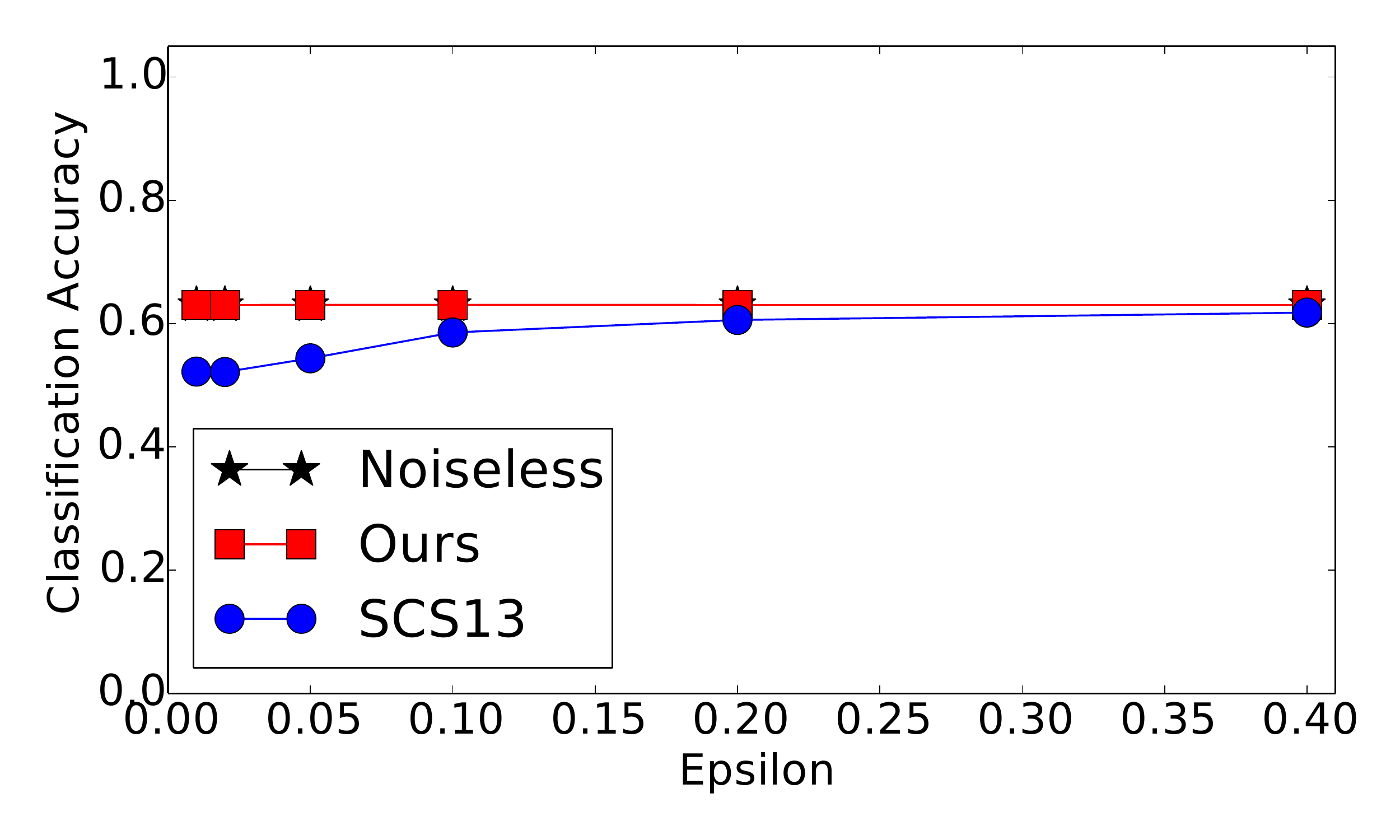}
  }
  \subfloat{
    \includegraphics[width=0.24\columnwidth]
    {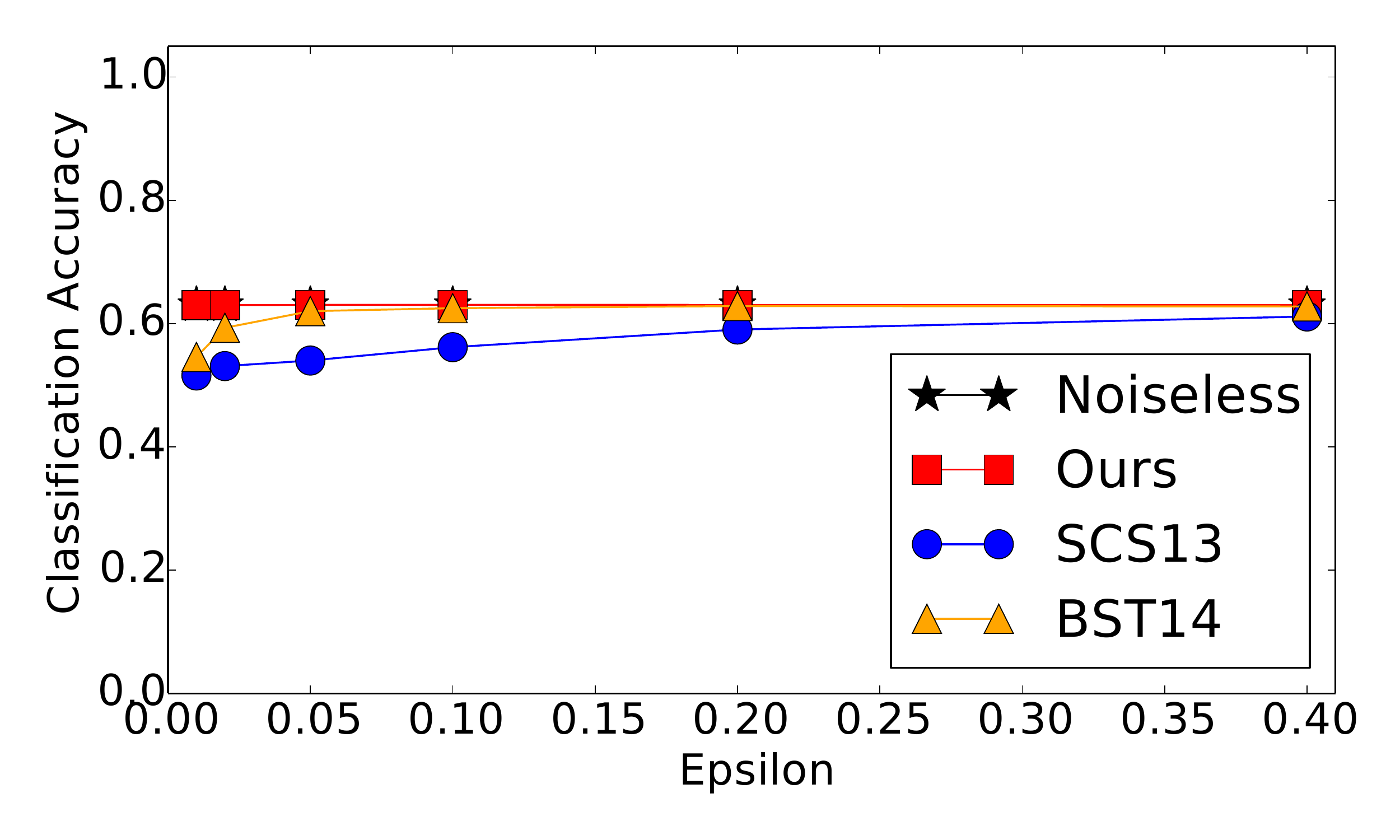}
  }\\[-3ex]
  \subfloat{
    \includegraphics[width=0.24\columnwidth]
    {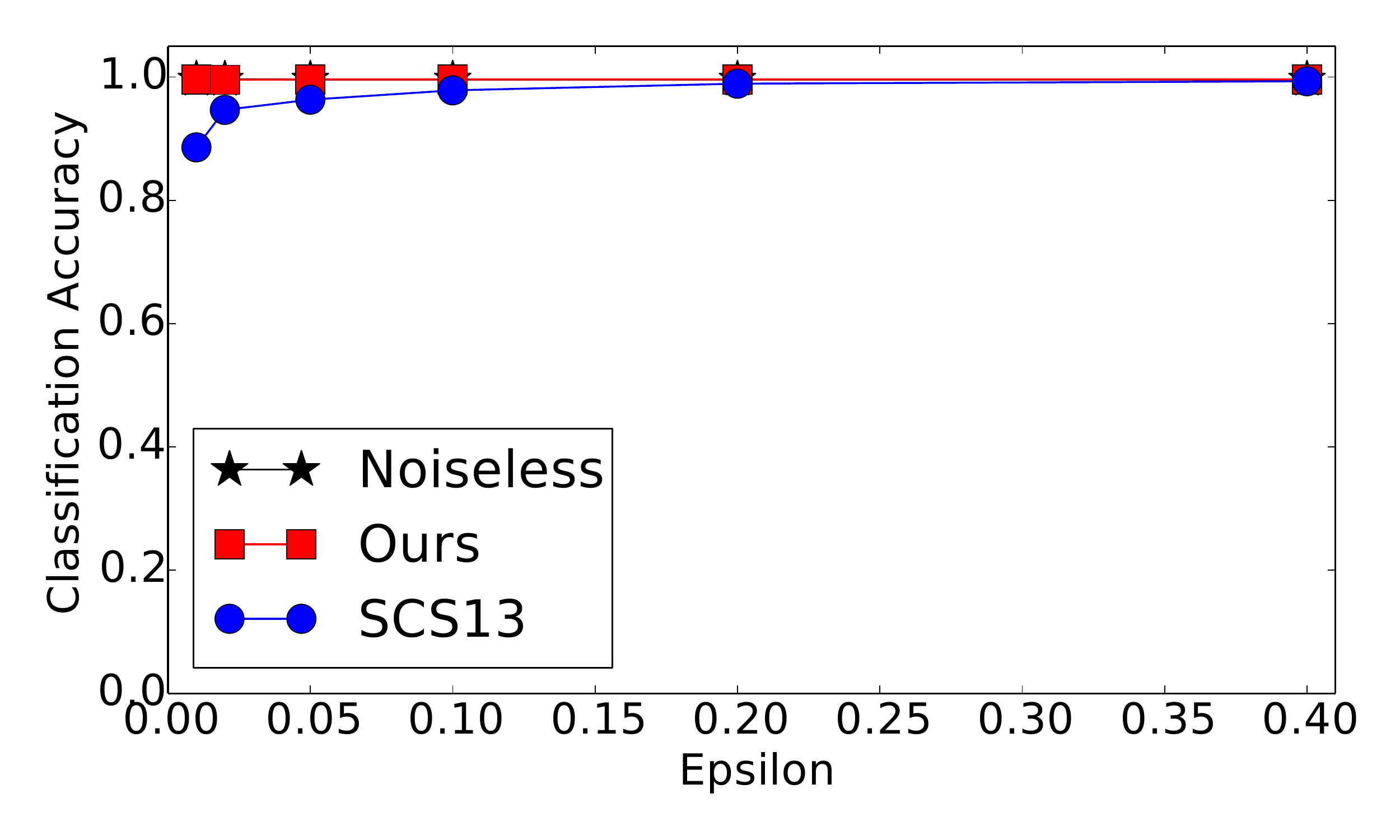}
  }
  \subfloat{
    \includegraphics[width=0.24\columnwidth]
    {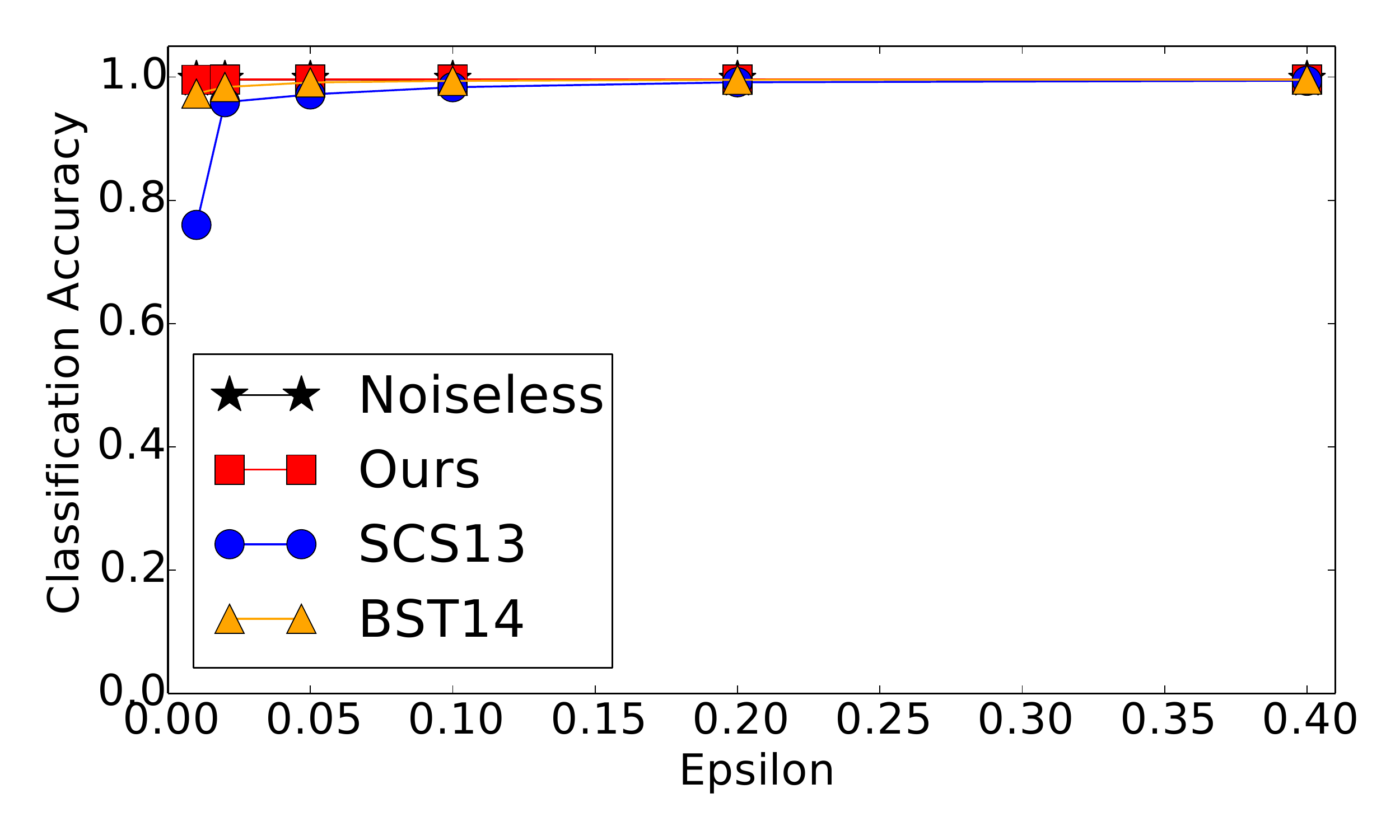}
  }
  \subfloat{
    \includegraphics[width=0.24\columnwidth]
    {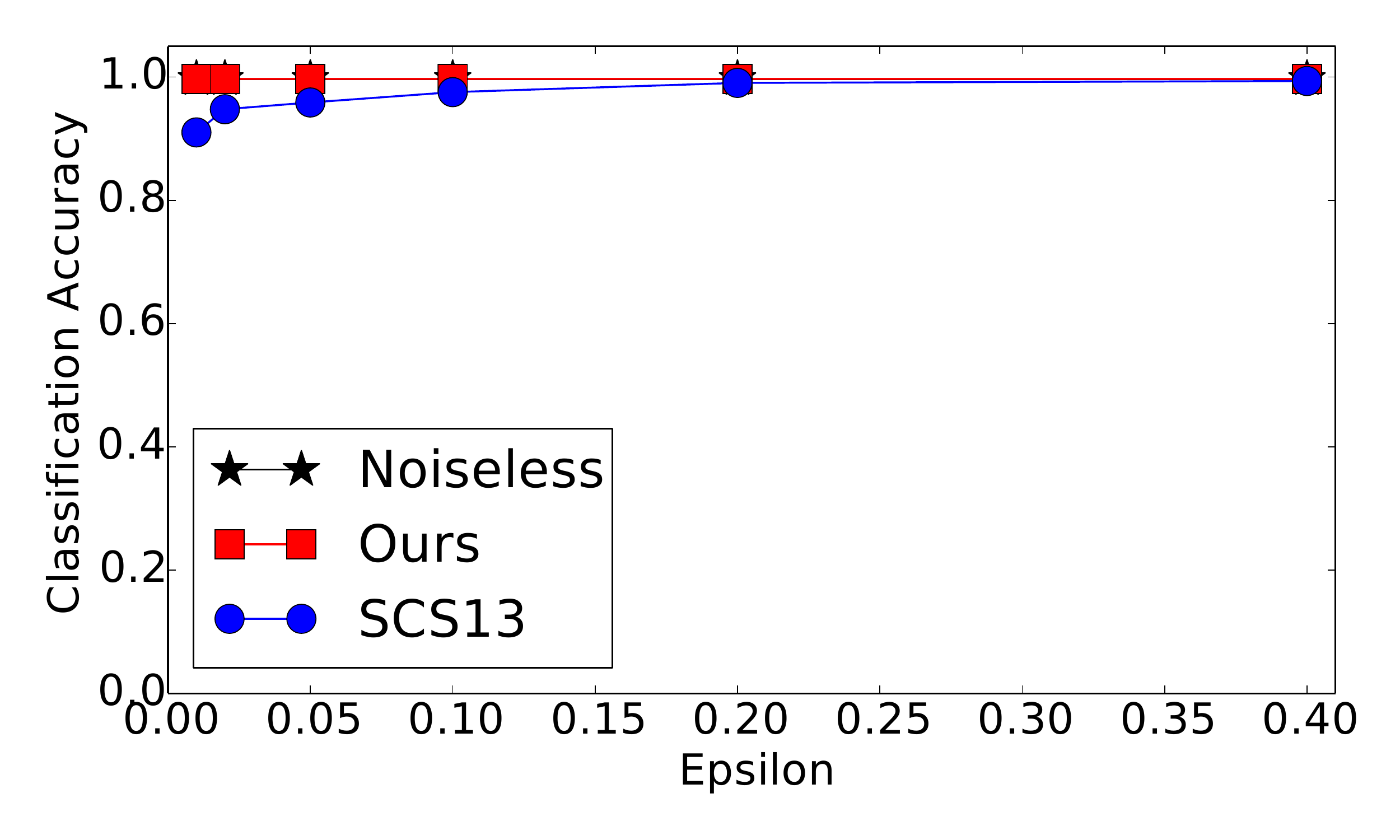}
  }
  \subfloat{
    \includegraphics[width=0.24\columnwidth]
    {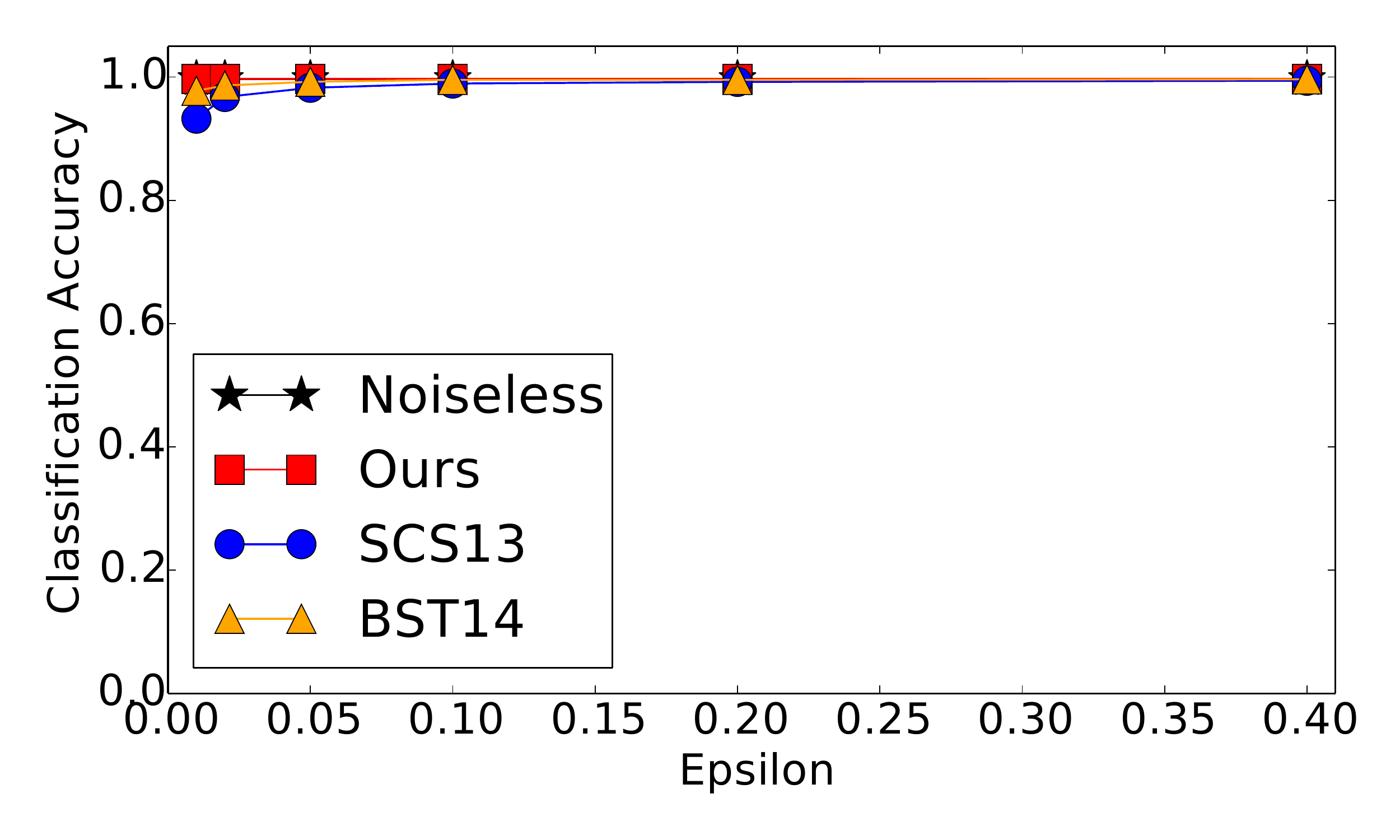}
  }
  \caption{
    {\bf More Accuracy Results of Tuning using Public Data.}
    Row 1 is HIGGS, row 2 is KDDCup-99.
  }
  \label{fig:notune-higgs-kddcup99:tests_50_mb_10_passes}
\end{figure*}

\section{Proofs}
\label{sec:proofs}

\subsection{Proof of Lemma~\ref{lemma:sensitivity:convex}}
\begin{proof}
  Let $T = km$, so we have in total $T$ updates.
  Applying Lemma~\ref{lemma:boundedness},
  Growth Recursion Lemma (Lemma~\ref{lemma:growth-recursion}),
  and the fact that the gradient operators are $1$-expansive, we have:
  \begin{align}
    \delta_t \le
    \begin{cases}
      \delta_{t-1} + 2L\eta_t
      & \begin{array}{@{}l@{}}
          \text{ if } t = i^* + jm, \\
          j = 0,\dots, k-1
        \end{array} \\
      & \\
      \delta_{t-1} & \text{ otherwise.}
    \end{cases}
  \end{align}
  Unrolling the recursion completes the proof.
\end{proof}

\subsection{Proof of Corollary~\ref{lemma:convex:decreasing-step}}
\begin{proof}
  We have that
  \[
    \sup_{S \sim S'}\sup_r \| A(r; S) - A(r; S') \|
    \le \frac{4L}{\beta}\left(\sum_{j=0}^{k-1}\frac{1}{m^c + jm + 1}\right).
  \]
  Therefore
  \begin{align*}
    \frac{4L}{\beta}\left(\sum_{j=0}^{k-1}\frac{1}{m^c + jm + 1}\right)
    =&\ \frac{4L}{\beta}\left(\frac{1}{m^c+1}
                            + \sum_{j=1}^{k-1}\frac{1}{m^c + jm + 1}\right) \\
    \le&\ \frac{4L}{\beta}\left(\frac{1}{m^c}
                            + \frac{1}{m}\sum_{j=1}^{k-1}\frac{1}{j}\right) \\
    \le&\ \frac{4L}{\beta}\left(\frac{1}{m^c} + \frac{\ln k}{m}\right)
  \end{align*}
  as desired.
\end{proof}

\subsection{Proof of Lemma~\ref{lemma:strongly-convex:constant-step}}
\begin{proof}
  Let $T = km$, so we have in total $T$ updates. We have the following recursion
  \begin{align}
    \label{recursion:sensitivity-multipasses}
    \delta_t \le
    \begin{cases}
      (1 - \eta\gamma)\delta_{t-1} + 2\eta L &
      \begin{array}{@{}l@{}}
        \text{ if } t = i^* + jm,\\
        j = 0, 1, \dots, k-1
      \end{array} \\
      & \\
      (1 - \eta\gamma)\delta_{t-1} & \text{ otherwise.}
    \end{cases}
  \end{align}
  This is because at each pass different gradient update operators
  are encountered only at position $i^*$
  (corresponding to the time step $t = i^* + jm$),
  and so the two inequalities directly follow from the growth recursion lemma
  (Lemma~\ref{lemma:growth-recursion}).
  Therefore, the contribution of the differing entry
  in the first pass contributes
  $2\eta L (1 - \eta\gamma)^{T-i^*}$,
  and generalizing this, the differing entry in the $(j+1)$-th pass
  $(j = 0, 1,\dots, k-1)$ contributes $2\eta L (1 - \eta\gamma)^{T-i^*-jm}$.
  Summing up gives the first claimed bound.

  For sensitivity, we note that for $j = 1, 2, \dots, k$,
  the $j$-th pass can only contribute at most
  $2\eta L \cdot (1-\eta\gamma)^{(k-j)m}$ to $\delta_T$.
  Summing up gives the desired result.
\end{proof}

\subsection{Proof of Lemma~\ref{lemma:strongly-convex:decreasing-step}}
\begin{proof}
  From the Growth Recursion Lemma (Lemma~\ref{lemma:growth-recursion}) we know that
  in the $\gamma$-strongly convex case, with appropriate step size,
  in each iteration either we have a contraction of $\delta_{t-1}$,
  or, we have a contraction of $\delta_{t-1}$ plus an additional additive term.
  In PSGD, in each pass the differing data point will only be encountered once,
  introducing an additive term, and is contracted afterwards.

  Formally, let $T$ be the number of updates, the differing data point
  is at location $i^*$. Let $\rho_t < 1$ be the expansion factor at iteration $t$.
  Then the first pass contributes $\delta_1^*\prod_{t=i^*+1}^T\rho_t$ to $\delta_T$,
  the second pass contributes $\delta_2^*\prod_{t=i^*+m+1}^T\rho_t$ to $\delta_T$.
  In general pass $j$ contributes $\delta_j^*\prod_{t=i^*+(j-1)m+1}^T\rho_t$
  to $\delta_T$.

  Let $\iota_j = \delta_j^*\prod_{t=i^*+(j-1)m+1}^T\rho_t$ be the contribution
  of pass $j$ to $\delta_T$. We now figure out $\delta_j^*$ and $\rho_t$.
  Consider $\iota_1$, we consider two cases. If $i^* \ge \frac{\beta}{\gamma}$,
  then $\eta_t \le \frac{1}{\gamma t} \le \frac{1}{\beta}$,
  and so $G_t$ is $(1-\eta_t\gamma) = (1 - \frac{1}{t})$ expansive.
  Thus if $i^* \ge  \frac{\beta}{\gamma}$ then before $i^*$ the gap is $0$
  and after $i^*$ we can apply expansiveness such that
  \[
  \frac{2L}{\gamma t} \cdot \prod_{i=t+1}^{km}
  \left( 1 - \frac{1}{i} \right)
  = \frac{2L}{\gamma t} \cdot \prod_{i=t+1}^{km} \frac{i-1}{i}
  = \frac{2L}{\gamma k m},
  \]
  The remaining case is when $i^* \le \frac{\beta}{\gamma} - 1$.
  In this case we first have $1$-expansiveness due to convexity
  that the step size is bounded by $\frac{1}{\beta} < \frac{2}{\beta}$.
  Moreover we have $(1 - \frac{1}{t})$-expansiveness for $G_t$
  when $\frac{\beta}{\gamma} \le t \le m$. Thus
  \[
  2L\eta_{i^*}
  \cdot
  \prod_{j = \frac{\beta}{\gamma}}^{km}\left(1 - \frac{1}{j}\right)
  \le \frac{2L\eta_{i^*} \beta/\gamma}{km}
  = 2L \cdot \frac{1}{\beta} \cdot \frac{\beta}{\gamma k m}
  = \frac{2L}{\gamma k m},
  \]
  Therefore $\iota_1 \le \frac{2L}{\gamma k m}$.
  Finally, for $j = 2, \dots, k$,
  \begin{align*}
    \iota_j \le
    \frac{2L}{\gamma((j-1)m + i^*)}
    \cdot \prod_{t = (j-1)m + i^* + 1}^{km} \frac{t - 1}{t}
    = \frac{2L}{\gamma k m}.
  \end{align*}
  Summing up gives the desired result.
\end{proof}

\subsection{Proof of Theorem~\ref{thm:Zinkevich-thm1}}
\begin{proof}
  The proof follows exactly the same argument as Theorem 1 of
  Zinkevich~\cite{Zinkevich03}, except we change the step size
  in the final accumulation of errors.
\end{proof}

\subsection{Proof of Theorem~\ref{lemma:convex:private-optimization-error}}
\begin{proof}
  The output of the private PSGD algorithm is $\tilde{w} = \bar{w}_T + \kappa$,
  where $\kappa$ is distributed according to a Gamma distribution
  $\Gamma(d, \frac{\Delta_2}{\varepsilon})$.
  By Lemma~\ref{lemma:convex:constant-step},
  $\Delta_2 \le 2L\eta = \frac{2R}{\sqrt{m}}$.
  Therefore by Lemma~\ref{lemma:error-due-to-privacy},
  $\Exp_\kappa[L_S(\tilde{w}) - L_S(\bar{w}_m)] \le \frac{2dR}{\varepsilon\sqrt{m}}$,
  where we use the fact that the expectation of the Gamma distribution is
  $\frac{d\Delta_2}{\varepsilon}$. Summing up gives the bound.
\end{proof}

% Logistic regression results on additional datasets: higgs and kddcup-99.
\begin{figure*}[!htb]
  \centering
  \subfloat{
    \includegraphics[width=0.24\columnwidth]
    {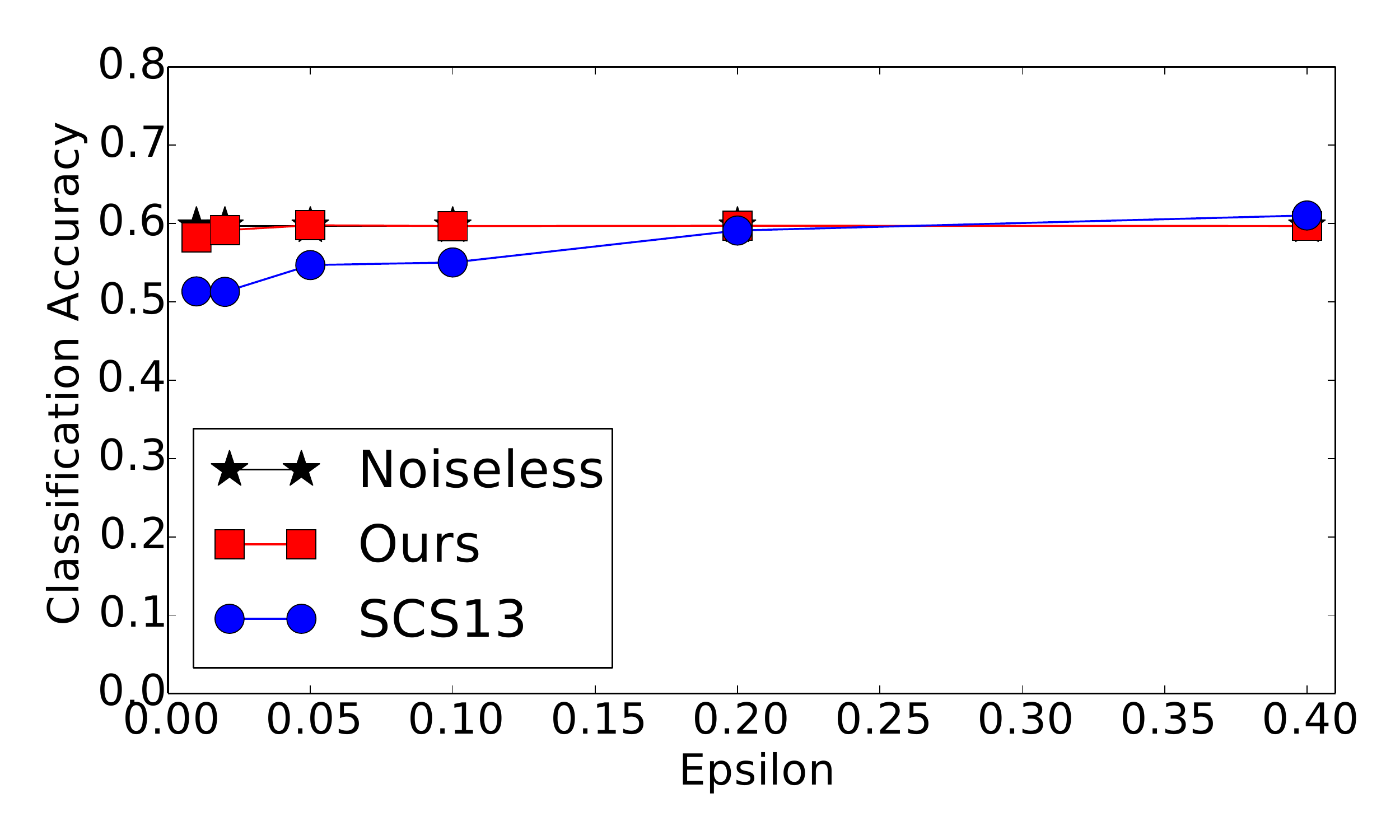}
  }
  \subfloat{
    \includegraphics[width=0.24\columnwidth]
    {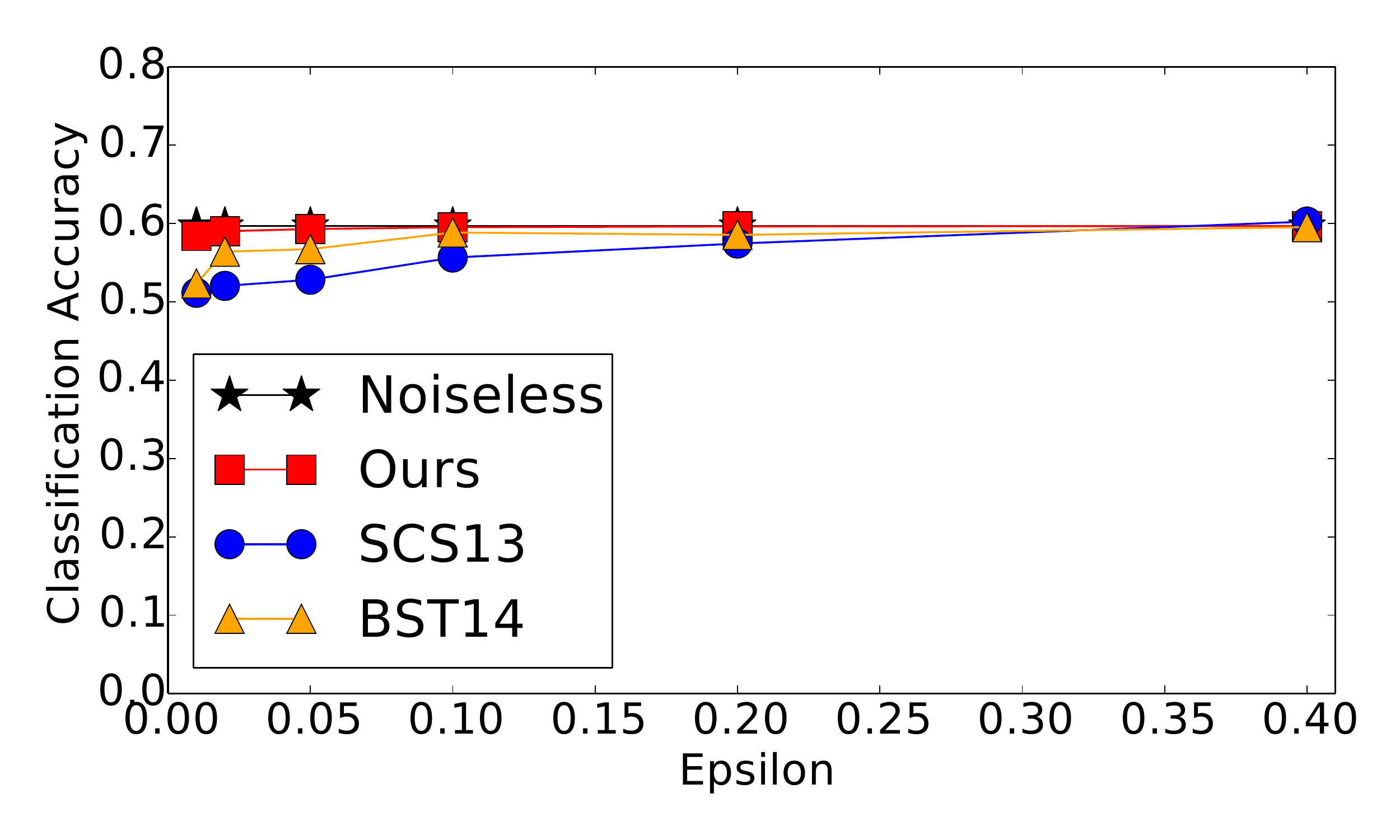}
  }
  \subfloat{
    \includegraphics[width=0.24\columnwidth]
    {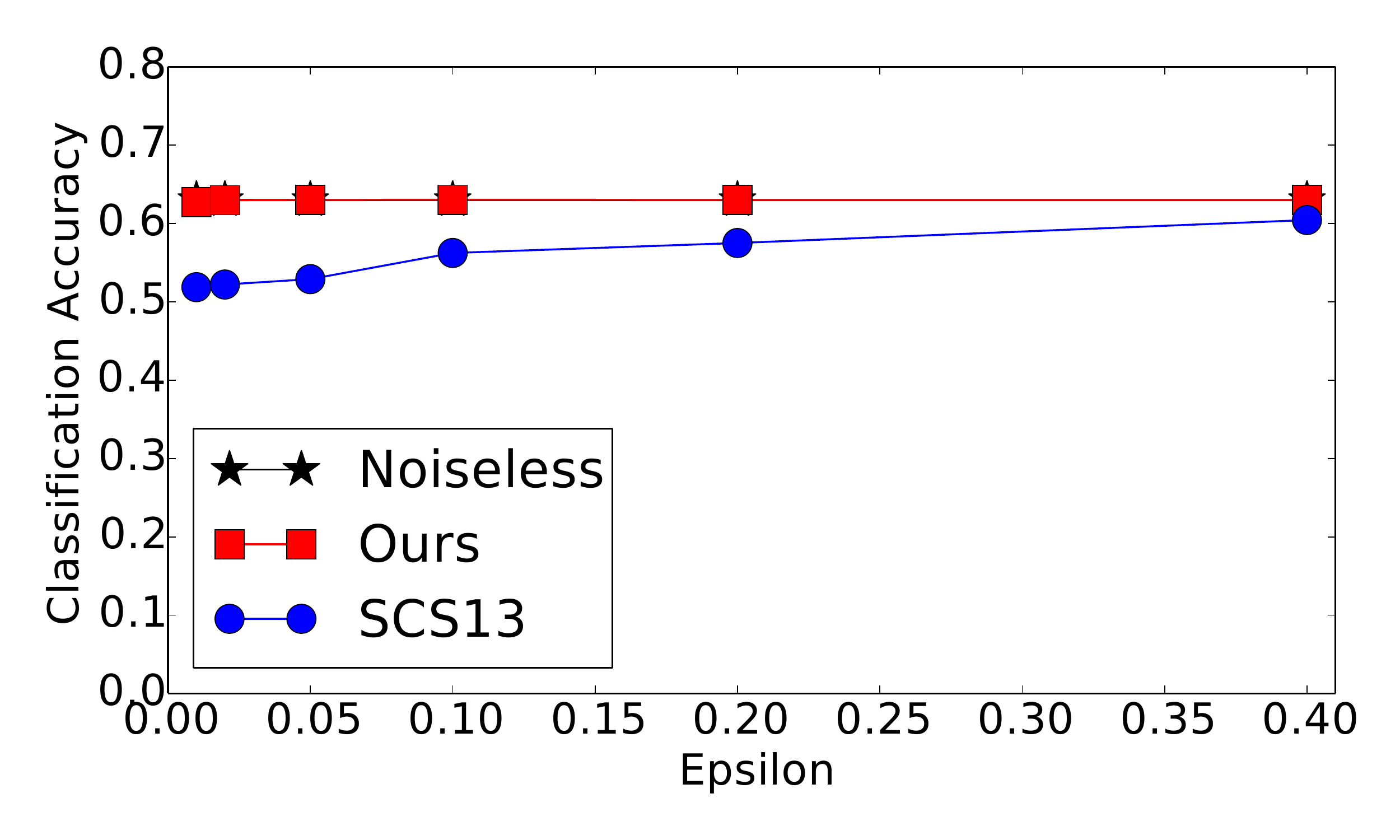}
  }
  \subfloat{
    \includegraphics[width=0.24\columnwidth]
    {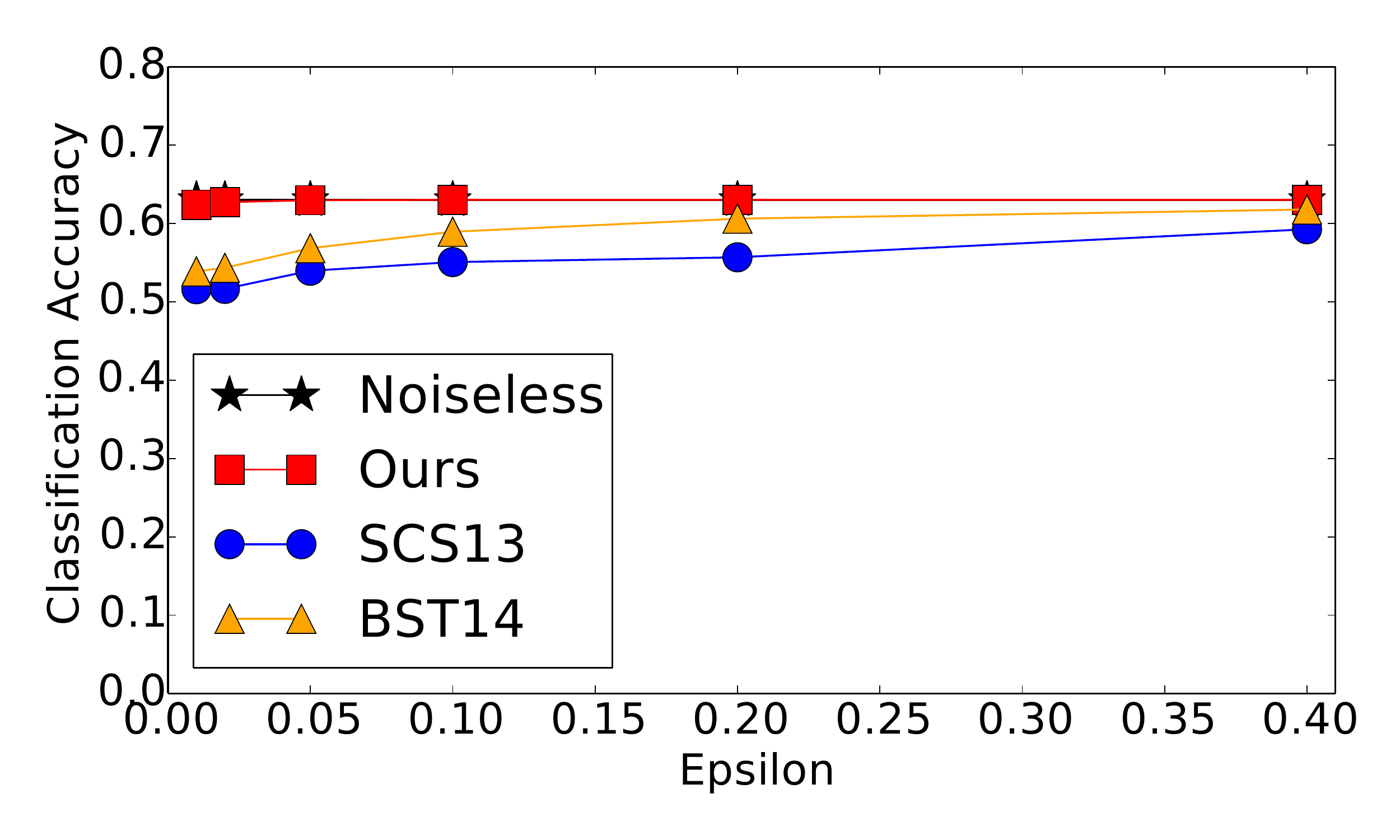}
  }\\[-3ex]
  \subfloat{
    \includegraphics[width=0.24\columnwidth]
    {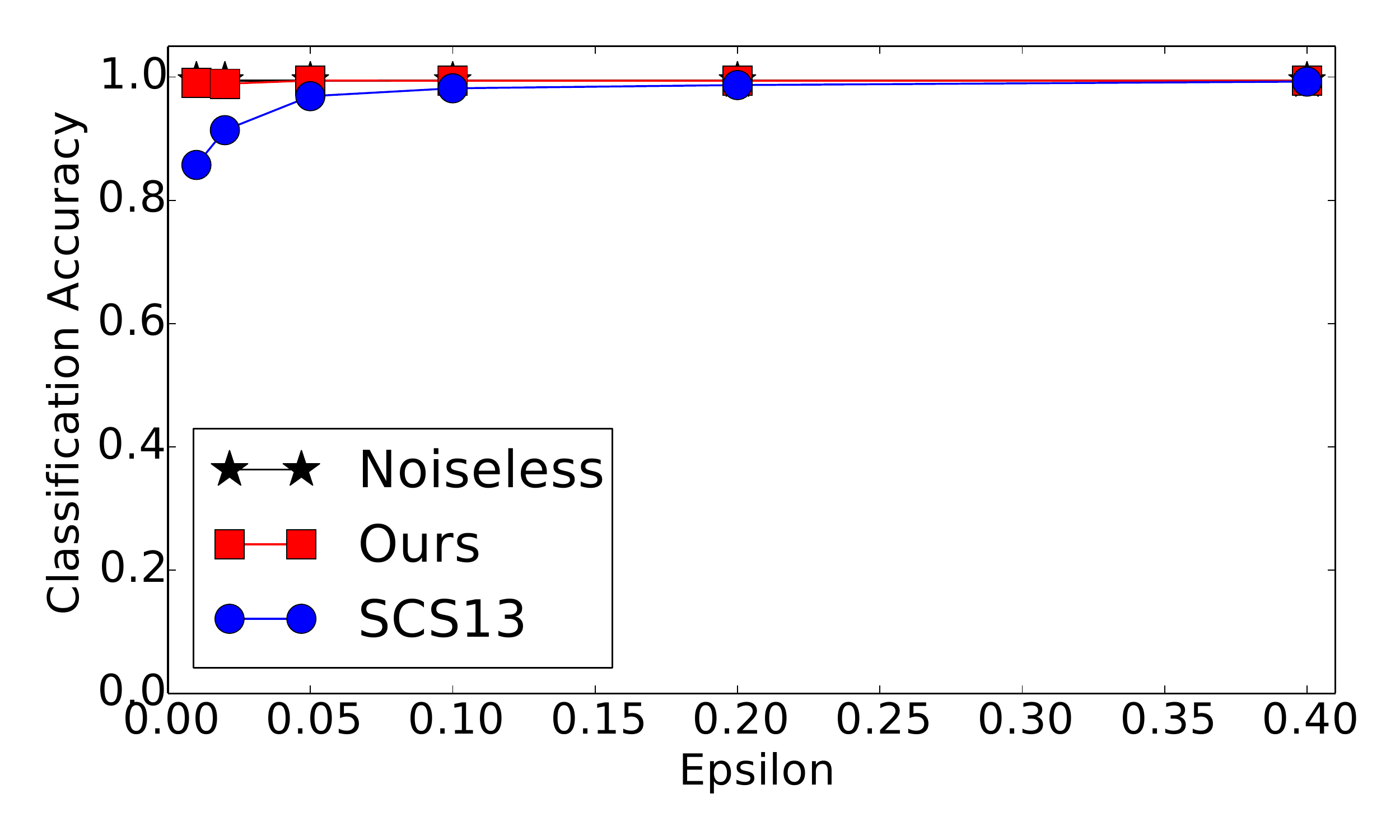}
  }
  \subfloat{
    \includegraphics[width=0.24\columnwidth]
    {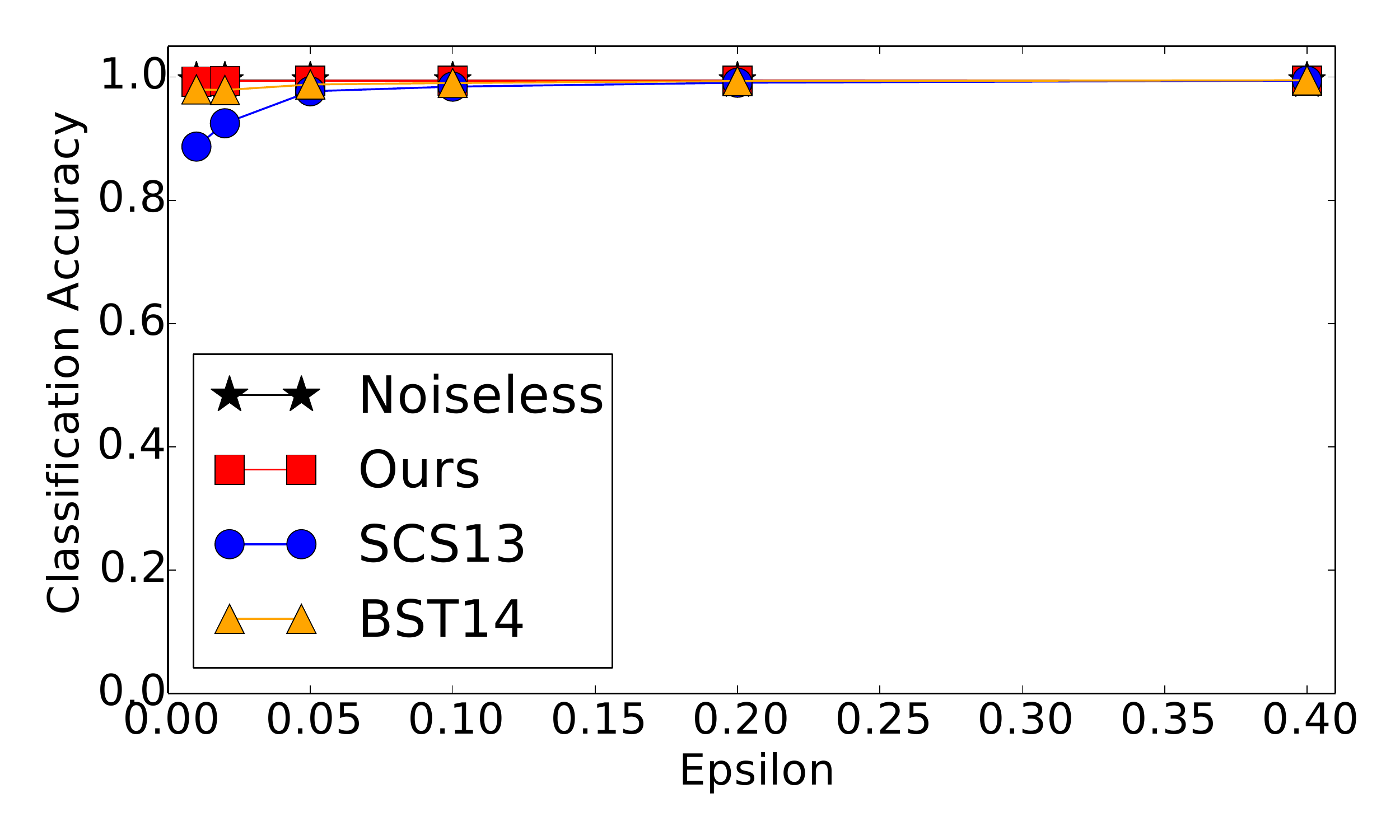}
  }
  \subfloat{
    \includegraphics[width=0.24\columnwidth]
    {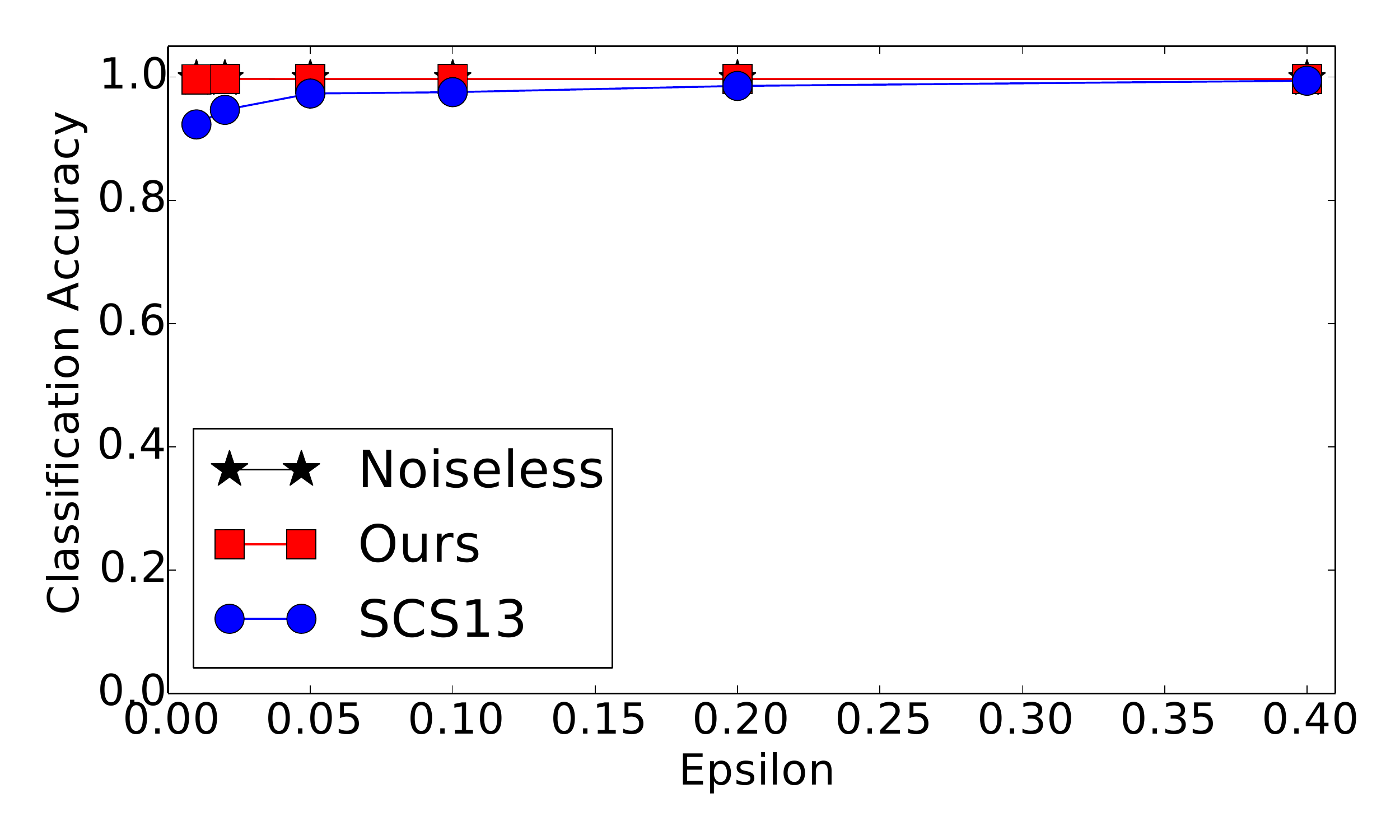}
  }
  \subfloat{
    \includegraphics[width=0.24\columnwidth]
    {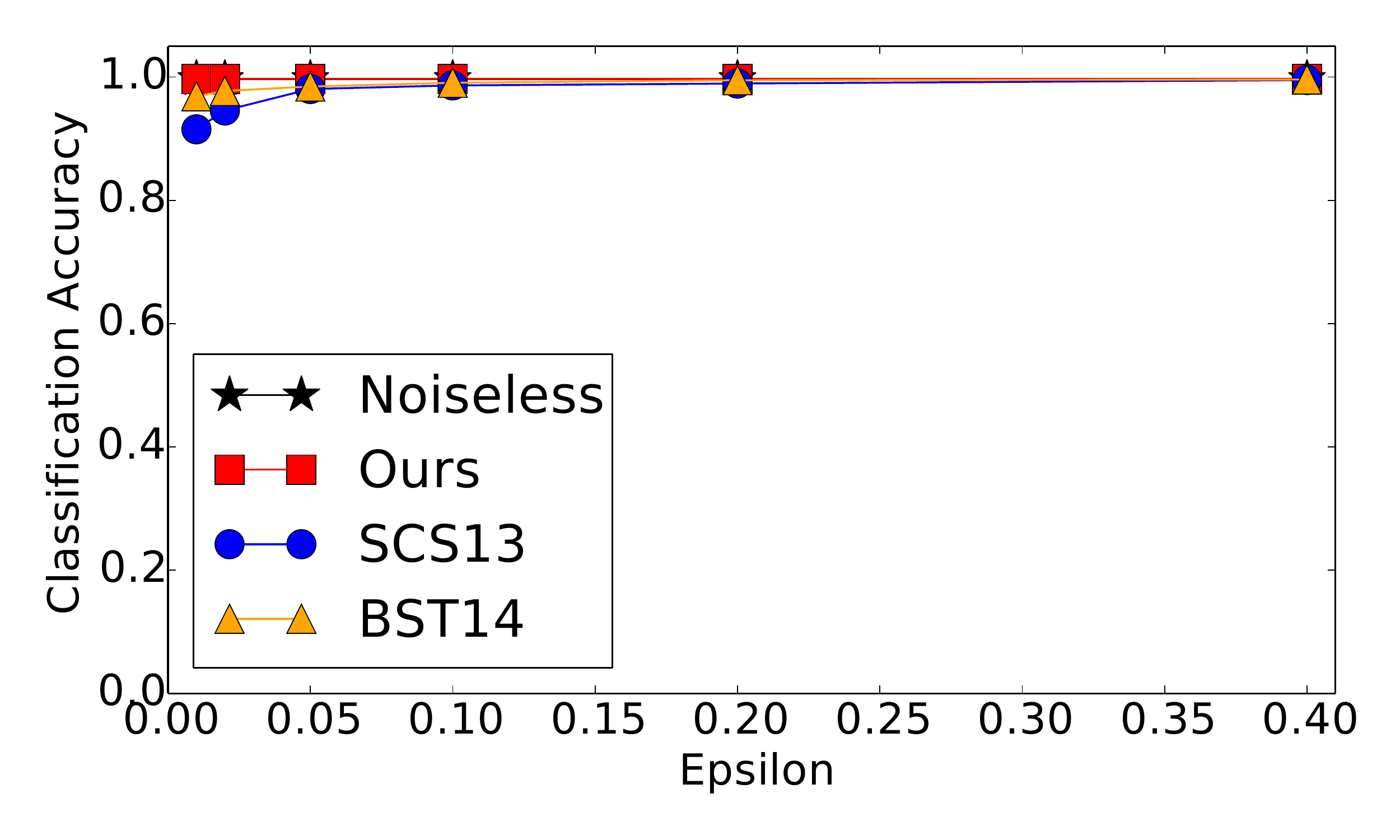}
  }
  \caption{
    {\bf More Accuracy Results with Private Tuning.}
    Row 1 is HIGGS, row 2 is KDDCup-99.
  }
  \label{fig:higgs-kddcup99:tests_50_mb_10_passes}
\end{figure*}

\section{Results with Huber SVM}
\label{sec:results-huber-svm}

In this section we report results on Huber SVM. The standard SVM uses
the {\em hinge loss} function, defined by
$\ell_{\rm SVM}(w, (x, y)) = \max(0, 1 - y\langle w, x\rangle)$,
where $x$ is the feature vector and $y \in \{\pm 1\}$ is the classification label.
However, hinge loss is not differentiable and so our results do not directly apply.
Fortunately, it is possible to replace hinge loss with a differentiable and smooth approximation,
and it works pretty well either in theory or in practice.
Let $z = y\langle w, x \rangle$, we use the following definition from~\cite{CMS11},
\begin{align*}
  \ell_{\rm Huber}(w, (x, y)) =
  \begin{cases}
    0  & \text{ if } z > 1 + h \\
    \frac{1}{4h}(1 + h - z)^2 & \text{ if } |1 - z| \le h \\
    1 - z & \text{ if } z < 1-h
  \end{cases}
\end{align*}
In this case one can show that (under the condition that all point are normalized to unit sphere)
$L \le 1$ and $\beta \le \frac{1}{2h}$ for $\ell_{\rm Huber}$, and our results thus apply.

Similar to the experiments with logistic regression,
we use standard Huber SVM for the convex case,
and Huber SVM regularized by $L_2$ regularizer for the strongly convex case.
Figure~\ref{fig:svm:accuracy:tests_50_mb_10_passes} reports the accuracy results
in the case of tuning with a private tuning algorithm.
Similar to the accuracy results on logistic regression results,
in all test cases our algorithms produce significantly more accurate models.
In particular for MNIST our accuracy is up to {\bf 6X}
better than BST14 and {\bf 2.5X} better than SCS13.

\red{
\section{Test Accuracy Results on Additional Datasets}
\label{sec:additional-datasets}
In this section we report test accuracy results on additional datasets:
HIGGS\footnote{\scriptsize\url{https://archive.ics.uci.edu/ml/datasets/HIGGS}.},
and KDDCup-99\footnote{\scriptsize\url{https://kdd.ics.uci.edu/databases/kddcup99/kddcup99.html}.}.
The test accuracy results of logistic regression are reported
in Figure~\ref{fig:notune-higgs-kddcup99:tests_50_mb_10_passes} for tuning with public data,
and in Figure~\ref{fig:higgs-kddcup99:tests_50_mb_10_passes} for private tuning.
These results further illustrate the advantages of our algorithms:
For large datasets differential privacy comes for free with our algorithms.
In particular, HIGGS is a very large dataset with $m = 10,500,000$ training points,
and this large $m$ reduces the noise to {\em negligible for our algorithms},
where we achieve almost the same accuracy as noiseless algorithms.
However, the test accuracy of SCS13 and BST14 are still {\em notably worse} than
that of the noiseless version, especially for small $\varepsilon$.
We find similar results for Huber SVM.

\section{Accuracy vs. Mini-batch Size}
\label{sec:acc-vs-mbsize}
In Figure~\ref{fig:accuracy-mbsize:mnist} we report more experimental results
when we increase mini-batch size from $50$ to $200$. Specifically we test for four mini-batch sizes,
$50, 100, 150, 200$. We report the test accuracy on MNIST using the strongly convex optimization,
and similar results hold for other optimization and datasets.
Encouragingly, we achieve almost native accuracy as we increase the mini-batch size.
On the other hand, while the accuracy also increases for SCS13 and BST14 for larger mini-batch sizes,
their accuracy is still significantly worse than our algorithms and noiseless algorithms.

\section{Sampling Laplace Noise}
\label{sec:sampling-laplace}
We discuss briefly how to sample from (\ref{output-perturbation:noise-distribution}).
We are given dimension $d$, $L_2$-sensitivity $\Delta$ and privacy parameter $\varepsilon$.
In the first step, we sample a uniform vector in the unit ball, say $\bf v$
(this can be done, for example, by a trick described in~\cite{sample-uniform-random-unit-vector}).
In the second step we sample a magnitude $l$ from Gamma distribution $\Gamma(d, \Delta/\varepsilon)$,
which can be done, for example, via standard Python API (np.random.gamma).
Finally we output $\kappa = l{\bf v}$. The same algorithm is used in~\cite{CMS11}.
}

\begin{figure*}[!htb]
  \centering
  \subfloat[$b=50$]{
    \includegraphics[width=0.24\columnwidth]
    {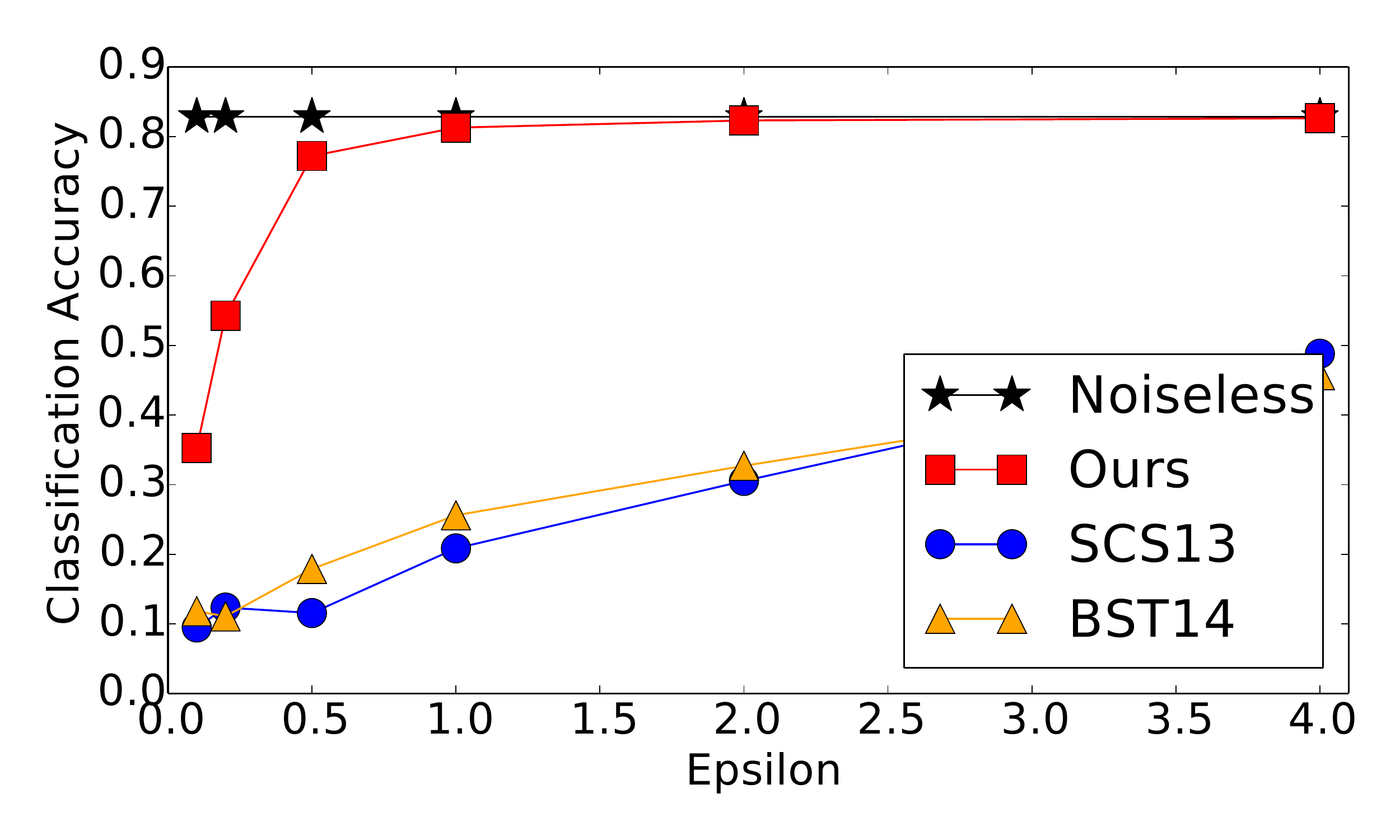}
  }
  \subfloat[$b=100$]{
    \includegraphics[width=0.24\columnwidth]
    {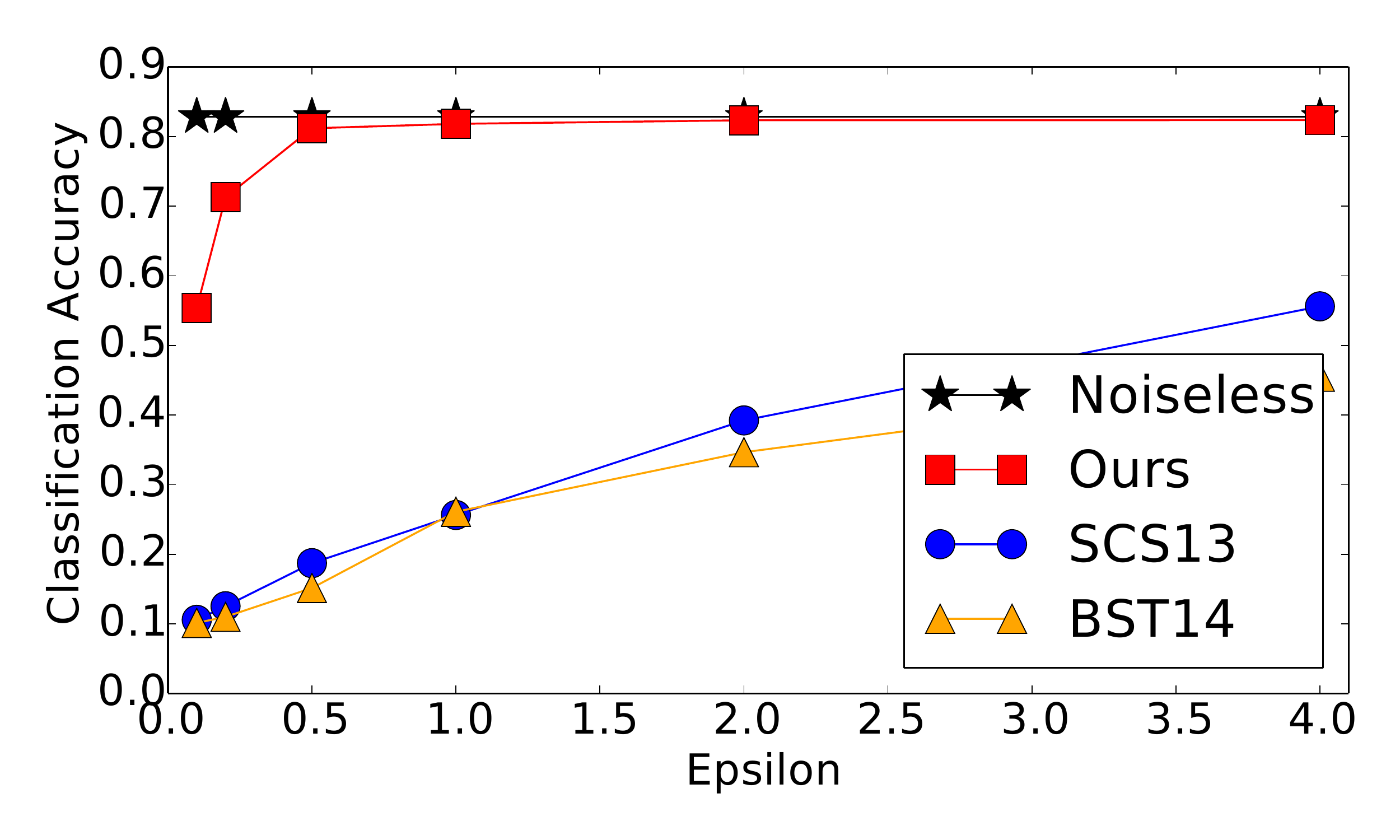}
  }
  \subfloat[$b=150$]{
    \includegraphics[width=0.24\columnwidth]
    {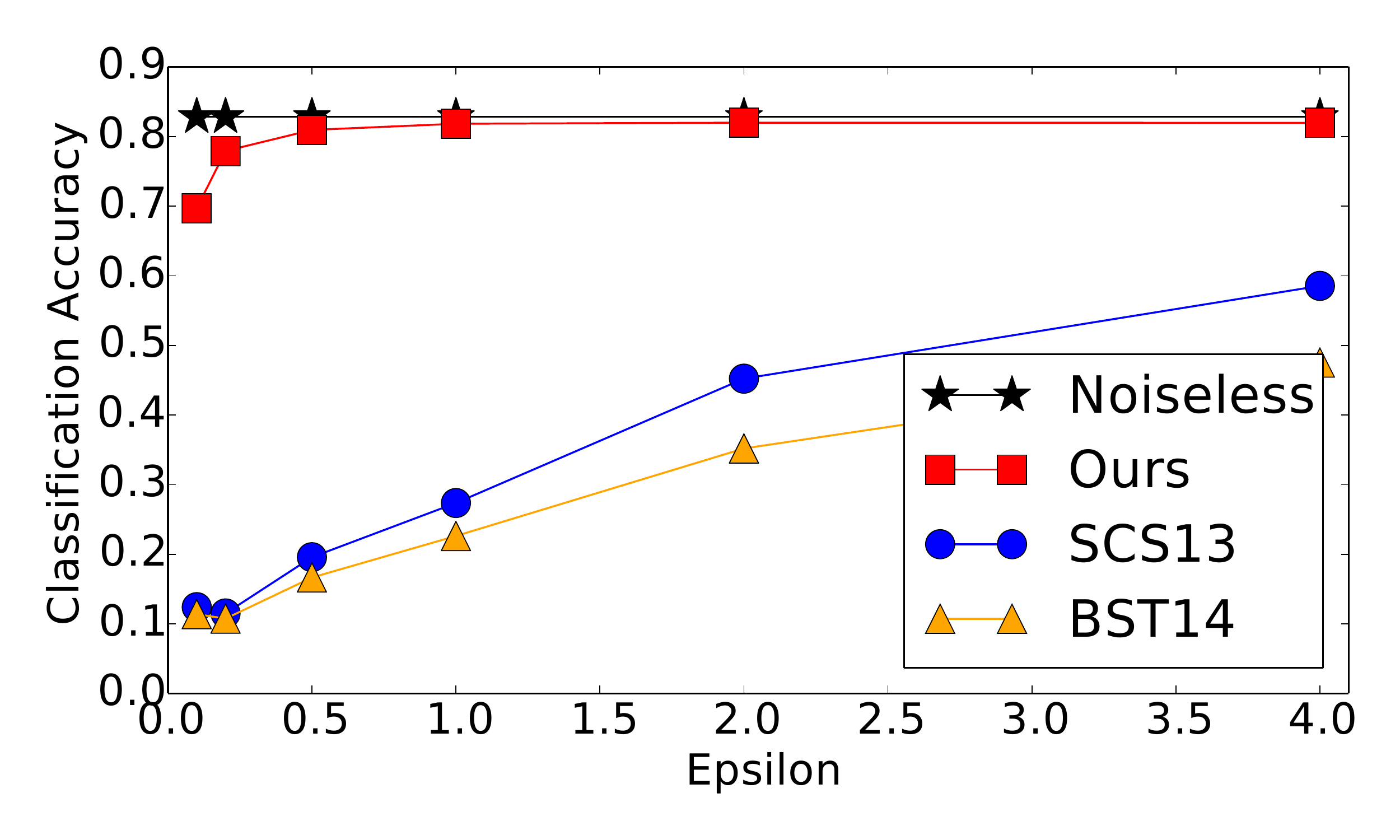}
  }
  \subfloat[$b=200$]{
    \includegraphics[width=0.24\columnwidth]
    {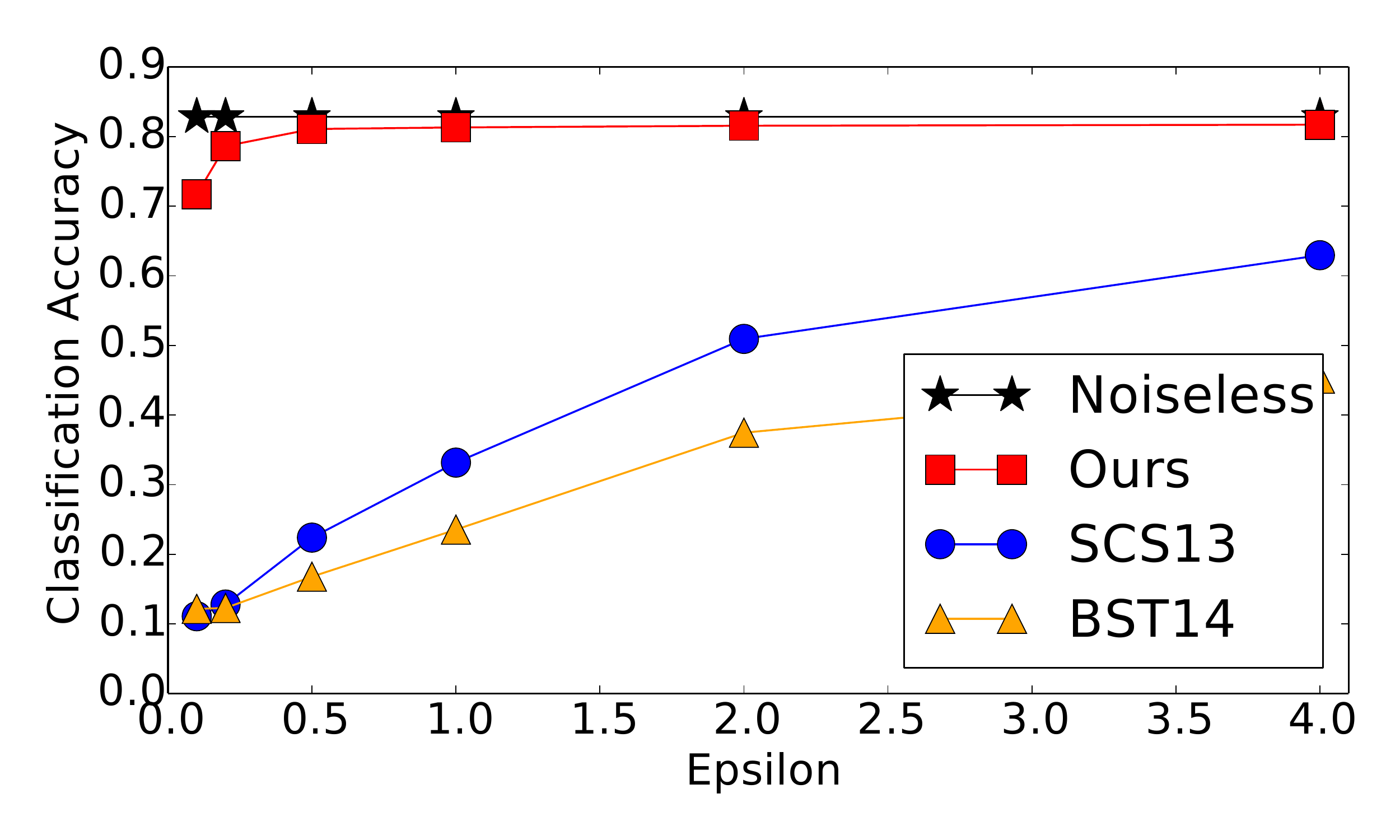}
  }
  \caption{
    {\bf Mini-batch Size vs. Accuracy: More Results.} We consider four mini-batch sizes $50, 100, 150, 200$.
  }
  \label{fig:accuracy-mbsize:mnist}
\end{figure*}

\section{BST14 with Constant Number of Epochs}
\label{sec:extended-bst14-algorithms}
\begin{algorithm}[!htp]
  \caption{Convex BST14 with Constant Epochs}
  \label{alg:convex-bst14-plus}
  \begin{algorithmic}[1]
    \Require{$\ell(\cdot, z)$ is convex for every $z$, $\eta \le 2/\beta$.}
    \Input{Data $S$, parameters $k, \varepsilon, \delta, d, L, R$}
    \Function{\sf\small ConvexBST14ConstNpass}{$S, k, \varepsilon, \delta, d, L, R$}
      \Let{$m$}{$|S|$}
      \Let{$T$}{$km$}
      \Let{$\delta_1$}{$\delta/km$}
      \Let{$\varepsilon_1$}{Solution of
        $\varepsilon=T\varepsilon_1(e^\varepsilon_1-1)
        + \sqrt{2T\ln(1/\delta_1)}\varepsilon_1$
      }
      \Let{$\varepsilon_2$}{$\min(1, m\varepsilon_1/2)$}
      \Let{$\sigma^2$}{$2\ln(1.25/\delta_1)/\varepsilon_2^2$}
      \Let{$w$}{$0$}
      \For{$t = 1, 2, \dots, T$}
        \State $i_t \sim [m]$ and let $(x_{i_t}, y_{i_t})$ be the data point.
        \State $z \sim {\cal N}(0, \sigma^2\iota I_d)$
        \Comment{$\iota = 1$ for logistic regression,
          and in general is the $L_2$-sensitivity localized to an iteration;
          $I_d$ is $d$-dimensional identity matrix.}
        \Let{$w$}{$\prod_{\cal W}\big( w
          - \eta_t(\nabla\ell(w; (x_{i_t}, y_{i_t}) + z) \big)$
          where $\eta_t = \frac{2R}{G\sqrt{t}}$ and $G = \sqrt{d\sigma^2 + b^2L^2}$.
        }
      \EndFor
      \State \Return{$w_T$}
    \EndFunction
  \end{algorithmic}
\end{algorithm}\newpage

\begin{algorithm}[!htp]
  \caption{Strongly Convex BST14 with Constant Epochs}
  \label{alg:strongly-convex-bst14-plus}
  \begin{algorithmic}[1]
    \Input{Data $S$, parameters $k, \varepsilon, \delta, d, L, R$}
    \Function{\sf\small StronglyConvexBST14ConstNpass}{$S, k, \varepsilon, \delta,
      d, L, R$}
      \Let{$m$}{$|S|$}
      \Let{$T$}{$km$}
      \Let{$\delta_1$}{$\delta/km$}
      \Let{$\varepsilon_1$}{Solution of
        $\varepsilon=T\varepsilon_1(e^\varepsilon_1-1)
        + \sqrt{2T\ln(1/\delta_1)}\varepsilon_1$
      }
      \Let{$\varepsilon_2$}{$\min(1, m\varepsilon_1/2)$}
      \Let{$\sigma^2$}{$2\ln(1.25/\delta_1)/\varepsilon_2^2$}
      \Let{$w$}{$0$}
      \For{$t = 1, 2, \dots, T$}
        \State $i_t \sim [m]$ and let $(x_{i_t}, y_{i_t})$ be the data point.
        \State $z \sim {\cal N}(0, \sigma^2\iota I_d)$
        \Let{$w$}{$\prod_{\cal W}\big( w
          - \eta_t(\nabla\ell(w; (x_{i_t}, y_{i_t}) + z) \big)$,
          $\eta_t = \frac{1}{\gamma t}$.
        }
      \EndFor
      \State \Return{$w$}
    \EndFunction
  \end{algorithmic}
\end{algorithm}

%%% Local Variables:
%%% mode: latex
%%% TeX-master: "make"
%%% End:

\end{document}